\pdfoutput=1 
\documentclass[
12pt, 
english, 
doublespacing, 
nolistspacing, 
liststotoc, 
headsepline, 
chapterinoneline, 
openany, 
]{DoctoralThesis} 

\usepackage[utf8]{inputenc} 
 
\usepackage[T1]{fontenc} 

\usepackage{mathpazo} 

\usepackage[backend=bibtex,style=authoryear,natbib=true]{biblatex} 

\usepackage[autostyle=true]{csquotes} 
\usepackage{enumerate}
\usepackage{xcolor}
\usepackage{graphicx}
\usepackage{mathrsfs}
\usepackage{multirow}
\usepackage{hhline}
\usepackage{pbox}
\usepackage{hyperref}
\usepackage{amsmath,amssymb,amsthm,esint}
\usepackage{longtable}
\usepackage{tikz}
\usepackage{bbm}
\usepackage{float}
\usepackage{pgfplots}
\usepackage{caption}
\usepackage{subcaption} 
\usepackage{empheq}
\usepackage{indentfirst}
\usepackage{algorithm}

\usepackage[many]{tcolorbox}
\usepackage[T1]{fontenc}
\usepackage[utf8]{inputenc}
\usepackage{lmodern}
\usepackage{algpseudocode}
\usepackage{booktabs}
\usepackage{pgfplots} 
\usepackage{tikz-3dplot}
\usepackage{dsfont}
\usepackage{mathpazo}

\usepackage[utf8]{inputenc} 
\usepackage[T1]{fontenc}    
\usepackage{hyperref}       
\usepackage{calc}
\usepackage{url}            
\usepackage{booktabs}       
\usepackage{amsfonts}       
\usepackage{nicefrac}       
\usepackage{microtype}      
\usepackage{textcase}
\usepackage{thmtools,thm-restate}

\newlength{\arrow}
\newtheorem{theorem}{Theorem}[section]

\newtheorem{corollary}[theorem]{Corollary}

\newtheorem{definition}[theorem]{Definition}

\newtheorem{proposition}[theorem]{Proposition}

\newcommand{\e}{\varepsilon}
\newcommand{\relu}{\text{ReLU}}
\newcommand{\norm}[1]{\left\vert #1 \right\vert}
\newcommand{\Norm}[1]{\left\Vert #1 \right\Vert}

\newcommand{\R}{\mathbb{R}}

\newcommand{\CC}{\mathcal{C}}

\newcommand{\ip}[1]{\left\langle #1 \right\rangle}

\def\layersep{2.0cm}

\tikzstyle{decision} = [diamond, draw, fill=blue!20,
    text width=3cm, text badly centered, node distance=3cm, inner sep=0pt]
\tikzstyle{block} = [rectangle, draw, fill=blue!20,
    text width=3cm, text centered, rounded corners, minimum height=2em]
\tikzstyle{line} = [draw, -latex']
\tikzstyle{cloud} = [draw, ellipse,fill=red!20, node distance=3cm,
    minimum height=2em]

\newcounter{row}
\newcounter{col}
\newcommand{\RR}{I\!\!R} 
\newlength{\depthofsumsign}
\newlength{\totalheightofsumsign}
\newlength{\heightanddepthofargument}

\makeatletter
\newcommand*{\DivideLengths}[2]{%
  \strip@pt\dimexpr\number\numexpr\number\dimexpr#1\relax*65536/\number\dimexpr#2\relax\relax sp\relax
}
\tdplotsetmaincoords{60}{115}
\pgfplotsset{compat=newest}

    \pgfplotsset{compat=newest}
    \usetikzlibrary{shapes,arrows,matrix,positioning,intersections}
    \usetikzlibrary{calc}
    \usetikzlibrary{decorations} 
    \pgfdeclarelayer{bg}
    \pgfsetlayers{bg,main}
    \tikzstyle{decision} = [diamond, draw, fill=blue!20, 
        text width=3cm, text badly centered, node distance=3cm, inner sep=0pt]
    \tikzstyle{block} = [rectangle, draw, fill=blue!20, 
        text width=3cm, text centered, rounded corners, minimum height=2em]
    \tikzstyle{line} = [draw, -latex']
    \tikzstyle{cloud} = [draw, ellipse,fill=red!20, node distance=3cm,
        minimum height=2em]

\usetikzlibrary{decorations.text,arrows.meta}
\usetikzlibrary{decorations.pathmorphing}
\usetikzlibrary{decorations.markings} 

\tikzset{
  on each segment/.style={
    decorate,
    decoration={
      show path construction,
      moveto code={},
      lineto code={
        \path [#1]
        (\tikzinputsegmentfirst) -- (\tikzinputsegmentlast);
      },
      curveto code={
        \path [#1] (\tikzinputsegmentfirst)
        .. controls
        (\tikzinputsegmentsupporta) and (\tikzinputsegmentsupportb)
        ..
        (\tikzinputsegmentlast);
      },
      closepath code={
        \path [#1]
        (\tikzinputsegmentfirst) -- (\tikzinputsegmentlast);
      },
    },
  },
  mid arrow/.style={postaction={decorate,decoration={
        markings,
        mark=at position .5 with {\arrow[#1]{stealth}}
      }}},
}

\thesistitle{A Geometric Framework for Adversarial Vulnerability in
  Machine Learning} 

\chair{David Glickenstein} 

\cochair{} 

\examiner{} 

\degree{Doctor of Philosophy} 

\author{Brian Bell} 

\addresses{413 E Geronimo Bluff Loop, Tucson AZ, 85705} 

\subject{Applied Mathematics} 

\keywords{Machine Learning, Geometry, Monte Carlo} 

\university{University of Arizona} 

\department{GRADUATE INTERDISCIPLINARY PROGRAM IN APPLIED MATHEMATICS} 

\group{\href{http://researchgroup.university.com}{Research Group Name}} 

\facultyA{Kevin Lin} 

\facultyB{Marek Rychlik} 

\facultyC{Hoshin Gupta} 

\facultyD{Type the name of your committee member here} 

\defense{Type defense date here} 

\AtBeginDocument{
\hypersetup{pdftitle=\ttitle} 
\hypersetup{pdfauthor=\authorname} 
\hypersetup{pdfkeywords=\keywordnames} 
\hypersetup{hidelinks} 
}

\usepackage{ragged2e}
\newenvironment{dedication}
  {\clearpage           
   \thispagestyle{empty}
   \vspace*{\stretch{1}}
   \section*{Dedication}
   \itshape             
   \justifying          

  }
  {\par                 
   \vspace{\stretch{3}} 
   \clearpage           
  }
\usepackage{tocloft,calc}

\makeatletter
\g@addto@macro\appendix{%
  \addtocontents{toc}{%
    \protect\renewcommand{\protect\cftchappresnum}{\appendixname\space}%
  }%
}
\makeatother

\begin{document}


\pagestyle{plain} 


%

\begin{titlepage}
\begin{singlespacing} 
\begin{center}

\vfill

\MakeUppercase{\ttitle}\\ 
\vspace{0.4in}
by\\ \vspace{0.4in}
{\authorname}\\ 
\vspace{0.6in}
\HRule \\[0.1cm] 
Copyright \textcopyright\space\authorname\space{\the\year}\\ 

\vspace{0.4in}

A Dissertation Submitted to the Faculty of the\\ 
\vspace{0.4in}
\MakeUppercase{\deptname} \\  
\vspace{0.4in}
In Partial Fulfillment of the Requirements \\ \medskip 
For the Degree of \\  
\vspace{0.4in}
\MakeUppercase{\degreename} \\ 
\vspace{0.4in} 
In the Graduate College \\  
\vspace{0.4in}
\MakeUppercase{The \univname} \\ 
\vspace{0.6in}
{\the\year}\\[4cm] 

\vfill
\end{center}
\end{singlespacing}
\end{titlepage}


\setcounter{page}{2} 

\begin{center}
\MakeUppercase{Acknowledgements}\\ \bigskip
\end{center}


%

\begin{dedication}
    For my Supportive Friends and Confidants\ldots \\\bigskip
    Dedicated to my humorous and
    wonderful late mother Carrie Bell\ldots.
    \end{dedication}

%
%
\tableofcontents 
%
\listoffigures 

\listoftables 


%

\addchaptertocentry{\abstractname} 

\begin{center}
\MakeUppercase{Abstract}\\ \bigskip
\end{center}

This work starts with the intention of using mathematics to understand
the intriguing vulnerability observed by ~\citet{szegedy2013} within
artificial neural networks. Along the way, we will develop some novel
tools with applications far outside of just the adversarial domain. We
will do this while developing a rigorous mathematical framework to
examine this problem. Our goal is to build out theory which
can support increasingly sophisticated conjecture about adversarial
attacks with a particular focus on the so called ``Dimpled Manifold
Hypothesis'' by ~\citet{shamir2021dimpled}. Chapter one will cover the history and architecture of neural network
architectures. Chapter two is focused on the background
of adversarial vulnerability. Starting from the seminal paper by
~\citet{szegedy2013} we will develop the theory of adversarial
perturbation and attack.

Chapter three will build a theory of persistence
that is related to Ricci Curvature, which can be used to measure
properties of decision boundaries. We will use this foundation to make
a conjecture relating adversarial attacks. Chapters four and five represent a sudden and wonderful digression
that examines an intriguing related body of theory for spatial
analysis of neural networks as approximations of kernel machines and
becomes a novel theory for representing neural networks with bilinear
maps. These heavily mathematical chapters will set up a framework
and begin exploring applications of what may become a very important
theoretical foundation for analyzing neural network learning with
spatial and geometric information. We will
conclude by setting up our new methods to address the conjecture from
chapter 3 in continuing research.


%
%



\vspace*{0.2\textheight}

\noindent\enquote{\itshape Young man, in mathematics you don't understand things. You just get used to them.}\bigbreak

\hfill John von Neumann

%
%

%
%

%
%
%



\pagestyle{thesis} 


\chapter{Introduction} 
\label{Chapter1} 

The primary aim of this study is to comprehend the perplexing vulnerability identified by \citet{szegedy2013} within artificial neural networks. These models consistently admit inputs that may be geometrically close and often indistinguishable, according to users, from natural data, yet they lead to substantial changes in output. In recent years, the utilization of such models has proliferated across scientific, industrial, and personal applications. Despite their widespread and growing use, the impact of adversarial vulnerability on security, reliability, and safety in machine learning remains poorly understood. To systematically investigate this phenomenon, a robust theory and a rigorous mathematical framework are essential. We intend to construct this framework from the ground up, leveraging high-dimensional geometry tools related to curvature, implicit representations drawn from functional analysis and the theory of Hilbert spaces, optimization, and rigorous experimentation. Our overarching goal is to develop a theoretical foundation that can support increasingly sophisticated conjectures about adversarial attacks, with a particular focus on the "Dimpled Manifold Hypothesis" proposed by \citet{shamir2021dimpled}.

Chapter One will delve into the history and architecture of neural networks. This section will provide a brief overview of the surprising theoretical and experimental results that underpin modern machine learning. Alongside these theoretical foundations, a concise yet rigorous introduction to the practical components of a neural network model, including its training objective and considerations, will be presented. Chapter Two will concentrate on the background of adversarial vulnerability. Commencing with the seminal paper by \citet{szegedy2013}, we will develop the theory of adversarial perturbation and attack. This chapter will explore both the theoretical and practical objectives and optimization problems related to creating adversarial examples. However, we will emphasize the establishment of careful definitions for what constitutes an "adversarial" scenario in the context of modern machine learning.

Chapter Three follows the outline of a paper which is currently under submission
on the topic of measuring geometric properties in order to understand
adversarial examples, and to identify certain classes of attack when
they are performed. This chapter will build a theory of persistence
which is related to Ricci Curvature, which can be used to measure
properties of decision boundaries. These properties include the theory
of persistence and how it changes while interpolating across decision
boundaries, measurement of normal vectors for decision boundaries, and
comparison of various interpolation trajectories between various
combinations of natural classed images, adversarial images, etc... The
conclusion of this chapter is a conjecture relating adversarial
attacks with properties of the decision boundary defined by an
arbitrary model. 

Chapter Four starts
from a recent theory of the neural tangent kernel. By a re-ordering of
a gradient formulation for the steps taken during optimization of a
model for a given loss function, a representation can be obtained as
the inner product of the model gradient for a test point with the sum
over the model gradients for a query point. This inner product
formulation is colloquially termed a ``kernel method'' or a ``kernel
machine'' which has some clearer mathematical properties than
artificial neural networks do. Well-known work has demonstrated 
that this representation is exact for infinite-width neural networks,
but has significant shortcomings for practical finite models. This
chapter outlines a recently accepted paper ~\citep{bell2023} in
which this kernel method is extended by careful integration along the
training steps of a model's optimization in order to produce an exact
representation for finite networks which can be computed in practice,
and whose numerical error can be easily controlled.

Chapter Five combines these theories and a better tool for
visualizing (soft ricci curvature) near decision boundaries in order
to refine the conjecture stated in Chapter Three. The main purpose of
this chapter is to consolidate this conjecture as a foundation for a
large body of continuing work. Based on the foundation presented
within the first parts of this work, we will paint an actionable
geometric picture for how adversarial examples come about and the
properties of both the data and models used within modern
machine-learning applications. This work will be collected into a
forthcoming submission to ICLR (International Conference on Learning Representations). 


\section{Background}

Artificial Neural Networks and other optimization-based general
function approximation models are the core of modern
machine learning ~\citep{prakash2018}. These models have dominated competitions in image processing, optical character recognition, object detection, video classification, natural language processing, and many other fields ~\citep{SCHMIDHUBER201585}. All such modern models are
trained via gradient-based optimization, e.g. Stochastic Gradient Descent (SGD) with
gradients computed via back propagation using theory brought to the
mainstream by ~\citet{goodfellow2013multidigit}. Although the performance of these models is practically
miraculous within the training and testing context for which they are
designed, they have a few intriguing properties. It was discovered in
2013 by ~\citet{szegedy2013} that images can be generated
which apparently trick such models in a classification context in  difficult-to-control ways ~\citep{khoury2018}. The intent of this
research is to investigate these \emph{adversarial examples} in a
mathematical context and use them to study pertinent 
properties of the learning models from which they arise.

\subsection{Artificial Neural Networks (ANNs)}


The history of Neural Networks (NNs) begins very gradually in the field of Theoretical
Neuropsychology with a much-cited paper by McCulloch and Pitts in
which the mechanics of cognition are described in the context of
computation by ~\citet{mcculloch1943logical}. This initial framework for
computational cognition did not include a notion for learning, however
this would be brought in the following decade with in the form of
optimization and many simple NNs (linear regression models applied to
computational cognition). The perceptron, the most granular element of
a neural network, was proposed in another much-cited paper ~\citet{rosenblatt1958perceptron}. A full 7 years
would pass before these building blocks would be assembled into
multilevel (deep) networks which were proposed by 1965 in a paper by ~\citet{ivakhnenko1965cybernetic}. Despite much
theoretical work up to this point, computing resources of the time
were not nearly capable of implementing all but the simplest toy
versions of these models. 

By the 1960s, these neural network models became disassociated from
the cutting-edge of cognitive science, and interest had shifted to
their application in modeling and industrial computation. The hardware
limitations of the time served as a significant barrier to wider
application and the concept of the "neural network" was generally
regarded as a cute solution looking for a problem. Compounding these
limitations was a significant roadblock published by ~\citet{minsky1969perceptrons}: A proof that basic perceptrons could not encode exclusive-or. As a result, interest in developing
neural network theory waned. The next necessary step in the
development of modern neural network models was an advance that would
allow them to be trained efficiently with computing power
available. Learning methods required a gradient, and the technique
necessary for computing gradients of large-scale multi-parameter
models was finally proposed in a 1970 in a Finnish masters thesis by
~\citet{linnainmaa1970representation}. Techniques from control theory
were applied to develop a means of propagating error backward through
models which could be described as directed graphs. The idea was
applied to neural networks by a Harvard student Paul Werbos ~\citep{werbos1974beyond} and refined in later publications. 

The final essential puzzle piece for neural network models was to take
advantage of their layered structure, which would allow
backpropagation computations at a given layer to be done in
parallel. This key insight, indeed the core of much of modern
computing, was a description of parallel and distributed processing in
the context of cognition by 
~\citet{mcclelland1986parallel} with an astonishing 22,453 citations (a
number that grows nearly every day). With these pieces in place, the
world was ready for someone to finally apply neural network models to
a relevant problem. In 1989, Yann LeCun and a group at Bell Labs
managed to do just that. Motivated by a poorly solved practical
problem -- recognition of handwriting on bank checks,
~\citet{lecun1989backpropagation} refined backpropagation into the
form used today, invented Convolutional Neural Networks \ref{cnn}
~\citep{lecun1995convolutional}, and by  1998, he had worked with his
team to implement what rapidly became the industry standard for banks
to recognize hand-written numbers ~\citep{lecun1998gradient}.

It is worth noting a couple of key ingredients that led to LeCun's
success. First, the problem context itself was both simple and
extremely rich in training data. Second, bank checks are protected both by customers' desire to write correct checks, and their
consistent formatting. Finally, the computing resources needed had only just
reached sufficient specifications and indeed, LeCun's team had rare
access to this level of resources. This combination along with the
rapid drive toward automation at the turn of the millennium left
fertile ground for this discovery. This also removed one crucial
threat to this system, adversarial attacks. Banks are both protected
legally from the crime of fraudulent checks, and practically by the
paucity of advanced computing resources at the time. While this
protected nascent machine-learning implementations, time will
gradually bring technological advancements that will unseat these
controlled conditions. 

Through the early 2000s, along with the internet hitting its stride,
neural networks have quietly become ubiquitous while  remaining
relegated to image recognition problems. Research and time was spent on increasing
speed and efficiency. To keep these increasingly complex and structured models able
to scale with the growing data available, Hinton and Bengio
distinguished themselves in these middle steps. 
~\citep{bengio2007greedy, hinton2006reducing, hinton2006fast}
While many observers still treated
them as toy models inferior to traditional modeling,  the
industrial success of these models was beginning to fuel a new wave of
serious theoretical and practical work in the field. A generation of
household names including Collobert, Hinton, Bengio, and Schmidhuber
came to distinguish themselves alongside LeCun
~\citep{coates2011analysis,
vincent2010stacked,
boureau2010learning,
hinton2010practical,
glorot2010understanding,
erhan2010does,
bengio2009learning}.
With this expansion in theory came a boom in the ability to
practically implement more structured and capable neural networks: recurrent networks ~\citep{mikolov2010recurrent},
convolutional neural networks ~\citep{lee2009convolutional}, natural
language processing ~\citep{collobert2011natural}, and
Long-Short-Term-Memory (LSTM) networks originally developed by
Horchreiter (a student of Shmidhuber) in
1997 ~\citep{hochreiter1997long} to actively understand time-series,
including a solution to the problem of vanishing
gradients problem whereby gradients computed by back propagation which
constitutes a large product of small numbers especially for early
layers in deep networks. This problem naturally arises for recurrent
networks and many approaches that address time-series. The solution
Horchreiter provided, the addition of residual connections to past
parameter states. This insight has made the LSTM one of the most cited
neural networks. This in combination with the surprising result that
Rectified Linear Units (ReLUs) could also solve the vanishing gradient
problem allowed for much deeper and more sophisticated neural networks
to be implemented than ever before. 

In step with this growing theoretical interest came expansion of
well-maintained libraries for working with neural networks including
the very early creation of the now famous Torch
~\citep{Collobert2002TorchAM} and its python interface that still
dominates the market-share of machine learning: PyTorch
~\citep{pytorch2019} with which most of the results in this
work have been computed. 

While neural networks had still not breached the mainstream of pop
culture, they had carved out an undeniable niche by 2009. The
field of computer science was ready to test what they could
do. With all of the ingredients in place, 2009-2010 saw competitions
across Machine Learning (ML) tasks go viral. The ImageNet Large Scale Visual Recognition
Challenge (ILSVRC) \url{http:/image-net.org/challenges/LSVRC/}), Segmentation of Neuronal Structures in Electron
Microscopy Stacks, Traffic sign recognition (IJCNN2012) and more. The
era became defined by these competitions which served to not only gain 
visibility for the field and its strongest participants, but also
rapidly push multiple ML applications up to the point of practical
utility. Schimdhuber's group
~\citep{SCHMIDHUBER201585}, and a similar group at Google dominated
many of these competitions. The cutting-edge became represented by
networks like Inception v4 designed by Google for image classification
which contains approximately 43 million parameters
~\citep{szegedy2013}. Early versions of this network took 1-2 million
dollars worth of compute-time to train. Artificial Neural Networks
(ANNs) now appear in nearly every industry from devices which use ANNs
to intelligently adapt their performance, to the sciences which rely
on ANNs to eliminate tedious sorting and identification of data that
previously had to be relegated to humans. Recently natural language
processing has received its own renaissance, led by Chat-Bots based on
the popular transformer architecture proposed by
~\citet{vaswani2017attention}. Neural network based models are here to
stay, but as these tools expand so wildly in application, we must
begin to ask hard questions about their limitations and implications. 

\subsection{Structure}
In this subsection we give a mathematical description of artificial neural networks. 
A \emph{neuron} is a nonlinear operator that takes input in $\R^n$ to $\R$, historically designed to emulate the activation characteristics of an organic neuron. A collection of neurons that are connected via a (usually directed) graph structure are known as an \emph{Artificial Neural Network (ANN)}. 

The fundamental building blocks of most ANNs are artificial neurons which we will refer to as \emph{perceptrons}.

\begin{definition}{A \textbf{perceptron} is  }
\label{perceptron}
a function $P_{\vec w}: \R^n \to \R$ which has \emph{weights} $\vec
w \in \R^n$ corresponding with each element of an input vector $\vec
x\in \R^n$ and a bias $b \in \R$:
\[P_{\vec w}(\vec x) = f\left(\ip{\vec w,\vec x} + b\right)\]
\[P_{\vec w}(\vec x) = f\left(b + \sum_{i = 1}^n w_i x_i\right)\]
where $f: \R \to \R$ is continuous. The function $f$ is called the \textbf{activation function} for $P$. 
\end{definition}

The only nonlinearity in $P_w$ is contained in $f$. If $f$ is chosen
to be linear, then $P$ will be a linear operator. Although this has
the advantage of simplicity, linear operators do not perform well on
nonlinear problems like classification. For this reason, activation
functions are generally chosen to be nonlinear. Historically,
heaviside functions were used for activation, later replaced based on
work by ~\citet{malik1990preattentive} with sigmoids for their smoothness, switching structure, and convenient
compactification of the output from each perceptron.  It was recently
discovered that a simpler nonlinear function, the \emph{Rectified
  Linear Unit (ReLU)} works as well or better in most
neural-network-type applications according to ~\citet{glorot2011deep}
and additionally training algorithms on ReLU activated networks
converge faster according to ~\citet{nair_rectified_nodate}. 

\begin{definition}{The Rectified Linear Unit (ReLU) function is}
\label{relu}

  \[\relu(x) = \begin{cases} 0, & x \leq 0;\\
      x, & x > 0,\end{cases}\]
\end{definition}

~\citet{petersen2018optimal} demonstrated that this single nonlinearity of this activation function
at $x = 0$ is sufficient to guarantee existence of $\epsilon$ approximation of smooth functions by an ANN composed of sufficiently numerous perceptrons connected by ReLU . In addition, ReLU is convex, which enables efficient numerical approximation of smooth functions in shallow networks ~\citep{li2017convergence}.

In general ANNs 
must not be cyclic and, for convenience, are often arranged into
independent layers. An early roadblock for neural networks was a proof
by ~\citet{minsky1969perceptrons} that single layers of perceptrons
could not encode exclusive-or. ~\citet{kak1993training} demonstrated
that depth, the number of layers in a neural network, is a key factor in its ability to approximate complicated functions including exclusive-or. For this reason, modern ANNs are usually composed of many layers (3-100). The most common instance of a neural network model is a fully connected \emph{feed forward (FF)} configuration. In this configuration data enters as an input layer which is fed into each of the nodes in the first layer of neurons. Output of the first layer is fed into each of the nodes in the second layer, and so on until the output of the final layer is fed into an output filter which generates the final result of the neural network.

In this example of a FF network, an input vector in $\R^7$ is mapped to a
an output in $\R^3$ which is fed into a classifier. Each blue circle
represents a perceptron with the ReLU activation function. 


\scalebox{.9}{
\begin{tikzpicture}[shorten >=1pt,->,draw=black!50, node distance=\layersep]
    \tikzstyle{every pin edge}=[<-,shorten <=1pt]
    \tikzstyle{neuron}=[circle,fill=black!25,minimum size=9pt,inner sep=0pt]
    \tikzstyle{input neuron}=[neuron, fill=green!50];
    \tikzstyle{output neuron}=[neuron, fill=red!50];
    \tikzstyle{hidden neuron}=[neuron, fill=blue!50];
    \tikzstyle{annot} = [text width=4em, text centered]

    \foreach \name / \y in {1,...,7}
        \node[input neuron] (I-\name) at (0,-\y) {};
    \foreach \name / \y in {1,...,6}
        \path[yshift=-0.5cm]
            node[hidden neuron] (H-\name) at (\layersep,-\y cm) {};

    \foreach \name / \y in {1,...,4}
        \path[yshift=-1.5cm,xshift=2.0cm]
            node[hidden neuron] (HH-\name) at (\layersep,-\y cm) {};

    \foreach \name / \y in {1,...,3}
        \path[yshift=-2cm,xshift=4.0cm]
            node[output neuron] (O-\name) at (\layersep,-\y cm) {};


    \foreach \source in {1,...,7}
        \foreach \dest in {1,...,6}
            \path (I-\source) edge (H-\dest);

    \foreach \source in {1,...,6}
        \foreach \dest in {1,...,4}
            \path (H-\source) edge (HH-\dest);

    \foreach \source in {1,...,4}
        \foreach \dest in {1,...,3}
            \path (HH-\source) edge (O-\dest);

  \node [rectangle, draw, minimum height=6.2cm, text width=.8cm, text
  centered, left =.8cm of I-4] (mm) {Data};

    \foreach \source in {1,...,7}
        \path [line] (mm.east|-I-\source) -- (I-\source);

    \node[annot,above of=H-1, node distance=2cm] (hl) {Layer 1};
    \node[annot,left of=hl] {Input };
    \node[annot,right of=hl] (h3) {Layer 2} ;
    \node[annot,right of=h3] {Output Layer};
  \node [rectangle, draw, minimum height=5cm, text width=1.6cm, text
  centered, right =6.8cm of I-4] (mc) {Classifier};
    \foreach \source in {1,...,3}
        \path [line] (O-\source) -- (mc.west|-O-\source);

\end{tikzpicture}
}



The output of this ANN is fed into a classifier. To complete this
example, we can define the most common classifier, Softmax:

\begin{definition}{Softmax (or the normalized exponential) is the function given by}
\[s : \R^n \to [0,1]^n\]
\[s_j(\vec x) = \frac{e^{x_j}}{\sum_{k = 1}^n e^{x_k}}\]
\end{definition}

\begin{definition}{We can define a classifier which picks the class corresponding with the largest output element from Softmax: }
\[\text{(Output Classification)  }   c_s(\vec x) = \text{argmax}_{i} s_i(\vec{x})\]
\end{definition}
During training, the output $y \in \R^n$ from a network can thus be
compressed using softmax into $[0,1]^n$ as a surrogate for probability
for each possible class or directly into the classes which we can
represent as the simplex for the vertices of $[0,1]^n$
\citep{Bishop:2006:PRM:1162264}. 

\subsubsection{Convolutional Neural Networks (CNNs)}\label{cnn}

Another common type of neural network which is a component in many modern applications including one in the experiments to follow are Convolutional Neural Networks (CNNs). CNNs are fundamentally composed of
perceptrons, but each layer is not fully connected to the
next. Instead, layers are arranged spatially and overlapping groups of
perceptrons are independently connected to the nodes of the next
layer, usually with a nonlinear filter that computes the maximum of
all of the incoming nodes to a new node. This structure has been shown
(e.g. by ~\citet{lecun1995convolutional}) to be very effective on problems with spatial information. 

\subsection{Training ANNs}

Neural networks consist of a very large number of perceptrons with many parameters. Directly solving the system implied by these parameters and the empirical risk minimization problem defined below would be difficult, so we must use a
modular approach which takes advantage of the simple and regular structure of ANNs.

 A breakthrough came with the application of techniques derived from
 control theory to ANNs in the late 1980s by ~\citet{rumelhart1986learning}, dubbed backpropagation. This technique was refined into its modern form in the thesis and continuing work of ~\citet{lecun1988theoretical}. In this method, error is propagated backward taking advantage of the directed structure of the network to compute a gradient for each parameter defining it. Because modern ANNs are usually separated into discrete layers, gradients can be computed in parallel for all perceptrons at the same depth of the network
\citep{Bishop:2006:PRM:1162264}. Leveraging modern GPUs and parallel computing technologies, these gradients can be computed very quickly. There are a number of important considerations in training. We discuss a few in the following subsections. 

\subsubsection{Selection of the Training Set}

The first step in training an ANN is the selection of a training set. ANNs fundamentally are universal function approximators: Given a set of input data and corresponding output data, they approximate a mapping from one to the other. Performance is dependent on how well the phenomenon we hope to model is represented by the training data. The training data must consist of a set of inputs (e.g., images) and a set of outputs (e.g., labels) which contain sufficient examples to characterize the intended model. In a way, this is how we pose a question to the neural network. One must always ask whether the question we wish to pose is well-expressed by the training data we have available. 

The most important attributes of a training dataset are the number of
samples it contains and its density near where the model will be making predictions. According to conventional wisdom, training a neural network with $K$ parameters
will be very challenging if there are fewer than $K$ training samples
available. The modular structure of ANNs can be combined with regularization of the weights to overcome these limitations ~\citep{liu2015very}. 
In general, we will denote a training set by $(X,Y)$ where $X$ is an indexed set of inputs and $Y$ is a corresponding indexed set of labels. 

\subsubsection{Selecting a Loss Function}

Once we have selected a set of training data (both inputs and outputs), we must decide how we will evaluate the match between the ANNs output and the defined outputs from the training dataset -- we will quantify the deviation of the ANN compared with the given correspondence as a Loss. In general \emph{loss functions} are nonzero functions which compare an output $y$ against a ground-truth $\hat y$. Generally they have the property that an ideal outcome would have a loss of 0. 

One commonly used loss function for classification is known as Cross-Entropy Loss:
\begin{definition}{The Cross-Entropy Loss comparing two possible outputs is}
$L(y,\hat y) = -\sum_i y_i \log \hat y_i$.
\end{definition}
Other commonly used loss functions include $L^1$ loss (also referred to as Mean Absolute Error (MAE)), $L^2$ loss (often referred to as Mean-Squared-Error (MSE)), and Hinge Loss (also known as SVM loss). 

To set up the optimization, the loss for each training example must be aggregated. Generally, ANN training is conducted via Empirical Risk Minimization where Empirical Risk is defined for a given loss function $L$ as follows:
\begin{definition}{Given a loss function $L$, the Empirical Risk over a training dataset $(X,Y)$ of size $N$ is }
\[R_{\text{emp}}(P_{\vec w}(x)) = \dfrac{1}{N} \sum_{(x,y) \in (X,Y)} L(P_{\vec w}(x)), y).\]
\end{definition}
We seek parameters $\vec w$ which will minimize $R_{\text{emp}}(P_{w}(x))$. This will be done with gradient-based optimization. 

\subsubsection{Computation of Gradient via Backpropagation}

Since it is relevant to the optimization being performed, we will briefly discuss the computation of gradients via backpropagation. For this discussion, we will introduce a small subset of a neural network in detail. In general, terms will be indexed as follows:
\[ x^{\text{[layer]}}_{\text{[node in layer], [node in previous layer]}}\]
When the second subscript is omitted, the subscript will only index the node in the current layer to which this element belongs. 

\scalebox{.9}{
\begin{tikzpicture}[shorten >=1pt,->,draw=black!50, node distance=\layersep]

\node[circle, minimum size=19pt, fill=black!25, inner sep=0pt] (n11) at (0,2) {$a^1_1$};
\node[circle, minimum size=19pt, fill=black!25, inner sep=0pt] (n12) at (0,0) {$a^1_2$};
\node[circle, minimum size=19pt, fill=black!25, inner sep=0pt] (n21) at (4,2) {$a^2_1$};
\node[circle, minimum size=19pt, fill=black!25, inner sep=0pt] (n22) at (4,0) {$a^2_2$};
\node[circle, minimum size=19pt, fill=black!25, inner sep=0pt] (n31) at (8,2) {$a^3_1$};
\node[circle, minimum size=19pt, fill=black!25, inner sep=0pt] (n32) at (8,0) {$a^3_2$};

\node (av1) at (0,2.9) {$\Bar{a}^1$};
\node (av2) at (4,2.9) {$\Bar{a}^2$};
\node (av3) at (8,2.9) {$\Bar{a}^3$};

\node (ai1) at (0,3.9) {Index: $i$};
\node (ai2) at (4,3.9) {Index: $\alpha$};
\node (ai3) at (8,3.9) {Index: $\lambda$};

\node (w2) at (2.6,2.9) {$W^2$};
\node (w3) at (6.6,2.9) {$W^3$};

\draw[- triangle 45] (n11)  -- node[rotate=0,shift={(0.3,0.3)}] {$w^2_{1,1}$} (n21);
\draw[- triangle 45] (n11)  -- node[rotate=0,shift={(0.6,0.65)}] {$w^2_{1,2}$} (n22);
\draw[- triangle 45] (n12)  -- node[rotate=0,shift={(0.3,-0.65)}] {$w^2_{2,1}$} (n21);
\draw[- triangle 45] (n12)  -- node[rotate=0,shift={(0.6,-0.3)}] {$w^2_{2,2}$} (n22);

\draw[- triangle 45] (n21)  -- node[rotate=0,shift={(0.3,0.3)}]  {$w^3_{1,1}$} (n31);
\draw[- triangle 45] (n21)  -- node[rotate=0,shift={(0.6,0.65)}] {$w^3_{1,2}$} (n32);
\draw[- triangle 45] (n22)  -- node[rotate=0,shift={(0.3,-0.65)}]  {$w^3_{2,1}$} (n31);
\draw[- triangle 45] (n22)  -- node[rotate=0,shift={(0.6,-0.3)}]  {$w^3_{2,2}$} (n32);
\end{tikzpicture}
}

In this diagram, the $W^l$ are matrices composed of the weights
indexed as above. Given an activation function for layer $n$, $A^n$ and
its element-wise application to a vector $\bar A^n$. As demonstrated
by ~\citet{Krause20}, we can now write the output $\bar a^n$ for any layer of an arbitrary ANN in two ways . Recursively, we can define
\begin{equation}
    a^n_\lambda = A^n(\sum_\alpha w^n_{\alpha, \lambda} a^{n-1}_\alpha)
\end{equation}
We can also write the matrix form of this recursion for every node in the layer:
\begin{equation}
\bar a^n = \bar A^n (W^n(\bar a^{n-1} ) )
\end{equation}
The matrix form makes it easier to write out a closed form for the output of the neural network. 
\begin{equation}
\bar a^n = \bar A^n (W^n(\bar A^{n-1}( W^{n-1} ( \cdots ( \bar A^{2} ( W^{2} \bar a^1) ) \cdots ) ) ) )
\end{equation}

Now, given a loss function $L = \sum_{i} \ell_i(a^n_i)$ where each $\ell_i$ is a loss function on the $i^{\text{th}}$ element of the output, we wish to compute the derivatives $\dfrac{\partial L}{\partial_{w^l_{i,j}}}$ for every $l, i,$ and $j$ which compose the gradient $\nabla L$. Using the diagram above, we can compute this directly for each weight using chain rule:
\begin{align*}
    \dfrac{\partial L}{\partial w^3_{\lambda,\alpha}} &= \dfrac{\partial L}{\partial a^3_{\lambda}} \dfrac{\partial a^3_{\lambda}}{\partial w^3_{\lambda,\alpha}} = \sum_{\lambda=1}^n \ell'_\lambda( a^3_\lambda) (A^3)' (\sum_{\alpha=1}^n w^3_{\alpha, \lambda} a_\alpha^2) a^2_\alpha.    
\end{align*}
Many of the terms of this gradient (e.g. the activations $a^n_i$ and the sums $\sum_{i} w^n_{i,j} a_i$) are computed during forward propagation when using the network to generate output. 
We will store such values during the forward pass and use a backward pass to fill in the rest of the gradient. Furthermore, notice that $\ell'_\lambda$ and $(A^n)'$ are well understood functions 
whose derivatives can be computed analytically almost everywhere. We can see that all of the partials will be of the form 
$\dfrac{\partial L}{\partial w^l_{n, i}} = \delta^l_n a^l_i$ where $\delta^l_n$  will contain terms which are either pre-computed or can be computed analytically. Conveniently, we can define this error signal recursively: 
\[
\delta^l_n = A'^l (a^l_{n}) \sum_{i = 1}^n w^{l+1}_{i, n} \delta^{l+1}_i
\]
In matrix form, we have
\[\bar \delta^l = \bar A'^l(W^l \bar a^l) \odot ((W^{l+1})^T \bar \delta^{l+1})\]
Where $\odot$ signifies element-wise multiplication. 

Then we can compute the gradient with respect to each layer's matrix $W^l$ as an outer product: 
\[\nabla_{W^l} L = \bar \delta^l \bar a^{(l-1)T}.\]
Since this recursion for layer $n$ only requires information from layer $n+1$, this allows us to propagate the error signals that we compute backwards through the network.

\subsubsection{Optimization of Weights}

Given a set of training input data and a method for computing gradients, our ultimate goal is to iteratively run our training-data through the network, updating weights gradually according to the gradients computed by backpropagation. In general, we start with some
default arrangement of the weights and choose a step
size $\eta$ for gradient descent. Then for each weight, in each iteration
of the learning algorithm, we apply a correction so that 
\[w'_{i',j',k'} = w_{i',j',k'}-\eta \frac{\partial E(Y,\hat Y)}{\partial
    w_{i',j',k'}}\]
    In this case, the step size (learning rate) $\eta$ is fixed throughout training.
    Numerical computation of the gradient requires first evaluating the network forward by computing the output for a given input. The value of every node in the network is saved and these values are used to weight the error as it is propagated backward through the network. Once the gradient is computed, the weights are adjusted according to the step defined above. This process is repeated until convergence is attained to within a tolerance. It should be
clear from the number of terms in this calculation that the initial
guess and step size can have significant effect on the eventual trained weights.
Due to lack of a guarantee for general convexity,
~\citet{Bishop:2006:PRM:1162264} observed that poor guesses for such a
large number of parameters can lead to gradients blowing up or
down. Due to non-linearity and the plenitude of local minima in the
loss function, classic gradient descent usually does not perform well
during ANN training. \\  
By far the most common technique for training the weights of neural networks adds noise in the form of random re-orderings of the training data to the general optimization process and is known as stochastic gradient descent. 
\begin{definition}{Stochastic Gradient Descent (SGD)}

Given an ANN $N: \R^n \to C$, an initial set of weights for this network $\vec w_0$ (usually a small random perturbation from 0), a set of training data $X$ with labels $Y$, and a learning rate $\eta$, the algorithm is as follows: 

\begin{algorithm}
\caption*{Batch Stochastic Gradient Descent}\label{sgd}
\begin{algorithmic}[ht]
\State $w = w_0$
\While{$E(\hat Y, P_w(X))$ (cumulative loss) is still improving} \Comment{ (the stopping condition may require that the weight change by less than $\e$ for some number of iterations or could be a fixed number of steps)}
\State Randomly shuffle $(X,Y)$
\State Draw a small batch $(\hat X, \hat Y) \subset (X, Y)$
\State $w \leftarrow w - \eta \left(\sum_{(x,y) \in (\hat X, \hat Y)}  \nabla L(P_w(\hat x), \hat y)\right)$
\EndWhile
\end{algorithmic}
\end{algorithm}
\end{definition}
Stochastic gradient descent achieves a smoothing effect on the gradient optimization by only sampling a subset of the training data for each iteration. Miraculously, this smoothing effect not only often achieves faster convergence, the result also generalizes better than solutions using deterministic gradient methods 
\citep{HardtRS15}. It is for this reason that SGD has been adopted as the de facto standard among ANN training applications.

\chapter{Adversarial Attacks}

\label{Chapter2} 

Deep Neural Networks (DNNs) and their variants are core to the success
of modern machine learning as summarized by ~\citet{prakash2018}. They
have dominated competitions in image processing, optical character
recognition, object detection, video classification, natural language
processing, and many other fields ~\citep{SCHMIDHUBER201585}. Ten years
ago an interesting property of
such networks was observed by ~\citet{szegedy2013}. Their approach was to define a loss function
relating the output of the ANN for a given initial image to a target adversarial 
output plus the $L^2$-norm of the input and use backpropagation to
compute gradients -- not on the weights of the neural network, but on
just the input layer to the network. The solution to this optimization
problem, efficiently approximated by a gradient-based optimizer, would
be a slightly perturbed natural input with a highly perturbed
output. Their experimental results are striking, which we can see in
Figure.~\ref{fig:szegedy}.  More mysteriously, these examples
are often transferable -- an attack generated against one
model may succeed against a totally different model. With the
incredible expansion of the application of universal function
approximators in machine-learning, their reliability has come to have
real-world significance. Self-driving cars, including those
manufactured by Tesla Incorporated, use image classification models to
distinguish stop-signs and speed-limit
signs. ~\citet{DBLP:journals/corr/EvtimovEFKLPRS17} have shown that
these models are not robust!. Other machine-learning (ML) models are
increasingly relied upon by the defense intelligence apparatus
~\citep{hutchins2011intelligence}. Social media and search engines
which are now the backbone of the internet use ML increasingly
to determine what content will receive attention. In order to wisely
use these tools, it is crucial that we carefully understand their
limitations. 

Adversarial examples occur when natural data can be perturbed in small
ways in order to produce a similar input which receives a
significantly different model output. ``Small'' in this context may
refer to small in a particular metric or sometimes is referred to in
the context of human perception. It is important to note that Adversarial
examples are not just a peculiarity, but seem to occur for most, if
not all, ANN  classifiers. For example, ~\citet{inevitable2018} used
isoperimetric inequalities on high dimensional spheres and hypercubes
to conclude that there is a reasonably high probability that each
correctly classified data point has a nearby adversarial
example. ~\citet{ilyas2019adversarial} argued that optimized models use
some subtle features for classification which are neither intuitive to
humans nor robust to perturbation. The argue that ML models can
efficiently extract features from training data, but that they do not
connect these features robustly across scales. The prevalence of these
features is illustrated by ~\citet{madry2018towards} with the simple experiment of adding vast
quantities of adversarially perturbed data during training. Although
this method increases adversarial robustness at a cost to prediction
accuracy ~\citep{tsipras2018robustness}, it does not do so very
significantly, and leaves behind vulnerabilities that can still be
reduced to non-robust features ~\citep{inevitable2018}.  


We will take a geometric approach to analyzing 
robustness, both in terms of the models' understanding of underlying
data geometry and by carefully defining the decision boundary of a
model and studying its properties. There have been many
attempts to identify adversarial examples using properties of the
decision boundary.  ~\citet{Fawzi2018empirical} found that decision
boundaries tend to have highly curved regions, and these regions tend
to favor negative curvature, indicating that regions that define
classes are highly nonconvex. These were found for a variety of ANNs
and classification  tasks. 
A related idea is that adversarial examples often arise within cones, outside of which images are classified in the original class, as observed by ~\citet{roth19aodds}. Many theoretical models of adversarial examples, for instance the dimple model developed by ~\citet{shamir2021}, have high curvature and/or sharp corners as an essential piece of why adversarial examples can exists very close to natural examples.

\begin{figure}[ht]
    \centering
\includegraphics[width=7.3cm]{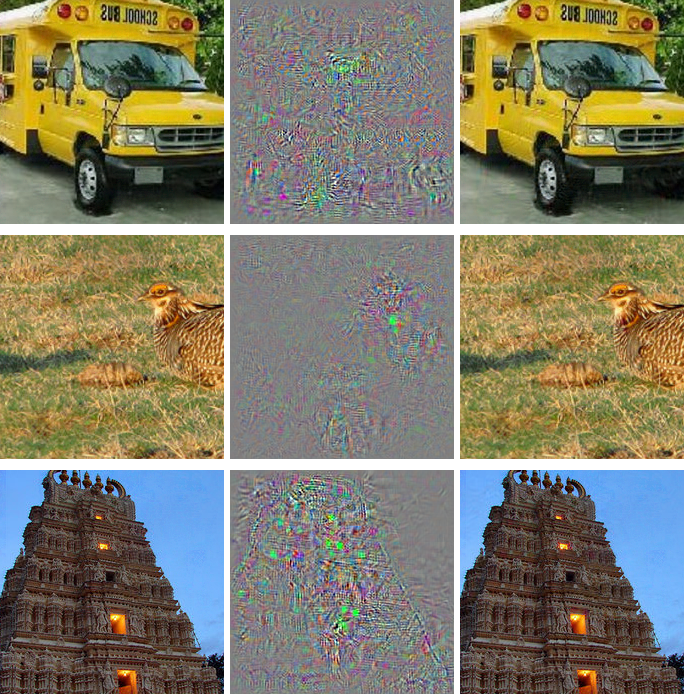}\includegraphics[width=7.3cm]{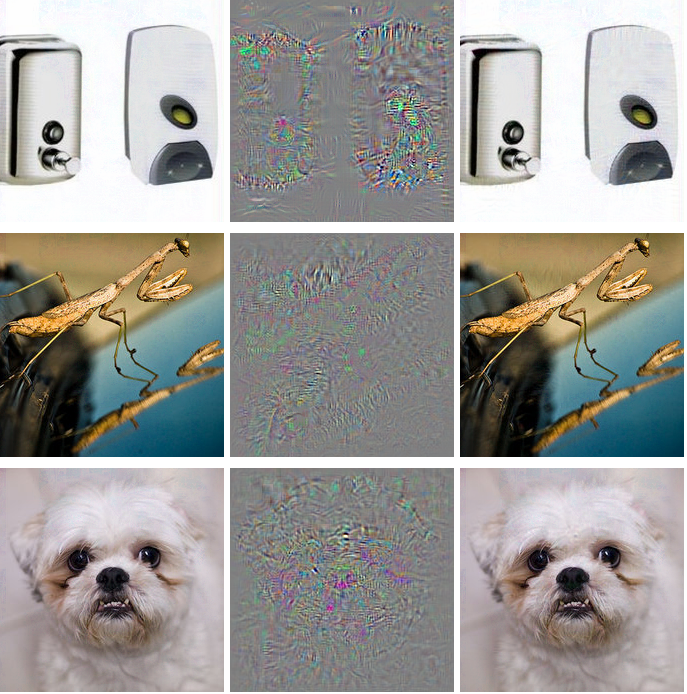}
    \caption{Natural Images are in columns 1 and 4, Adversarial images are in columns 3 and 6, and the difference between them (magnified by a factor of 10) is in columns 2 and 5. All images in columns 3 and 6 are classified by AlexNet as "Ostrich" ~\citep{szegedy2013}.}
    \label{fig:szegedy}
\end{figure}

\section{Common Datasets}

The first step toward understand adversarial attacks understanding
the data on which neural networks are built. We will limit our
investigation mostly to classic image classification problems,
although several of our results will hold more generally. The data set
used above in Figure.~\ref{fig:szegedy} is known as ImageNet -- a
large set of labeled images varying in size originally compiled for
the ImageNet Large Scale Visual Recognition Challenge (ILSVRC
~\citet{ILSVRC15}). This dataset has become a
standard for image classification and feature identification
experiments. In the experiments that follow, ImageNet will be featured
alongside the Modified National Institute of Standards and Technology
dataset (MNIST ~\citet{MNIST}) which is a database of hand written
digits often used to develop image processing and character
recognition systems. This dataset is much lower resolution than
ImageNet and therefore experiments run much more quickly on it and
require less complex input/output.  

\section{Common Attack Techniques}
Adversarial attacks are generally produced by introducing an objective
function. This objective balances achieving a change in predicted
classification with minimizing the perturbation needed to achieve the
desired prediction. The adversarial objective can use \emph{cross-entropy
loss} to compare predictions against a specific target or the negation
of the original model prediction for a given input
~\citep{good1963maximum}. Perturbation size is often measured using a
regularization term in image space (e.g. the $L^2$ norm) which
penalizes the generated adversary for being too far from its starting
point. This loss function is combined with an optimization algorithm
in order to produce an attack technique. 

\subsection{L-BFGS Minimizing Distortion}\label{lbfgs}

The original attack used by ~\citet{szegedy2013} set up a
box-constrained optimization problem whose approximated solution
generates these targeted mis-classifications. We will write this
precisely according to their formulation: \\

Let $f : \R^m \to \{1,...,k\}$ be a classifier and assume $f$ has an
associated continuous loss function denoted by loss$_f : \R^m \times
\{1,...,k\} \to \R^+$ and $l$ a target adversarial class or output. \\
\textbf{ Minimize} $\Norm{r}_2$ subject to:
\begin{enumerate}[1.]
\item $f(x + r) = l$
\item $x + r \in [0,1]^m$
\end{enumerate}

The solution is approximated with L-BFGS (see Appendix \ref{appa}) as
implemented in Pytorch or Keras. This technique yields examples that
are close to their original counterparts in the $L^2$ sense, but are
predicted to be another class by the model with high confidence.  \\



\paragraph{L-BFGS: Mnist}
The following examples are prepared by implementing the above
technique via pytorch ~\citep{pytorch2019} on images from the Mnist
dataset with FC200-200-10, a neural network with 2 hidden layers with
200 nodes each in Figure.~\ref{lbfgsa}:
\begin{figure}[ht]

\includegraphics[trim=200 185 100 200, clip, width=7cm]{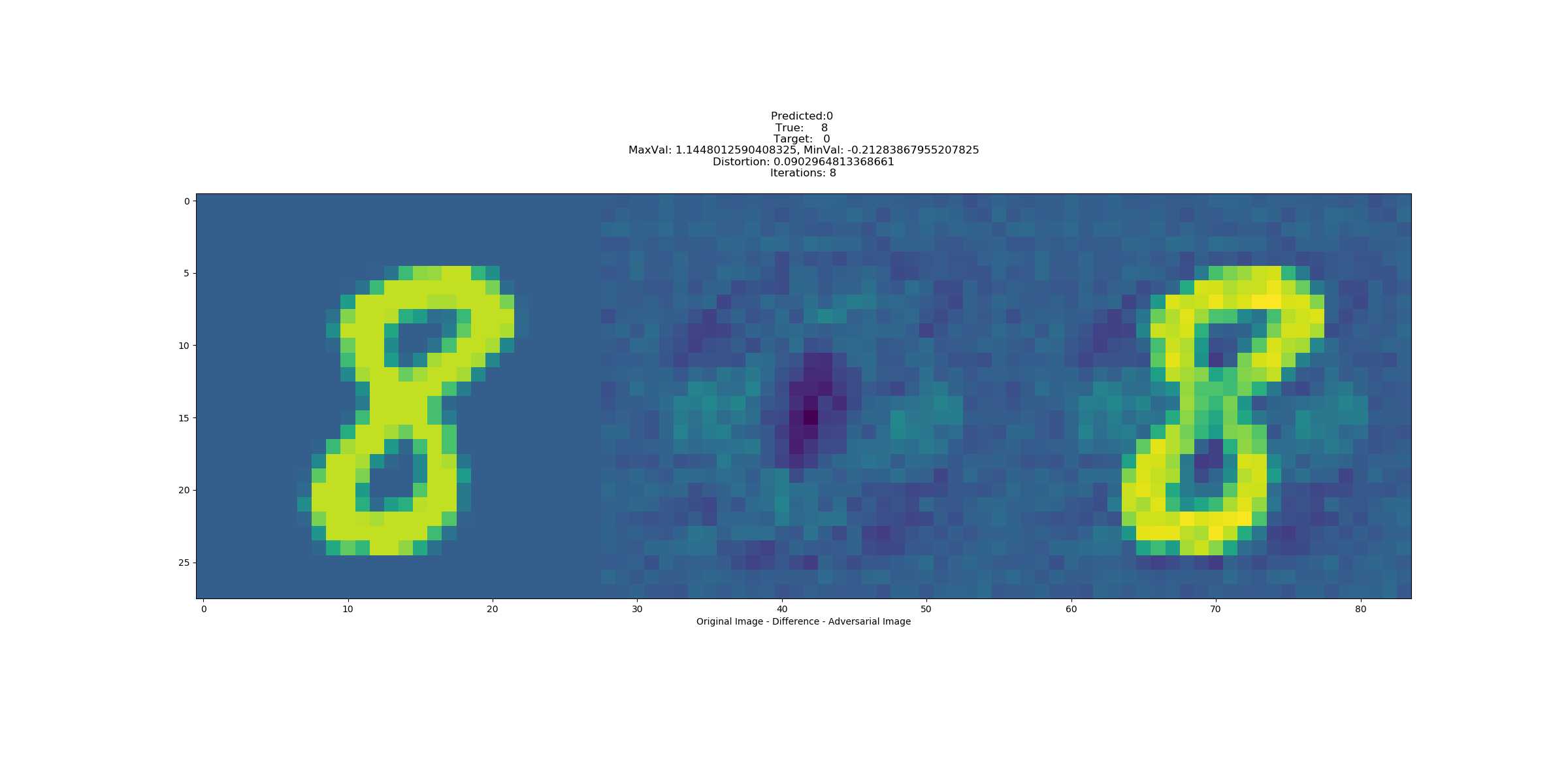}\includegraphics[trim=200 185 100 200, clip,width=7cm]{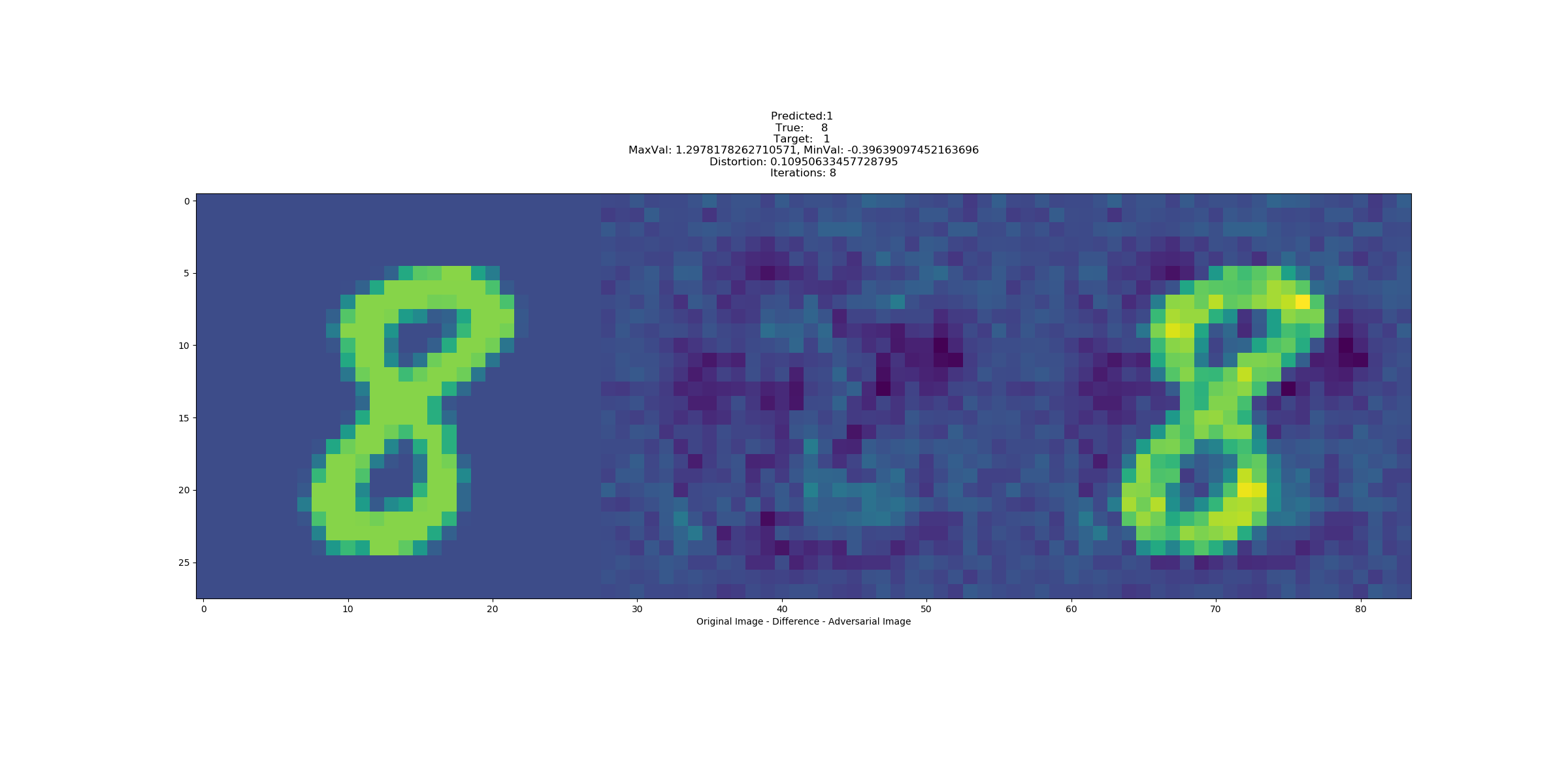}
\includegraphics[trim=200 185 100 200, clip,width=7cm]{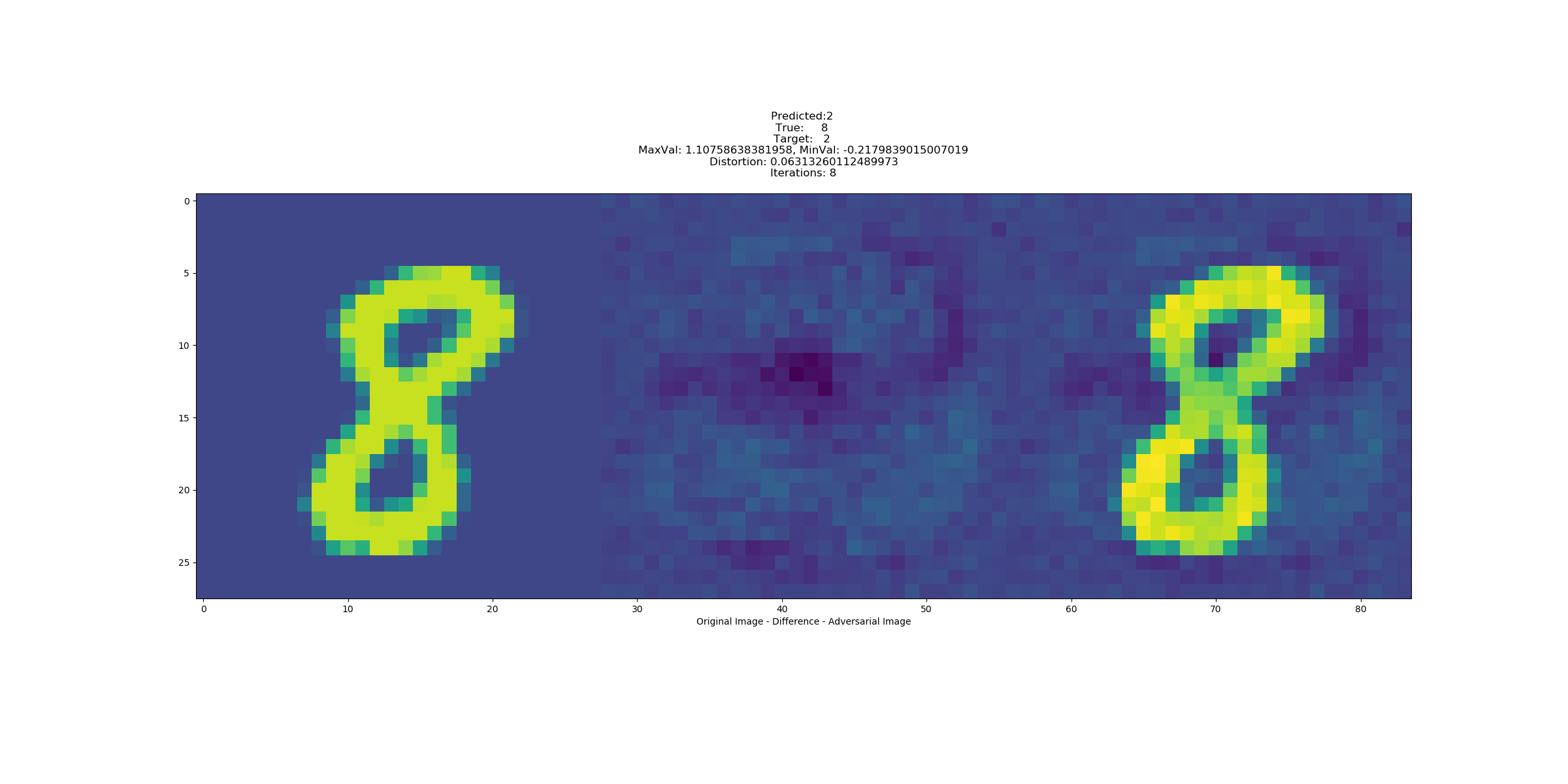}\includegraphics[trim=200 185 100 200, clip,width=7cm]{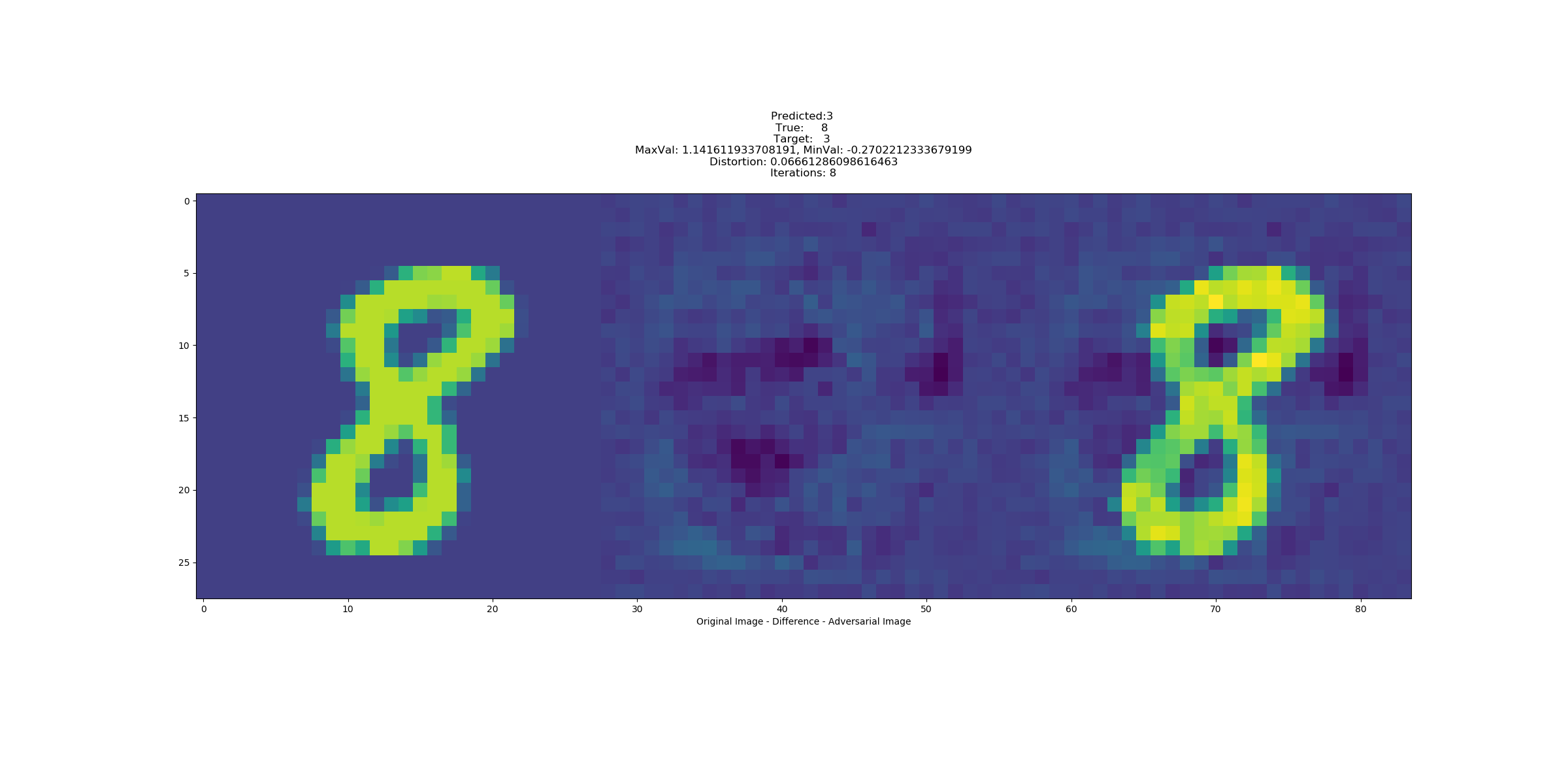}
\includegraphics[trim=200 185 100 200, clip,width=7cm]{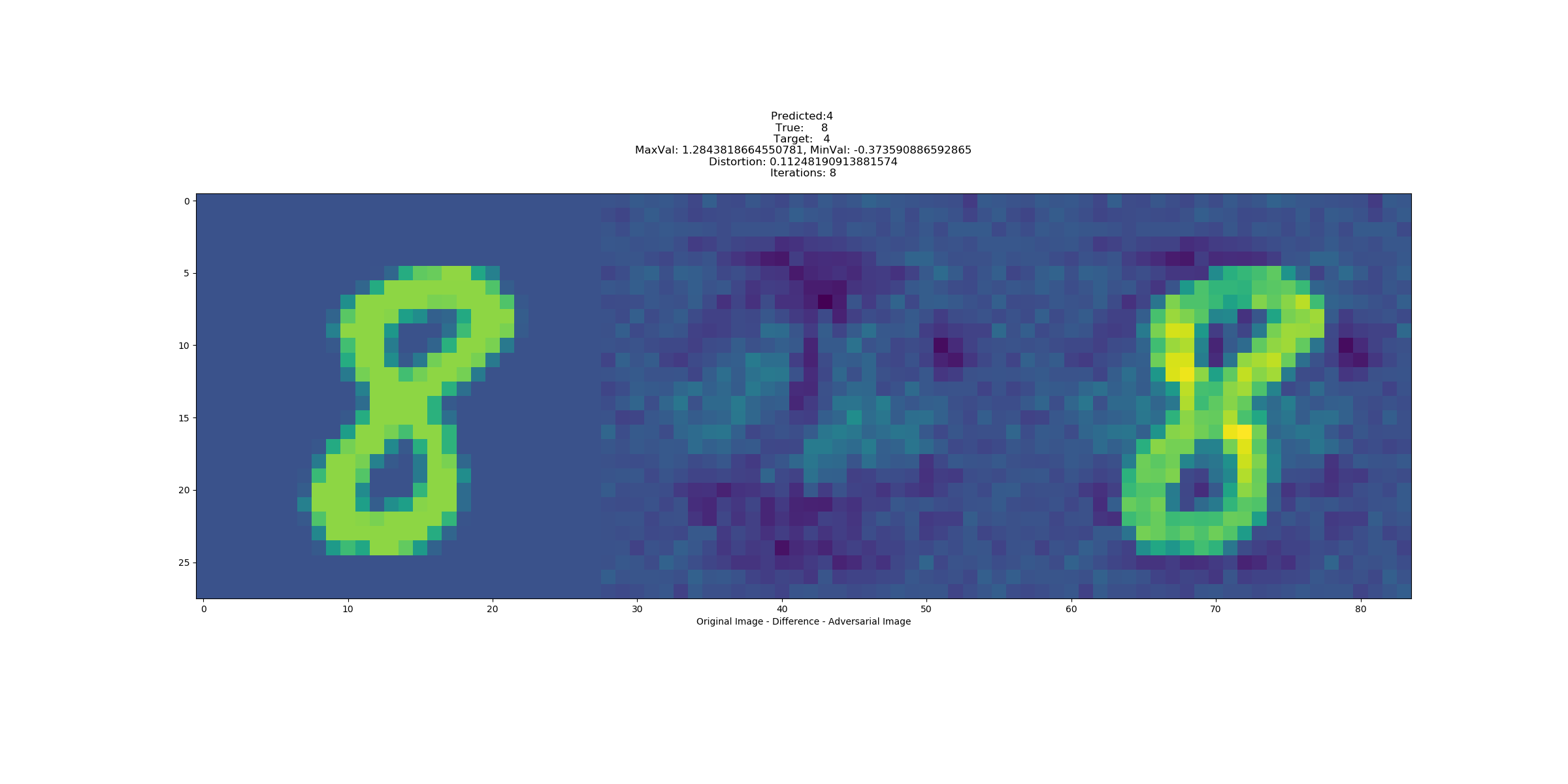}\includegraphics[trim=200 185 100 200, clip,width=7cm]{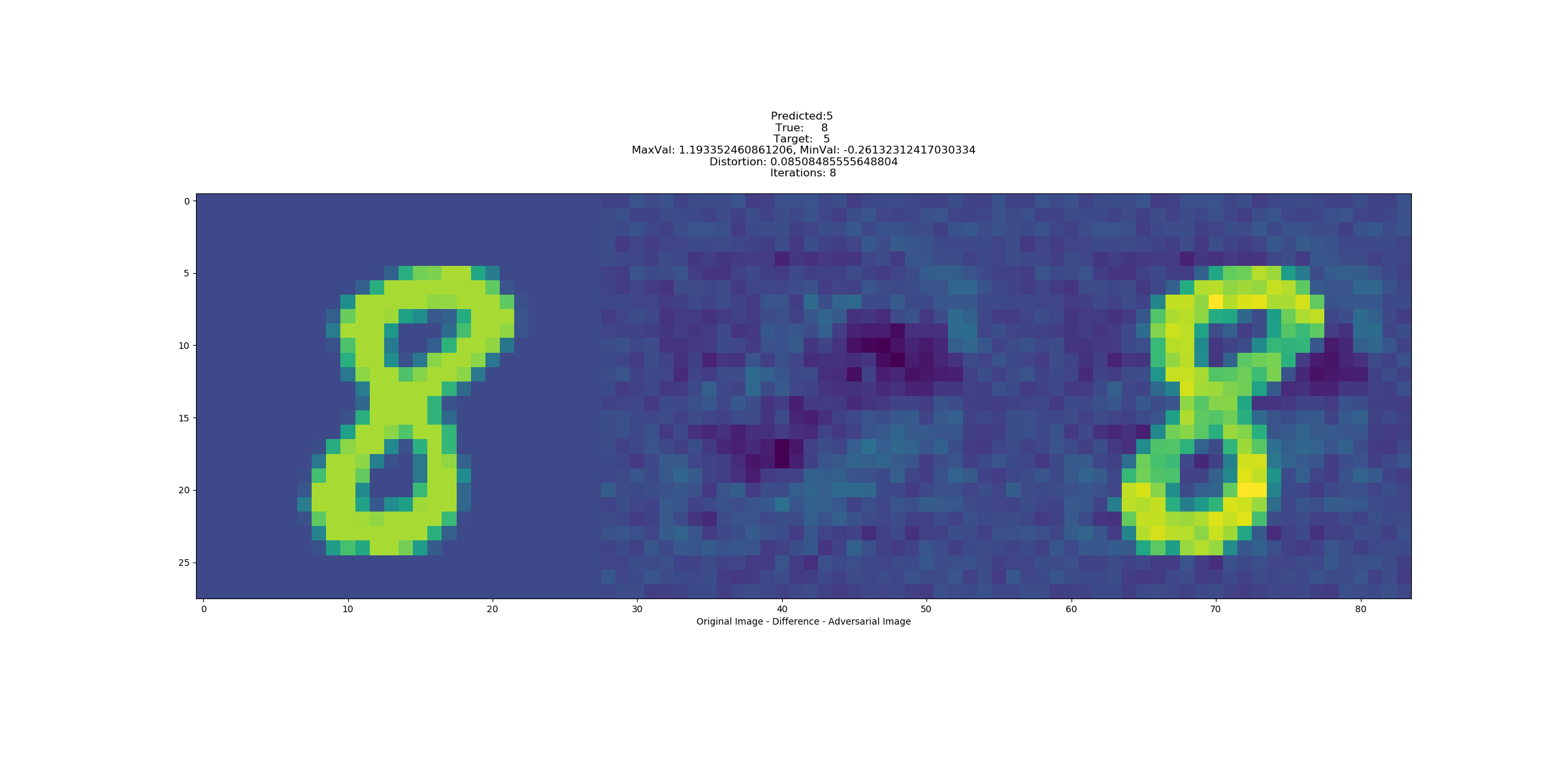}
\caption{Original images on the left, Perturbation is in the middle, Adversarial Image (total of Original with Perturbation) is on the right. Column 1 shows an original 8 being perturbed to adversarial classes 0, 2, and 4. Column 2 shows adversarial classes 1, 3, and 5}
\label{lbfgsa}
\end{figure}
Szegedy et al. define a metric to compare the magnitude of these perturbations:
\begin{definition}{Distortion is the $L^2$ norm of the difference between an original image and a perturbed image, divided by the square root of the number of pixels in the image: }
\[\sqrt{\dfrac{\sum_i  (\hat x_i - x_i)^2}{n}}\]
\end{definition}
Distortion is $L^2$ magnitude normalized by the square-root of the number of dimensions so that values can be compared for modeling problems with differing numbers of pixels. 

900 examples were generated for the network above. We measured an
average distortion of 0.089 with a distribution given in
Figure.~\ref{lbfgsh}. Another histogram is provided for distortions
measured from attacks against the VGG16 (Visual Geometry Group
Network 16) network trained on the ImageNet dataset
in\label{lbfgsi}. This histogram demonstrates that ImageNet networks
are vulnerable to much more subtle adversarial attacks. 

\begin{figure}[ht]
\includegraphics[trim=200 80 100 100, clip, width=16cm]{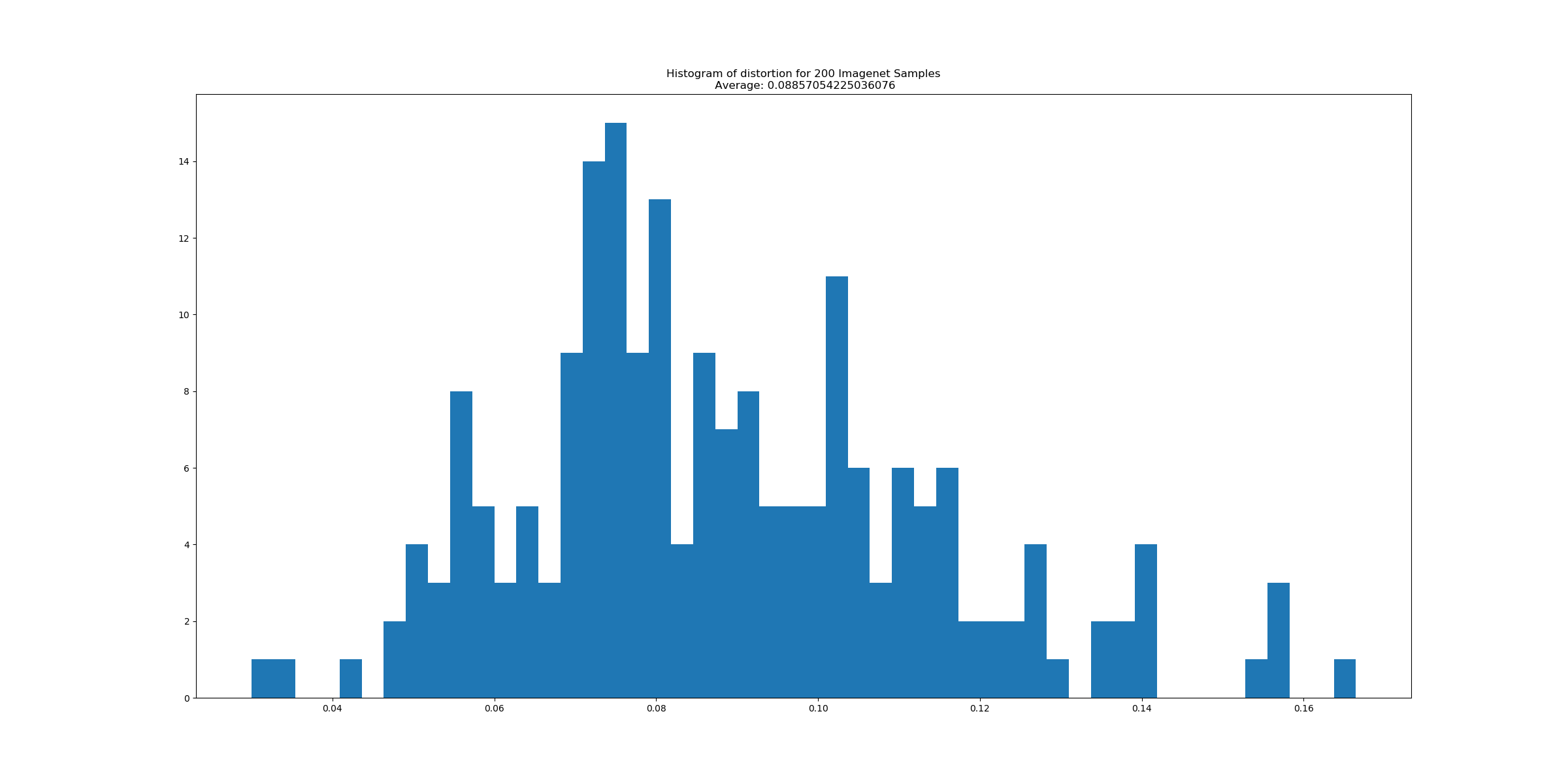}
\caption{A histogram of the distortion measured for each of 900 adversarial examples generated using L-BFGS against the FC-200-200-10 network on Mnist. Mean distortion is 0.089.}
\label{lbfgsh}
\end{figure}

\paragraph{L-BFGS: ImageNet}
\label{lbfgs-s}
We also tried to replicate the results of ~\citet{szegedy2013} on ImageNet. Attacking VGG16, a well known model from the ILSVRC-2014 competition ~\citep{simonyan2014very}, on ImageNet images with the same technique generates the examples in Figure.~\ref{lbfgsis}: 

\begin{figure}[ht]

\includegraphics[trim=200 185 100 200, clip, width=8cm]{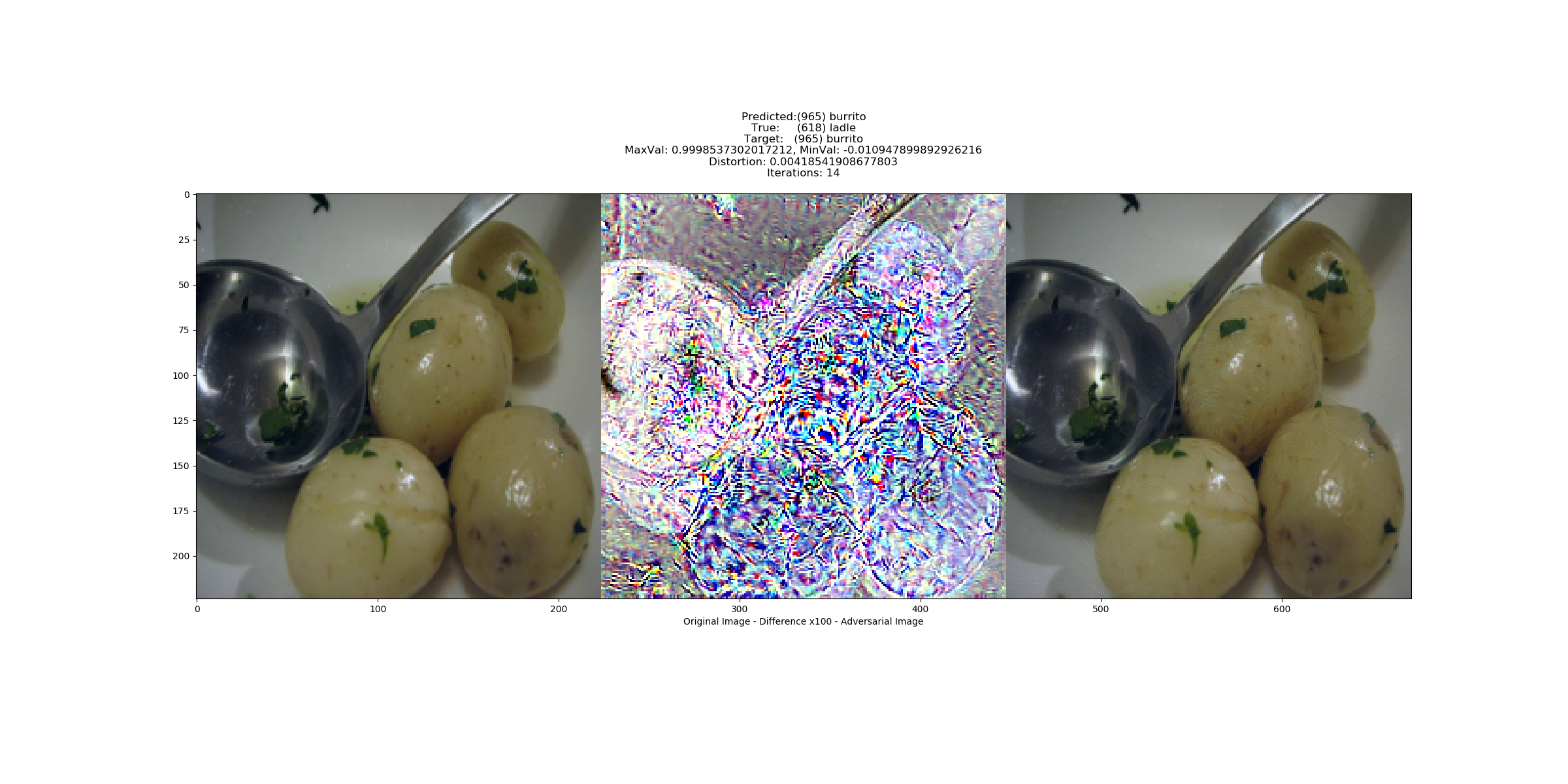}\includegraphics[trim=200 185 100 200, clip, width=8cm]{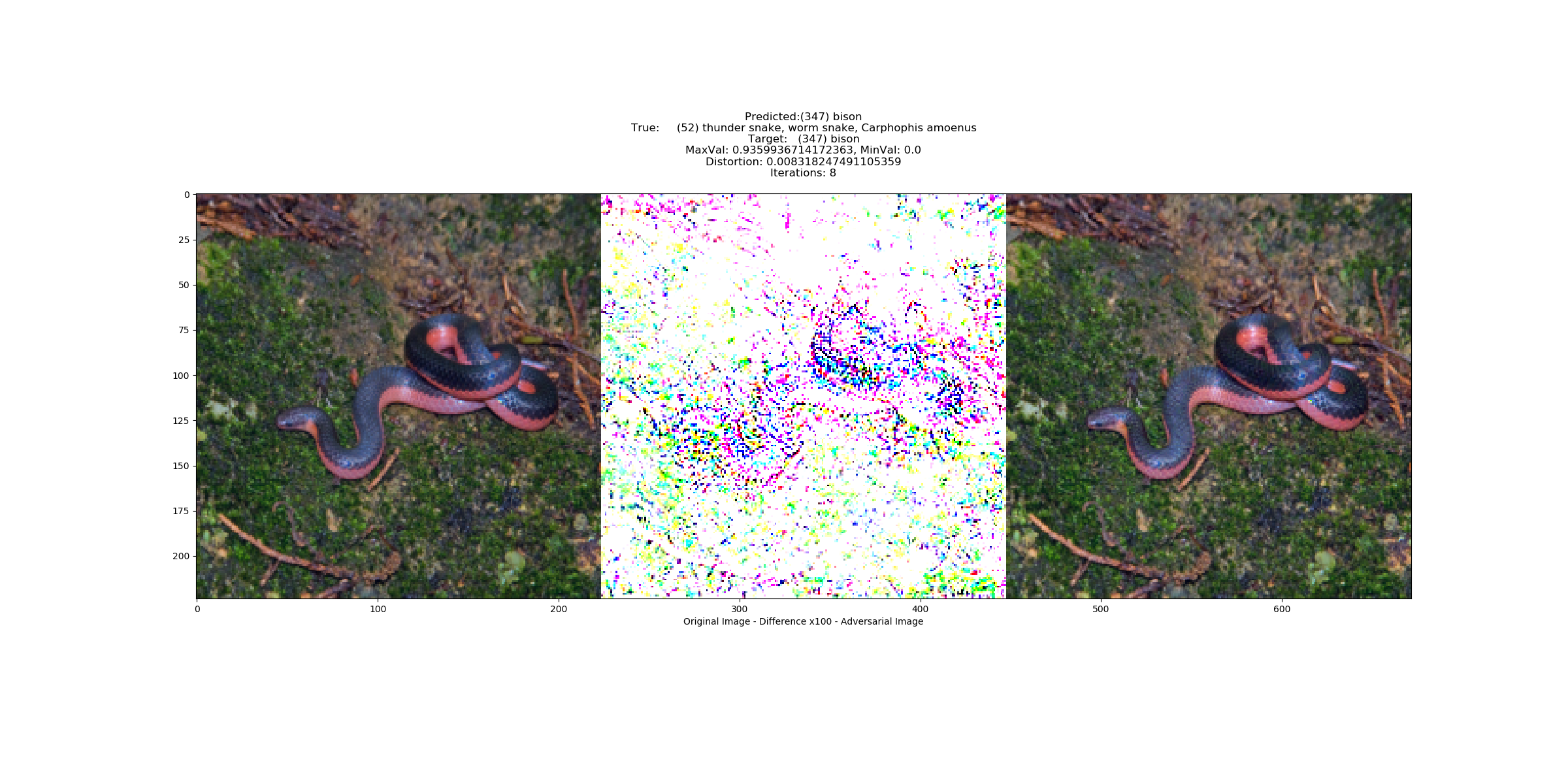}
\includegraphics[trim=200 185 100 200, clip, width=8cm]{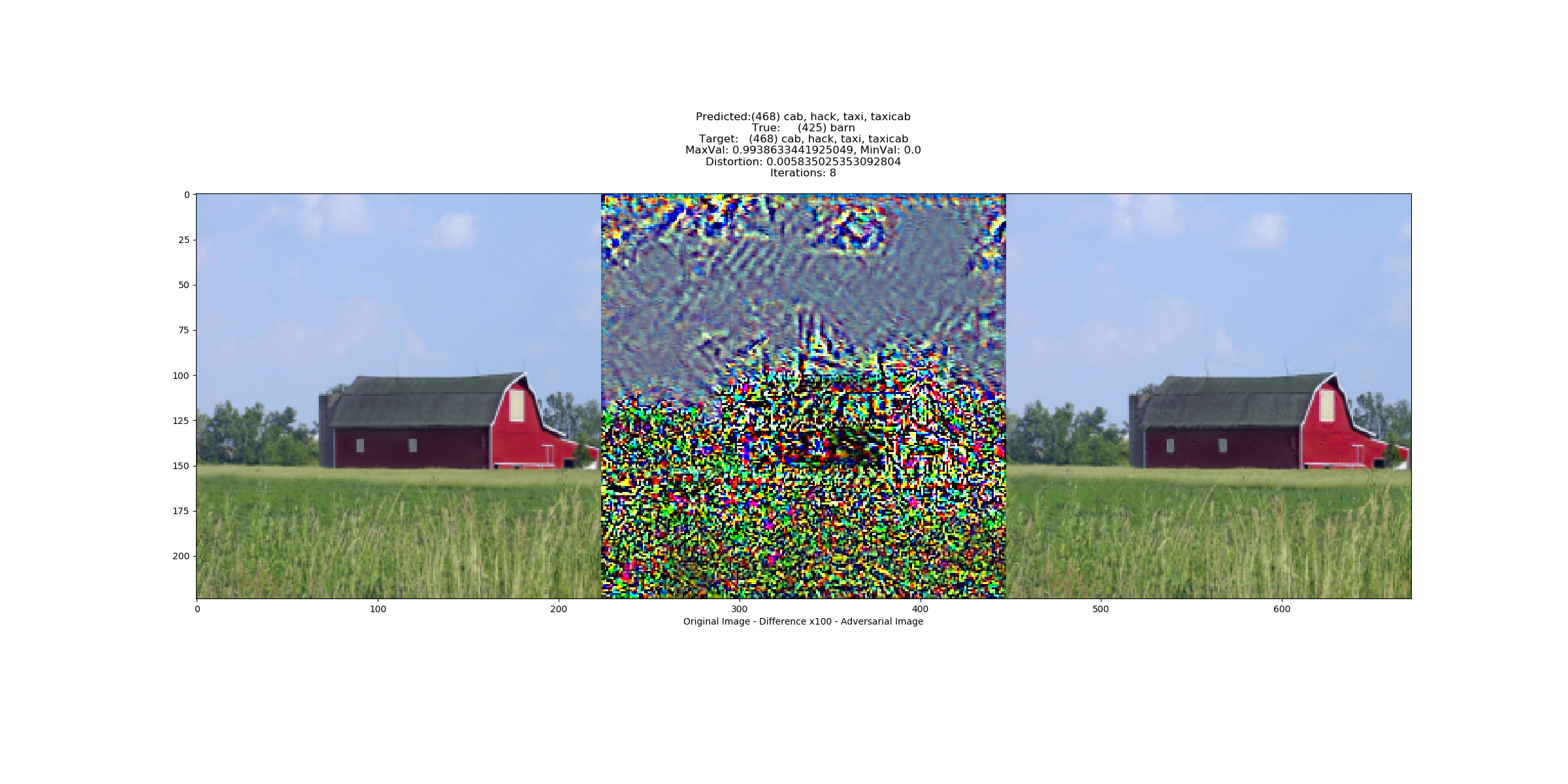}\includegraphics[trim=200 185 100 200, clip, width=8cm]{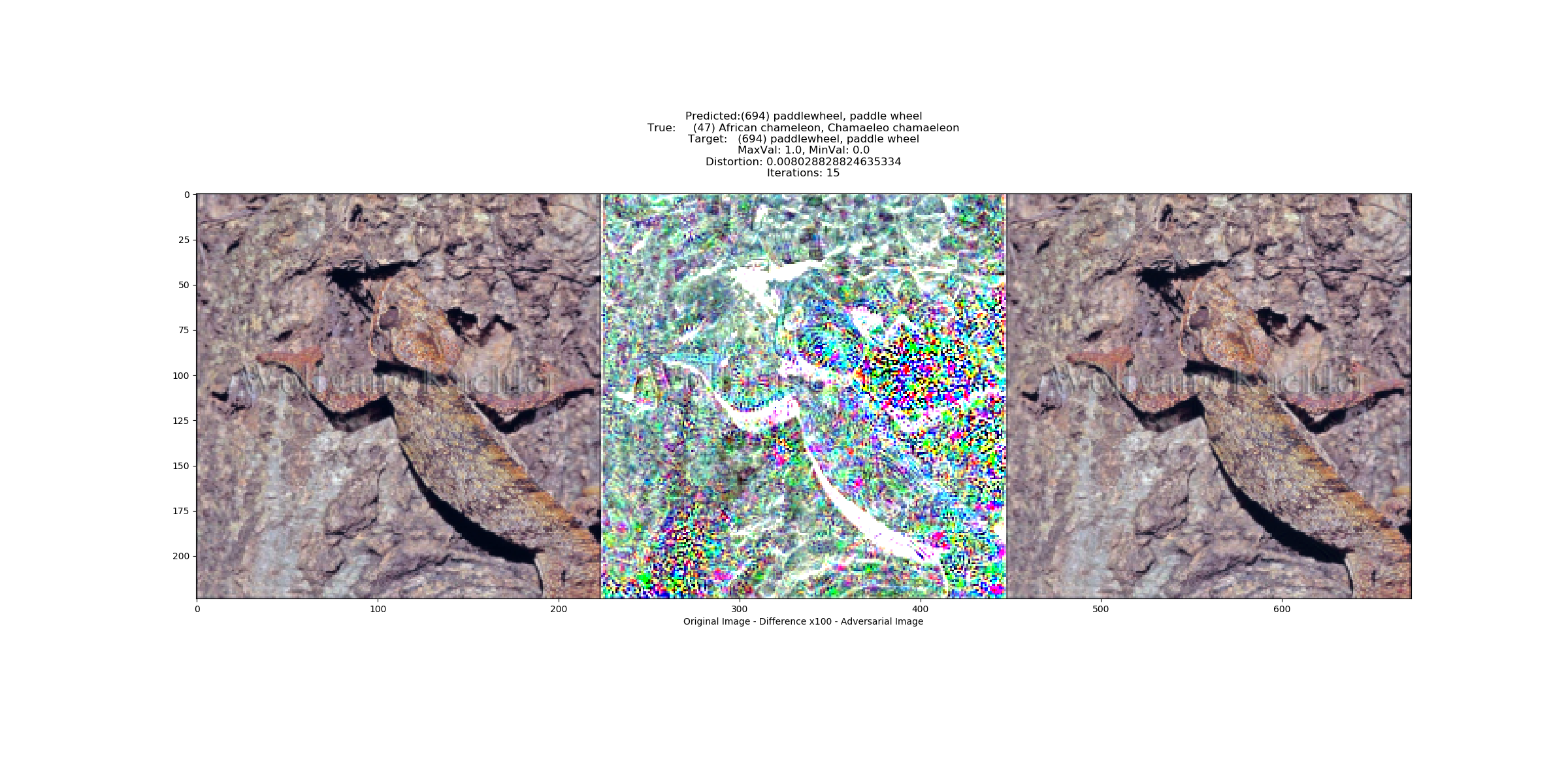}
\caption{Original images on the left, Perturbation (magnified by a factor of 100) by is in the middle, Adversarial Image (total of Original with Perturbation) is on the right. }
\label{lbfgsis}
\end{figure}


\begin{figure}[ht]
\includegraphics[trim=200 80 100 100, clip,width=14cm]{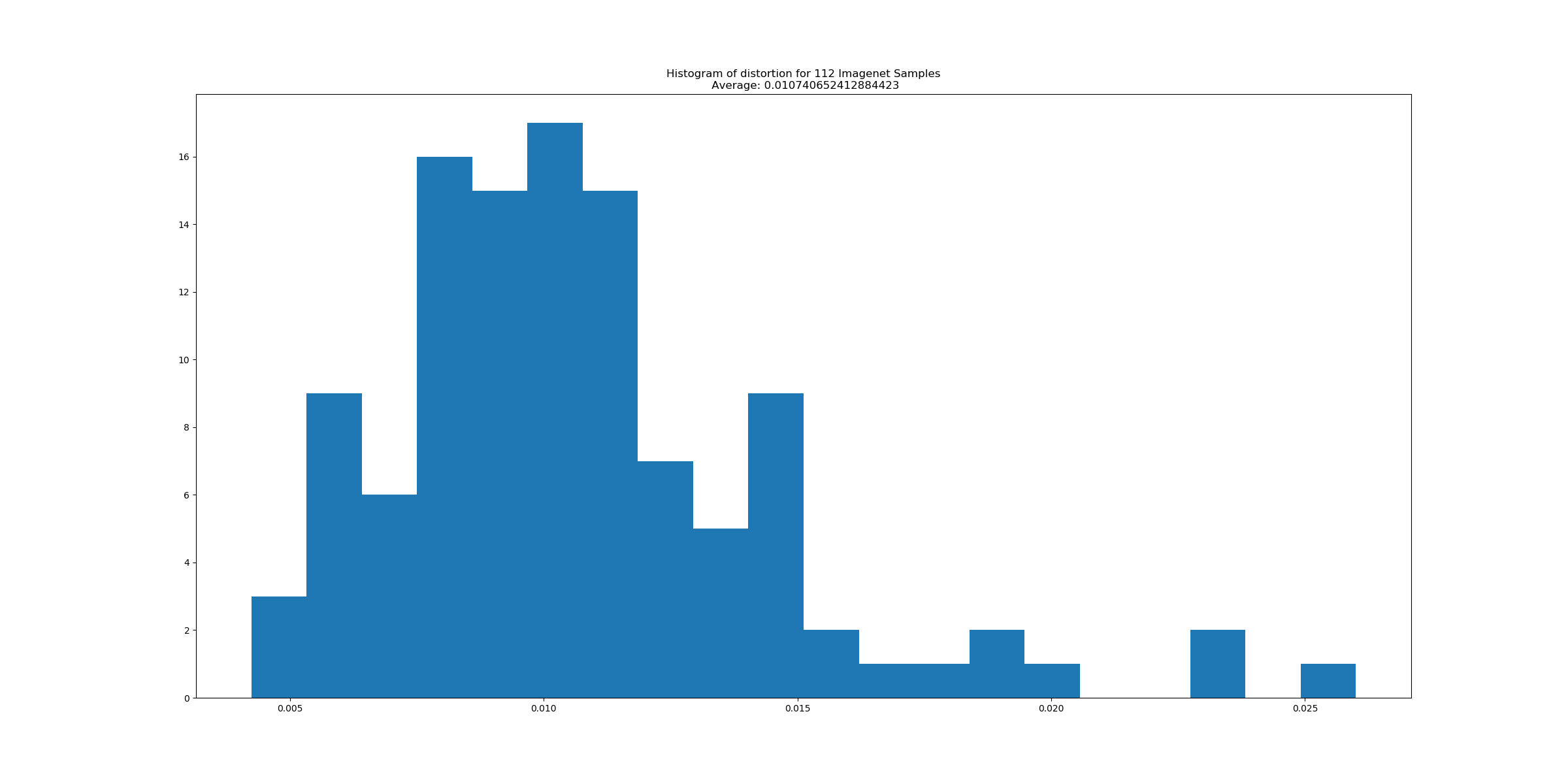}
\caption{A histogram of the distortion measured for each of 112 adversarial examples generated using L-BFGS against the VGG16 network on ImageNet images with mean distortion 0.0107}
\label{lbfgsi}
\end{figure}

\paragraph{Fast Gradient Sign Method (FGSM)} 

As the study of adversarial examples has expanded, it has become known
that often very simple single-step attacks are successful and
sufficiently subtle. ~\citet{goodfellow_explaining_2014} proposed one
such attack which we have also implemented. This is a single step
attack process which uses the sign of the gradient of the loss
function $L$  with respect to the image to find the adversarial
perturbation. For given $\e$, the modified  image $\hat x$ is
computed as 
\begin{equation}
\hat{x} = x + \epsilon \text{sign} (\nabla L (P_w(x),x))
\end{equation}

This method is simpler and much faster to compute than the L-BFGS technique described above, but produces adversarial examples less reliably and with generally larger distortion. Performance was similar but inferior to the Iterative Gradient Sign Method summarized below.  

\paragraph{Iterative Gradient Sign Method (IGSM)}
\label{igsm-s}
In work by ~\citet{kurakin_adversarial_2016}
  an iterative application of FGSM was proposed. After each
  iteration, the image is clipped to a $\e L_\infty$ neighborhood of the original. Let $x'_0 = x$, then after $m$ iterations, the adversarial image obtained is:
\begin{equation}
x_{m+1}' = \text{Clip}_{x,\epsilon} \Bigl\{x_m' + \alpha \times \text{sign}(\nabla \ell (F(x'_m),x'_m))  \Bigr\} 
\label{igsm}
\end{equation}
Where $\text{Clip}_{x,\e}$ takes the minimum of $x+\e$ and $x'_m$ for
elements larger than $x+\e$ and vice versa. 
This method is faster than L-BFGS and more reliable than FGSM but
still produces examples with greater distortion than L-BFGS. An
example is shown in Figure.~\ref{fgsmhip}.  
\begin{figure}[ht]
  \centering
\includegraphics[trim=200 110 1200 102, clip,width=4cm]{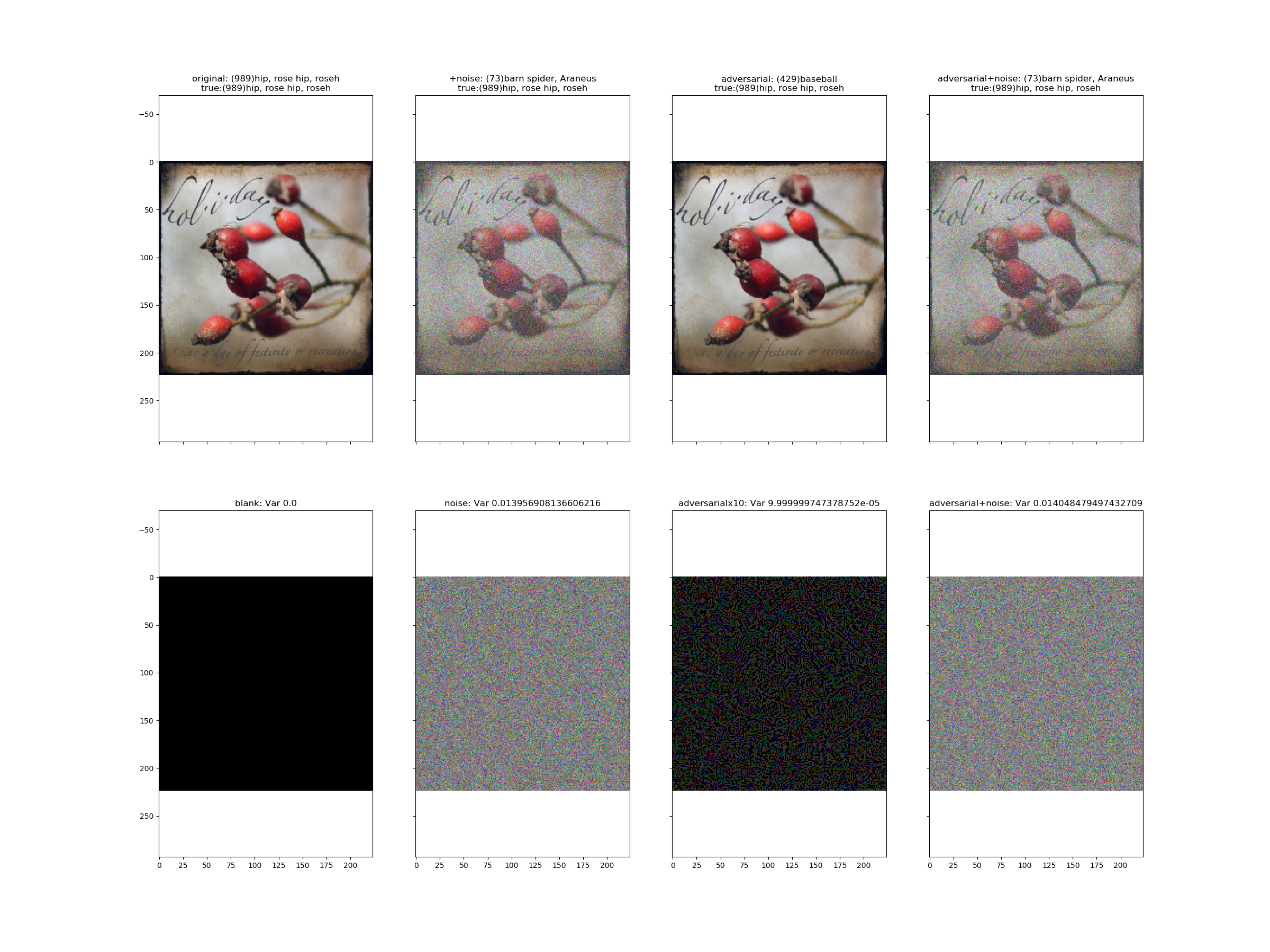}\includegraphics[trim=900 110 500 102, clip,width=4cm]{c1_figures/ILSVRC2012_val_00002900summary_plot.png}
\caption{adversarial example generated against VGG16 (ImageNet) with IGSM. Original Image on the left, adversarial image and added noise (ratio of variance adversarial noise/original image: 0.0000999) on the right. }
\label{fgsmhip}
\end{figure}

\subsection{Other Attacks}
The following attack techniques are also prevalent in the literature 

\paragraph{Jacobian-based Saliency Map Attack (JSMA)} Another attack noted by  ~\citet{papernot_limitations_2015}
  estimates the \emph{saliency map}, a rating for each of the input features (e.g. each pixel) on how influential it is for causing the model to predict a particular class with respect to the model output ~\citep{wiyatno2018saliency}. This attack modifies the pixels that are most salient. This is a targeted attack, and saliency is designed to find the pixel which increases the classifier's output for the target class while tending to decrease the output for other classes.

\paragraph{Deep Fool (DFool)} A technique proposed by ~\citet{moosavi-dezfooli_deepfool:_2015}
  to generate an un-targeted iterative attack. 
This method approximates the classifier as a linear decision boundary and then finds the smallest perturbation needed to cross that boundary.
This attack minimizes $L_2$ norm with respect to  to the original image.

\paragraph{Carlini \& Wagner (C\&W)} In work by ~\citet{carlini_towards_2016}
  an adversarial attack is proposed which updates the loss function
  such that it jointly minimizes $L^p$ and a custom differentiable
  loss function based on un-normalized outputs of the classifier
  (called \textit{logits}). 
Let $Z_k$ denote the logits of a model for a given class $k$, and $\kappa$ a margin parameter. Then C\&W tries to minimize:
\begin{equation}
|| x - \hat{x} ||_p + c* max\left(Z_k(\hat{x}_y) - max\{Z_k(\hat{x}) : k \neq y\},-\kappa\right)
\end{equation}

\subsection{Attack Standards and Toolbox}

Since adversarial robustness has expanded as a field, many papers have
been released pushing various methods for defending against
adversarial attacks. While initially this approach -- producing a
defense that fit a narrow context and releasing it to the community
for evaluation was seen as useful. However, most such approaches would
inevitably face simple rebuttals by small modification of the attack
techniques used. Carlini and their group gained a particular
reputation for brief rebuttals ~\citep{carlini_towards_2016,
  papernot_cleverhans_2016} of such methods. These approaches were
finally codified by ~\citet{tramer2020adaptive} in the form of a set of
guidelines that should be used to attack any proposed defense before
releasing it to the community. This high bar has greatly reduced the
number of low quality defenses which gain attention, but it has also
demonstrated the incredible difficulty of producing successful general
defenses against adversarial attacks. Despite its poor
performance, the strategy of adversarial training proposed by
~\citet{tramer2019adversarial} is one of the few
defenses which have maintained any advantage under the Tramer/Carlini
adaptive framework.

\section{Theory of Adversarial Examples}

Despite the prevalence of studies developing and analyzing adversarial
attacks, the field is characterized by a plethora of definitions for
what it means to be ``adversarial''. We will analyze a few of these in
order to develop our own precise definitions. Indeed, defining an
adversarial example is intimately related with the task of
identification, which leaves a paradox of sorts: If we can precisely
define an adversarial example and that definition allows us to
identify them, then that definition constitutes a perfect defense. In
practice, however, we know this is at least not trivial. 

\subsection{Defining Adversarial Attacks}

~\citet{roth19aodds} proposed a statistical method to identify adversarial examples from natural data. Their main idea was to consider how the last layer in the neural network (the logit layer) would behave on small perturbations of a natural example. 
This is then compared to the behavior of a potential adversarial example. If it differs by a predetermined threshold, the example is flagged as adversarial. Successfully flagging adversarial examples in this way works best when adversarial examples tend to perturb toward the original class from which the adversarial example was perturbed. However, this is not always the case.
It was shown by ~\citet{hosseini2019odds} that it is possible to produce adversarial examples, for instance using a logit mimicry attack, that instead of perturbing an adversarial example toward the true class, actually perturb to some other background class. In fact, we will see in Section \ref{sec:mnist} that the emergence of a background class, which was observed as well by ~\citet{roth19aodds}, is quite common. 

We primarily consider adversarial examples for classifiers.  Let $X$ be a set of possible data and let $L$ be a set of labels. We will consider classifier as a map $\CC: X \to L$. In general $X$ may be much larger than the actual space from which our data are drawn. If the data actually come from a submanifold of $X$, we call this the \emph{data submanifold}. The data submanifold may not be a strict submanifold, and we often do not know the shape or even dimension of it.

Data is drawn from a distribution $\mu$ on $X$ that is usually not known. The overarching goal of classification is to produce a classifier such that $\CC$ is as good as possible on the support of $\mu$. 
We define $X_N \subseteq X$ to be the support of $\mu$ and call it the set of \emph{natural data}. 
Usually our classification problem is the following: given a set of i.i.d. samples $\Sigma \sim \mu$
, where we consider $\Sigma \subseteq X_N$, 
and a classifier $\CC_\Sigma$ on $\Sigma$, find a classifier $\CC$ on $X$ such that $\CC$ lies in some class of ``good functions'' in such a way that it is relatively good at interpolating and/or extrapolating $\CC_\Sigma$. In particular, we hope that $\CC$ is as accurate as possible on the support of $\mu$, which we call the \emph{natural data}. 
The classifier $\CC$ partitions $X$ into classes, each of which is defined as $\CC^{-1}(\ell)$ for some $\ell \in L$. Points on the boundaries of these classes do not have a clear choice of label, and the points in $X$ on the boundaries of the classes make up the \emph{decision boundary} for $\CC$.

To build up to a mathematical framework for adversarial attacks in the
context of geometric analysis, we develop definitions and terms to
refer to adversarial examples without relying on subjective
characteristics like human vision. Let $X$ denote a set of possible
data and $L$ denote a set of labels that distinguish the different
classes. We are now ready to define adversarial examples.

\begin{definition} \label{def:advers}
Let $d$ be a metric on $X$, let $x\in X$ have label $\ell\in L$, and
let $\CC:X\to L$ be a classifier.
We say that $x$ admits an \emph{$(\e,d)$--adversarial example} to $\CC$ if there exists $\hat x \in X$ such that $d(x,\hat x) < \e$ and $\CC(\hat x) \neq \ell$.
\end{definition}

One typically considers Definition \ref{def:advers} in the context of small $\e$. 
Often consideration is made of when such a misclassification is a result of an intentional act by an adversary. 
There are various methods of producing adversarial examples which are discussed later. In some cases, the adversarial label is explicitly targeted:
\begin{definition}
Let $d$ be a divergence on $X$, let $x\in X$ have label $\ell\in L$, and let $\CC:X\to L$ be a classifier.  Let $\varepsilon>0$ and $\ell_t\neq \ell$ be fixed. We say that $x$ admits an \emph{$(\varepsilon,d,\ell_t)$--targeted adversarial example} to $\mathcal{C}$ if there exists $\hat{x}\in X$ such that $d(x,\hat{x})<\varepsilon$ and $\CC(\hat{x})=\ell_t$.
Consider a point $x \in X$ with corresponding class $\ell \in C$ and a classifier $\CC: X \to C$. We say that $x$ admits an \emph{$(\e,d,\ell_t)-$targeted adversarial example}  if there exists a point $\hat x$ such that $d(x,\hat x) < \e$ and $\CC(\hat x) = \ell_t$. 
\end{definition}

These definitions rely on a metric $d$, emphasizing the reliance on the choice of distance to understand notions of closeness. From here on, we will assume that $(X,d)$ is a Euclidean vector space with $d$ being the Euclidean metric. This will allow for the use of standard Gaussian distributions as well.

 The solution to this optimization
problem, efficiently approximated by a gradient-based optimizer, would
be a slightly perturbed natural input with a highly perturbed
output. We have already shown several examples of these techniques
being applied, and one more example can be seen in
Figure.~\ref{fig:my_label} for AlexNet, a cutting-edge model designed
by ~\citet{alexnet} in collaboration with Geoffrey Hinton. \footnote{Note that
Alex Stutskever was both a member of DNNResearch, Hinton's company
which was acquired by Google to become the core of Google Brain and
since 2015 he has been Chief Scientist at OpenAI. He was in the news
recently as a member of the board of OpenAI in November 2023 when Sam
Altman was fired as CEO and then re-hired a week later.} There has
since been significant work describing methods of producing and identifying
adversarial examples. In the next chapter, we will attempt to develop a
framework to understand what properties of an ANN are being exploited
by such methods.

\begin{figure}[ht]
   \centering
\includegraphics[width=7.3cm]{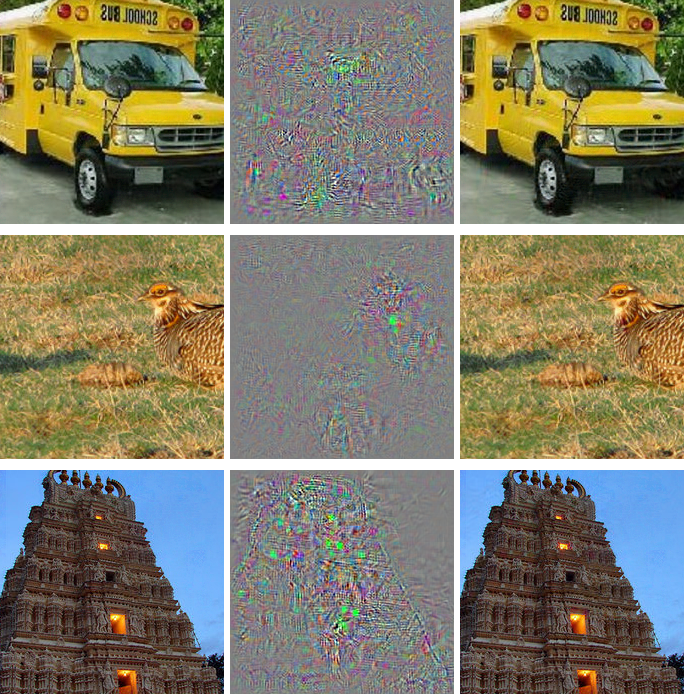}\includegraphics[width=7.3cm]{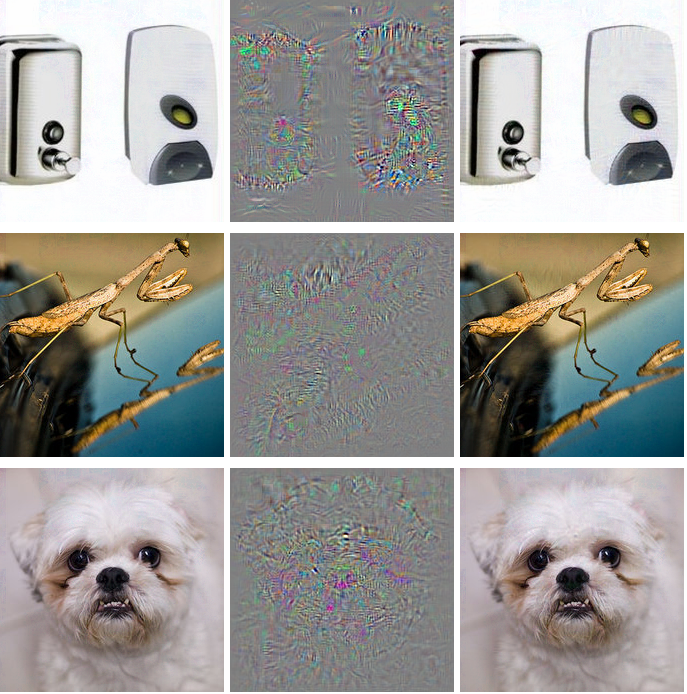}
   \caption{Natural Images are in columns 1 and 4, Adversarial images are in columns 3 and 6, and the difference between them (magnified by a factor of 10) is in columns 2 and 5. All images in columns 3 and 6 are classified by AlexNet as "Ostrich" ~\citep{szegedy2013}}
   \label{fig:my_label}
\end{figure}


\chapter{Persistent Classification}
\label{Chapter3}

This chapter is devoted to defining geometric properties
that can be used to assess robustness of neural networks and to
define adversarial examples in testable and separable
ways. This synthesizes the concept of Persistence, a distributional
method for analyzing curvature around data, with the optimization
tasks used to generate adversarial examples and the direct measurement
of properties of decision boundaries for neural networks. The high
angles of incidence observed when interpolating across decision
boundaries of neural networks among natural images and the much higher
angles observed when interpolating among adversarial examples indicate
that geometric properties may explain or be related to why adversarial
examples are so easy to find for certain networks. This paper is in
submission to the CODA Journal, and is composed of theoretical and
practical work on Persistence and decision boundary incidence angles
and other properties all conducted by Brian Bell. A second body of
work analyzing the manifold alignment of classification models is
related to this work, attempting to correct some of these geometric
inadequacies by forcing networks to be orthogonal or parallel to approximated manifolds in specific regimes. This second body of work was completed by Michael Geyer.

Whereas ~\citet{roth19aodds} consider adding various types of noise to a given point and ~\citet{hosseini2019odds} consider small Gaussian perturbations of $x$ sampled from $N(x,\varepsilon^2 I)$ for small $\varepsilon$, 
we specifically focus on 
tuning the standard deviation parameter to determine a statistic describing how a given data point is placed within its class. The $\gamma$-persistence then gives a measurement similar to distance to the boundary but that is drawn from sampling instead of distance. The sampling allows for a better description of the local geometry of the class and decision boundary, as we will see in Section \ref{subsec:stab}. Our statistic is based on the fraction of a Gaussian sampling of the neighborhood of a point that receives the same classification; this is different from that of ~\citet{roth19aodds}, which is the expected difference of the output of the logit layer of the original data point and the output of the logit layer for perturbed data points.  Additionally, while their statistics are defined pairwise with reference to pre-chosen original and candidate classes, ours is not.


There are a number of hypotheses underlying the existence of
adversarial examples for classification problems, including the high-dimensionality of the data, high codimension in the ambient space of the data manifolds of interest, and that the structure of machine learning models may encourage classifiers to develop decision boundaries close to data points. 
This article proposes a new framework for studying adversarial examples that does not depend directly on the distance to the decision boundary. 
Similarly to the smoothed classifier literature, we define a (natural or adversarial) data point to be $(\gamma,\sigma)$-stable if the probability of the same classification is at least $\gamma$ for points sampled in a Gaussian neighborhood of the point with a given standard deviation $\sigma$. 
We focus on studying the differences between persistence metrics along interpolants of natural and adversarial points.
We show that adversarial examples have significantly lower persistence than natural examples for large neural networks in the context of the MNIST and ImageNet datasets. 
We connect this lack of persistence with decision boundary geometry by measuring angles of interpolants with respect to decision boundaries.
Finally, we connect this approach with robustness by developing a manifold alignment gradient metric and demonstrating the increase in robustness that can be achieved when training with the addition of this metric. 

\section{Introduction}

Deep Neural Networks (DNNs) and their variants are core to the success of modern machine learning ~\citep{prakash2018}, and have dominated competitions in image processing, optical character recognition, object detection, video classification, natural language processing, and many other fields \citep{SCHMIDHUBER201585}. Yet such classifiers are notoriously susceptible to manipulation via adversarial examples ~\citep{szegedy2013}. Adversarial examples occur when natural data can be subject to subtle perturbation which results in substantial changes in output. Adversarial examples are not just a peculiarity, but seem to occur for most, if not all, DNN classifiers. For example, \citet{inevitable2018} used isoperimetric inequalities on high dimensional spheres and hypercubes to conclude that there is a reasonably high probability that a correctly classified data point has a nearby adversarial example. This has been reiterated using mixed integer linear programs to rigorously check minimum distances necessary to achieve adversarial conditions ~\citep{tjeng2017evaluating}. \citet{ilyas2019adversarial} showed that adversarial examples can arise from features that are good for classification but not robust to perturbation. 

There have been many attempts to identify adversarial examples using
properties of the decision boundary. \citet{Fawzi2018empirical} found
that decision boundaries tend to have highly curved regions, and these
regions tend to favor negative curvature, indicating that regions that
define classes are highly nonconvex. The purpose of this work is to
investigate these geometric properties related to the decision
boundaries. We will do this by proposing a notion of stability that is
more nuanced than simply measuring distance to the decision boundary,
and is also capable of elucidating information about the curvature of
the nearby decision boundary. We develop a statistic extending
prior work on smoothed classifiers by \citet{cohen2019certified}. We denote this metric as Persistence and use it as a measure of how far away from a point one can go via Gaussian sampling and still consistently find points with the same classification. One advantage of this statistic is that it is easily estimated by sending a Monte Carlo sampling about the point through the classifier. In combination with this metric, direct measurement of decision boundary incidence angle with dataset interpolation and manifold alignment can begin to complete the picture for how decision boundary properties are related with neural network robustness. 

 These geometric properties are related to the alignment of gradients
 with human perception \citep{ganz2022perceptually,
   kaur2019perceptually, shah2021input} and with the related
 underlying manifold \citep{kaur2019perceptually,
   ilyas2019adversarial} which may imply robustness. For our purposes,
 Manifold Aligned Gradients (MAG) will refer to the property that the
 gradients of a model with respect to model inputs follow a given data
 manifold $\mathcal{M}$ extending similar relationships from other
 work by \citet{shamir2021dimpled}. 

{\bf Contributions.} We believe these geometric properties are related
to why smoothing methods have been useful in robustness tasks
~\citep{cohen2019certified, lecuyer2019certified, li2019certified}. We propose three approaches in order to connect robustness with geometric properties of the decision boundary learned by DNNs: 
\begin{enumerate}
    \item We propose and implement two metrics based on the success of smoothed classification techniques:  $(\gamma,\sigma)$-stability and $\gamma$-persistence defined with reference to a classifier and a given point (which can be either a natural or adversarial image, for example) and demonstrate their validity for analyzing adversarial examples. 
    \item We interpolate across decision boundaries using our persistence metric to demonstrate an inconsistency at the crossing of a decision boundary when interpolating from natural to adversarial examples.
    \item We demonstrate via direct interpolation across decision boundaries and measurement of angles of interpolating vectors relative to the decision boundary itself that dimensionality is not solely responsible for geometric vulnerability of neural networks to adversarial attack. 
\end{enumerate}

\section{Motivation and Related Work}

Our work is intended to shed light on the existence and prevalence of adversarial examples to DNN classifiers. It is closely related to other attempts to characterize robustness to adversarial perturbations, and here we give a detailed comparison.

{\bf Distance-based robustness.}

A typical approach to robustness of a classifier is to consider
distances from the data manifold to the decision boundary
~\citep{Wang2020Improving, xu2023exploring, he2018decision}.
\citet{khoury2018} define a classifier to be robust if the class of
each point in the data manifold is contained in a sufficiently large
ball that is entirely contained in the same class. The larger the
balls, the more robust the classifier. It is then shown that if
training sets are sufficiently dense in relation to the reach of the
decision axis, the classifier will be robust in the sense that it
classifies nearby points correctly. In practice, we do not know that
the data is so well-positioned, and it is quite possible, especially
in high dimensions, that the reach is extremely small, as evidenced by
results on the prevalence of adversarial examples, e.g.,
\citet{inevitable2018} and in evaluation of ReLU networks with mixed
integer linear programming e.g., ~\citet{tjeng2017evaluating}.

\citet{tsipras2018robustness} investigated robustness in terms of how
small perturbations affect the the average loss of a classifier. They
define standard accuracy of a classifier in terms of how often it
classifies correctly, and robust accuracy in terms of how often an
adversarially perturbed example classifies correctly. It was shown
that sometimes accuracy of a classifier can result in poor robust
accuracy. \citet{gilmer2018adversarial} use the expected distance to
the nearest different class (when drawing a data point from the data
distribution) to capture robustness, and then show that an accurate
classifier can result in a small distance to the nearest different
class in high dimensions when the data is drawn from concentric
spheres. May recent works ~\citep{he2018decision, chen2023aware,
  jin2022roby} have linked robustness with decision boundary dynamics,
both by augmenting training with data near decision boundaries, or
with dynamics related to distances from decision boundaries. We
acknowledge the validity of this work, but will address some of its
primary limitations by carefully studying the orientation
of the decision boundary relative to model data.

A related idea is that adversarial examples often arise within cones,
outside of which images are classified in the original class, as
observed by \citet{roth19aodds}. Many theoretical models of
adversarial examples, for instance the dimple model developed by
\citet{shamir2021}, have high curvature and/or sharp corners as an
essential piece of why adversarial examples can exists very close to
natural examples. 

{\bf Adversarial detection via sampling.}
While adversarial examples often occur, they still may be rare in the
sense that most perturbations do not produce adversarial
examples. \citet{yu2019new} used the observation that adversarial
examples are both rare and close to the decision boundary to detect
adversarial examples. They take a potential data point and look to see
if nearby data points are classified differently than the original
data point after only a few iterations of a gradient descent
algorithm. If this is true, the data point is likely natural and if
not, it is likely adversarial. This method has been generalized with
the developing of smoothed classification methods
~\citep{cohen2019certified, lecuyer2019certified, li2019certified}
which at varying stages of evaluation add noise to the effect of
smoothing output and identifying adversaries due to their higher
sensitifity to perturbation.. These methods suffer from significant
computational complexity ~\citep{kumar2020curse} and have been shown
to have fundamental limitations in their ability to rigorously certify
robustness ~\citep{blum2020random, yang2020randomized}. We will generalize this approach into a metric which will allow us to directly study these limitations in order to better understand how geometric properties have given rise to adversarial vulnerabilities. In general, the results of \citet{yu2019new} indicate that considering samples of nearby points, which approximate the computation of integrals, is likely to be more successful than methods that consider only distance to the decision boundary.

\citet{roth19aodds} proposed a statistical method to identify adversarial examples from natural data. Their main idea was to consider how the last layer in the neural network (the logit layer) would behave on small perturbations of a natural example. 
This is then compared to the behavior of a potential adversarial example. 

It was shown by \citet{hosseini2019odds} that it is possible to produce adversarial examples, for instance using a logit mimicry attack, that instead of perturbing an adversarial example toward the true class, actually perturb to some other background class. In fact, we will see in Section \ref{sec:mnist} that the emergence of a background class, which was observed as well by \citet{roth19aodds}, is quite common. Although many recent approaches have taken advantage of these facts ~\citep{taori2020shifts, lu2022randommasking, Osada_2023_WACV, blau2023classifier} in order to measure and increase robustness, we will leverage these sampling properties to develop a metric directly on decision-boundary dynamics and how they relate to the success of smoothing based robustness. 

{\bf Manifold Aware Robustness}

The sensitivity of convolutional neural networks to imperceptible changes in input has thrown into question the true generalization of these models.
~\citet{jo2017measuring} study the generalization performance of CNNs by transforming natural image statistics.  
Similarly to our MAG approach, they create a new dataset with well-known properties to allow the testing of their hypothesis.
They show that CNNs focus on high level image statistics rather than human perceptible features.
This problem is made worse by the fact that many saliency methods fail basic sanity checks \citep{adebayo2018sanity, kindermans2019reliability}.

Until recently, it was unclear whether robustness and manifold alignment were directly linked, as the only method to achieve manifold alignment was adversarial training.
Along with the discovery that smoothed classifiers are perceptually
aligned, comes the hypothesis that robust models in general share this
property put forward by ~\citet{kaur2019perceptually}.
This discovery raises the question of whether this relationship
between manifold alignment of model gradients and robustness is bidirectional.

~\citet{khoury2018} study the geometry of natural images, and create a
lower bound for the number of data points required to effectively
capture variation in the manifold.
Unfortunately, they demonstrate that this lower bound is so large as to be intractable.
~\citet{shamir2021dimpled} propose using the tangent space of a
generative model as an estimation of this
manifold. ~\citet{magai2022topology} thoroughly review certain
topological properties to demonstrate that neural networks
intrinsically use relatively few dimensions of variation during
training and evaluation. ~\citet{vardi2022gradient} demonstrate that even models that satisfy strong conditions related to max margin classifiers are implicitly non-robust. PCA and manifold metrics have been recently used to identify adversarial examples ~\citep{aparne2022pca, nguyen-minh-luu-2022-textual}. We will extend this work to study the relationship between robustness and manifold alignment directly by baking alignment directly into networks and comparing them with another approach to robustness. 

{\bf Summary.}
 In Sections \ref{sec:meth} and \ref{sec:experiments}, we will investigate stability of both natural data and adversarial examples by considering sampling from Gaussian distributions centered at a data point with varying standard deviations. Using the standard deviation as a parameter, we are able to derive a statistic for each point that captures how entrenched it is in its class in a way that is less restrictive than the robustness described by \citet{khoury2018}, takes into account the rareness of adversarial examples described by \citet{yu2019new}, builds on the idea of sampling described by \citet{roth19aodds} and \citet{hosseini2019odds}, and represent curvatures in a sense related to \citet{Fawzi2018empirical}. Furthermore, we will relate these stability studies to direct measurement of interpolation incident angles with decision boundaries in Subsection~\ref{subsec:db} and ~\ref{subsec:dbe} and the effect of reduction of data onto a known lower dimensional manifold in Subsections ~\ref{subsec:ma} and ~\ref{subsec:mae}.  

\section{Methods} \label{sec:meth} 

In this section we will lay out the theoretical framework for studying stability, persistence, and decision boundary corssing-angles. 

\subsection{Stability and Persistence} \label{subsec:stab}
In this section we define a notion of stability of classification of a point under a given classification model. In the following, $X$ represents the ambient space the data is drawn from (typically $\RR^n$) even if the data lives on a submanifold of $X$, and $L$ is a set of labels (often $\{1,\dots,\ell\}$).  Note that points $x\in X$ can be natural or adversarial points.

\begin{definition}
Let $\CC:X\to L$ be a classifier, $x \in X$, $\gamma\in(0,1)$, and $\sigma>0$. We say $x$ is \emph{$(\gamma,\sigma)$-stable} with respect to $\CC$ if $\mathbb{P}[\CC(x')=\CC(x)] \geq \gamma$ for $x' \sim \rho = N(x, \sigma^2 I)$; i.e. $x'$ is drawn from a Gaussian with variance $\sigma^2$ and mean $x$.
\end{definition}

In the common setting when $X=\RR^n$, we have
\[\mathbb{P}[\CC(x')=\CC(x)] = \int_{\RR^n} \mathbbm{1}_{\CC^{-1}(\CC(x))} (x') d\rho (x') = \rho(\CC^{-1}\CC(x)).\]
Note here that $\CC^{-1}$ denotes preimage. 
One could substitute various probability measures $\rho$ above with mean $x$ and variance $\sigma^2$ to obtain different measures of stability corresponding to different ways of sampling the neighborhood of a point.  Another natural choice would be sampling the uniform measure on balls of changing radius. Based on the concentration of measure for both of these families of measures we do not anticipate significant qualitative differences in these two approaches. We propose Gaussian sampling because it is also a product measure, which makes it easier to sample and simplifies some other calculations below.

For the Gaussian measure, the probability above may be written more concretely as
\begin{equation}\label{EQN:Gaussian}
\frac{1}{\left(\sqrt{2\pi}\sigma\right)^{n}} \int_{\RR^n} \mathbbm{1}_{\CC^{-1}(\CC(x))} (x')e^{-\frac{\norm{x - x'}^2}{2\sigma^2}} dx'.
\end{equation}
In this work, we will conduct experiments in which we estimate this stability for fixed $(\gamma,\sigma)$ pairs via a Monte Carlo sampling, in which case the integral \eqref{EQN:Gaussian} is approximated by taking $N$ i.i.d. samples $x_k \sim \rho$ and computing
\[
    \frac{\norm{x_k : \CC(x_k) = \CC(x)}}{N}.
\]
Note that this quantity converges to the integral \eqref{EQN:Gaussian} as $N\to\infty$ by the Law of Large Numbers.

The ability to adjust the quantity $\gamma$ is important because it is much weaker than a notion of stability that requires a ball that stays away from the decision boundary as by \citet{khoury2018}. By choosing $\gamma$ closer to $1$, we can require the samples to be more within the same class, and by adjusting $\gamma$ to be smaller we can allow more overlap.

We also propose a related statistic, \emph{persistence}, by fixing a particular $\gamma$ and adjusting $\sigma$. For any $x\in X$ not on the decision boundary, for any choice of $0<\gamma<1$ there exists a $\sigma_\gamma$ small enough such that if $\sigma < \sigma_\gamma$ then $x$ is $(\gamma,\sigma)$-stable. We can now take the largest such $\sigma_\gamma$ to define persistence.

\begin{definition}
    Let $\CC:X\to L$ be a classifier, $x \in X$, and $\gamma\in(0,1)$. Let $\sigma_\gamma^*$ be the maximum $\sigma_\gamma$ such that $x$ is $(\gamma, \sigma)$-stable with respect to $\CC$ for all $\sigma<\sigma_\gamma$. We say that $x$ has \emph{$\gamma$-persistence} $\sigma_\gamma^*$.
\end{definition}

The $\gamma$-persistence quantity $\sigma_\gamma^*$ measures the stability of the neighborhood of a given $x$ with respect to the output classification. Small persistence indicates that the classifier is unstable in a small neighborhood of $x$, whereas large persistence indicates stability of the classifier in a small neighborhood of $x$. In the later experiments, we have generally taken $\gamma = 0.7$. This choice is arbitrary and chosen to fit the problems considered here. In our experiments, we did not see significant change in results with small changes in the choice of $\gamma$.

In our experiments, we numerically estimate $\gamma$-persistence via a bisection algorithm that we term the Bracketing Algorithm.  Briefly, the algorithm first chooses search space bounds $\sigma_{\min}$ and $\sigma_{\max}$ such that $x$ is  $(\gamma,\sigma_{\min})$-stable but is not $(\gamma,\sigma_{\max})$-stable with respect to $\CC$, and then proceeds to evaluate stability by bisection until an approximation of $\sigma_\gamma^*$ is obtained.

\subsection{Decision Boundaries} \label{subsec:db}

In order to examine decision boundaries and their properties, we will carefully define the decision boundary in a variety of equivalent formulations. 

\subsubsection{The Argmax Function Arises in Multi-Class Problems}

A central issue when writing classifiers is mapping from continuous outputs or probabilities to discrete sets of classes. Frequently argmax type functions are used to accomplish this mapping. To discuss decision boundaries, we must precisely define argmax and some of its properties. 

In practice, argmax is not strictly a function, but rather a mapping from the set of outputs or activations from another model into the power set of a discrete set of classes:

\begin{equation}
    \text{argmax} : \R^k \to \mathcal{P}(L)
\end{equation}

Defined this way, we cannot necessarily consider $\text{argmax}$ to be a function in general as the singleton outputs of argmax overlap in an undefined way with other sets from the power set. However, if we restrict our domain carefully, we can identify certain properties. 
\begin{figure}[!ht]
\begin{center}
\begin{tikzpicture}
  \draw[->] (-0.5, 0) -- (3.5, 0) node[right] {$x$};
  \draw[->] (0, -0.5) -- (0, 3.5) node[above] {$y$};
  \node at (1, 2) (c1) {Class 1};
  \node at (2, 1) (c2) {Class 2};
  \draw[-, fill, blue, opacity=.3] (0, 3) -- (3, 3) -- (0, 0);
  \draw[-, fill, red, opacity=.3] (3, 0) -- (3, 3) -- (0, 0);
  \draw[scale=1, domain=0:3, smooth, variable=\x, black, line width=0.45mm] plot ({\x}, {\x});
\end{tikzpicture}\begin{tikzpicture}
  \draw[->] (-0.5, 0) -- (3.5, 0) node[right] {$x$};
  \draw[->] (0, -0.5) -- (0, 3.5) node[above] {$y$};
  \node at (1, 2) (c1) {Class 1};
  \node at (2, 1) (c2) {Class 2};
  \draw[-, fill, blue, opacity=.3] (0, 3) -- (3, 3) -- (0, 0);
  \draw[-, fill, red, opacity=.3] (3, 0) -- (3, 3) -- (0, 0);
  \draw[scale=1, domain=0:3, smooth, variable=\x, black, line width=0.45mm] plot ({\x}, {\x});
  \draw[-, orange, line width=0.45mm] (0,3) -- (3, 0);
\end{tikzpicture}

\caption{Decision boundary in $[0,1] \times [0,1]$ (left) and decision boundary restricted to probabilities (right) }
\label{fig:pdb}
\end{center}
\end{figure}
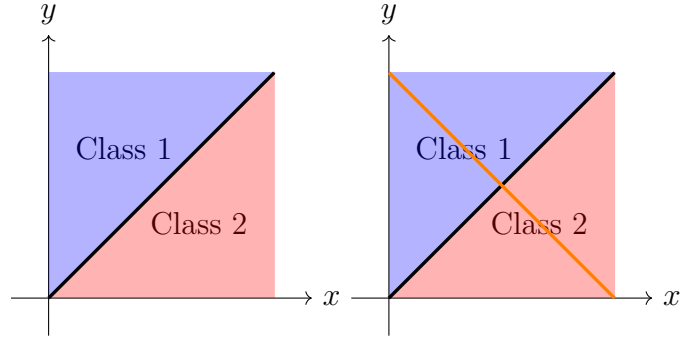
Restricting to only the pre-image of the singletons, it should be
clear that argmax is constant.  Indeed, restricted to the pre-image of
any set in the power-set, argmax is constant and thus
continuous. This induces the discrete topology whereby the pre-image
of an individual singleton is open. Observe that for any point whose
image is a singleton, one element of the domain vector must exceed the 
others by $\varepsilon > 0$. We shall use the $\ell^1$ metric for
distance, and thus if we restrict ourselves to a ball of radius
$\varepsilon$, then all elements inside this ball will have that
element still larger than the rest and thus map to the same singleton
under argmax. We can visualize this in two dimensions. Suppose the output of a function $F$ are \emph{probabilities} which add to one, then all inputs $x$ will map to the orange line on the right side of Figure~\ref{fig:pdb}. We note that the point $(0.5, 0.5)$ is therefore the only point on the decision boundary for probability valued $F$. We may generalize to higher dimensions where all probability valued models $F$ will map into the the plane $x + y + z + \cdots = 1$ in $Y$ and the decision boundary will be partitioned into $K-1$ components, where the $K$-decision boundary is the intersection of this plane with the \emph{centroid} line $x = y = z = \cdots$ and the $2$-decision boundaries become planes intersecting at the same line.

Since the union of infinitely many open sets is open in
$\R^k$, the union of all singleton pre-images is an open
set. Conveniently this also provides proof that the union of all of
the non-singleton sets in $\mathcal{P}(C)$ is a closed set. We will
call this closed set the argmax Decision Boundary. We will list two
equivalent formulations for this boundary.  

\paragraph{Complement Definition}

A point $x$ is in the \emph{decision interior} $D_{\mathcal{C}}'$ for
a classifier $\mathcal{C}: \mathbb{R}^N \to L$ if there exists $\delta
> 0$ such that $\forall \epsilon < \delta$, the number of elements $n(\mathcal{C}(B_\epsilon(x))) = 1$. 

The \emph{decision boundary} of a classifier $\mathcal{C}$ is the closure of the complement of the decision interior $\overline{\{x : x \notin D_{\mathcal{C}}'\}}$. 

\paragraph{Level Set Definition}

For an input space $X$, the decision boundary $D \subset X$ of a probability valued function $f$ is the pre-image of a union of all level sets of outputs $f(X) = {c_1, c_2, ..., c_k}$ defined by a constant $c$ such that for some set of indices $I$, we have $c = c_i$ for every $i$ in $I$ and $c > c_j$ for every $j$ not in $I$. The pre-image of each such set are all $x$ such that $f(x) = A_c$ for some $c$.

\section{Experiments} \label{sec:experiments}

In this section we investigate the stability and persistence behavior of natural and adversarial examples for MNIST \citep{MNIST} and ImageNet \citep{ILSVRC15} using a variety of different classifiers. For each set of image samples generated for a particular dataset, model, and attack protocol, we study $(\gamma,\sigma)$-stability and $\gamma$-persistence of both natural and adversarial images, and also compute persistence along trajectories from natural to adversarial images. In general, we use $\gamma = 0.7$, and note that the observed behavior does not change significantly for small changes in $\gamma$. While most of the adversarial attacks considered here have a clear target class, the measurement of persistence does not require considering a particular candidate class.  Furthermore, we will evaluate decision boundary incidence angles and apply our conclusions to evaluate models trained with manifold aligned gradients. 

\subsection{MNIST Experiments}

Since MNIST is relatively small compared to ImageNet, we trained several classifiers with various architectures and complexities and implemented the adversarial attacks directly. Adversarial examples were generated against each of these models using Iterative Gradient Sign Method (IGSM \citep{kurakin_adversarial_2016}) and Limited-memory Broyden-Fletcher-Goldfarb-Shanno (L-BFGS \citep{liu1989limited}).

\subsubsection{Investigation of $(\gamma, \sigma)$-Stability on MNIST}\label{sec:mnist}

We begin with a fully connected ReLU network with layers of size 784, 100, 20, and 10 and small regularization $\lambda = 10^{-7}$ which is trained on the standard MNIST training set. We then start with a randomly selected MNIST test image $x_1$ from the \texttt{1}'s class and generate adversarial examples $x_0,x_2,\dots,x_9$ using IGSM for each target class other than \texttt{1}. The neighborhoods around each $x_i$ are examined by generating 1000 i.i.d. samples from $N(x_i,\sigma^2I)$ for each of 100 equally spaced standard deviations $\sigma\in(0,1.6)$. Figure \ref{fgsmo} shows the results of the Gaussian perturbations of a natural example $x_1$ of the class labeled \texttt{1} and the results of Gaussian perturbations of the adversarial example $x_0$ targeted at the class labeled \texttt{0}. We provide other examples of $x_2,\ldots,x_9$ in the supplementary materials. Note that the original image is very stable under perturbation, while the adversarial image is not. 

\begin{figure}[!ht]
  \includegraphics[width = .49\textwidth]{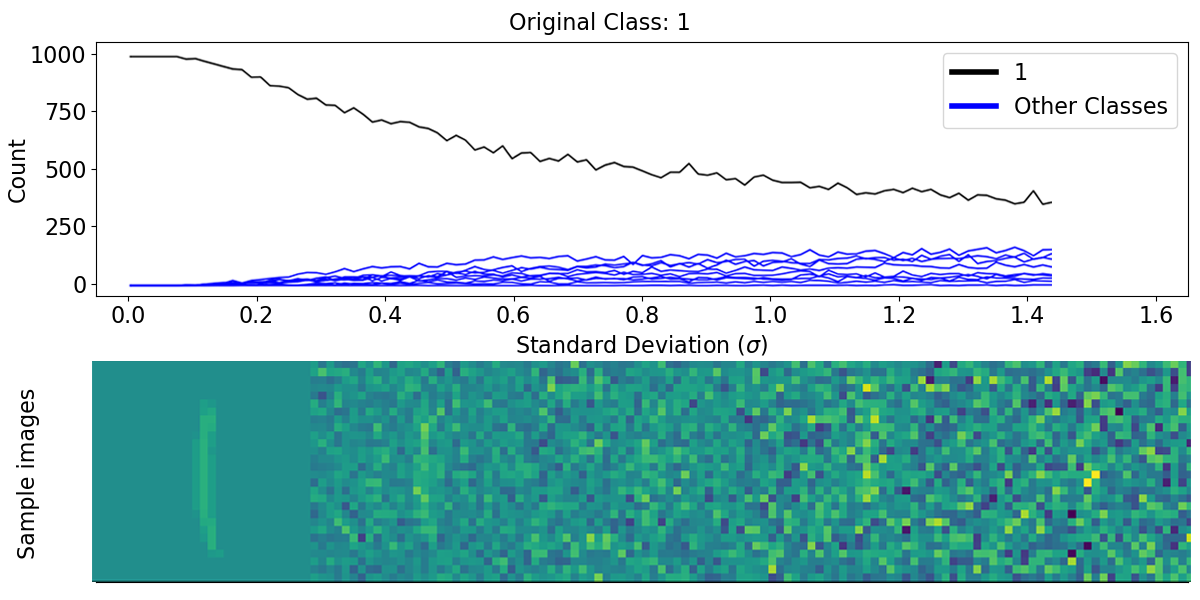}
  \includegraphics[width = .49\textwidth]{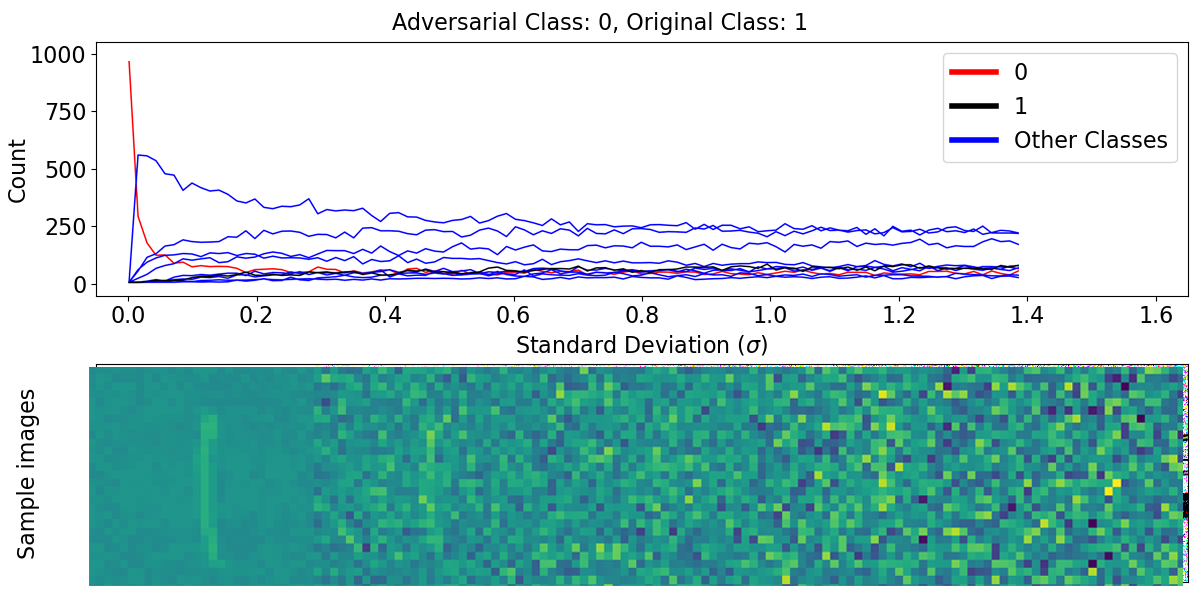}
\caption{Frequency of each class in Gaussian samples with increasing variance around a natural image of class \texttt{1} (left) and around an adversarial attack of that image targeted at \texttt{0} generated using IGSM (right). The adversarial class (\texttt{0}) is shown as a red curve. The natural image class (\texttt{1}) is shown in black. Bottoms show example sample images at different standard deviations for natural (left) and adversarial (right) examples.}\label{fgsmo}
\end{figure}

\subsubsection{Persistence of Adversarial Examples for MNIST}

To study persistence of adversarial examples on MNIST, we take the same network architecture as in the previous subsection and randomly select 200 MNIST images. For each image, we used IGSM to generate 9 adversarial examples (one for each target class) yielding a total of 1800 adversarial examples. In addition, we randomly sampled 1800 natural MNIST images. For each of the 3600 images, we computed $0.7$-persistence; the results are shown in Figure \ref{fig:IGSMpersistenceMNIST}. One sees that $0.7$-persistence of adversarial examples tends to be significantly smaller than that of natural examples for this classifier, indicating that they are generally less stable than natural images. We will see subsequently that this behavior is typical.

\begin{figure}[!ht]
\centering
\includegraphics[trim=200 80 100 100, clip,width=.5\textwidth]{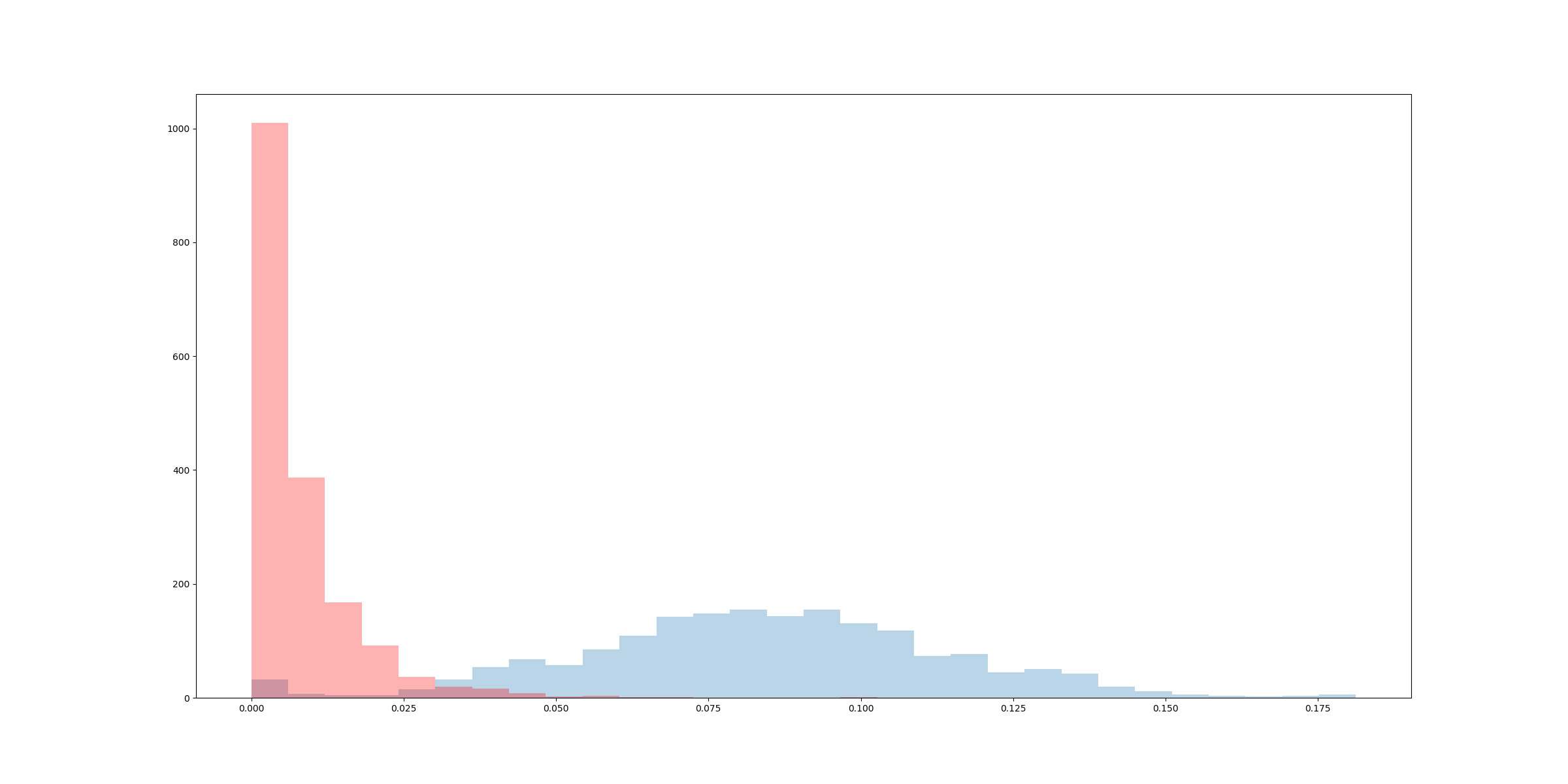}
\caption{Histogram of $0.7$-persistence of IGSM-based adversarial examples (red) and natural examples (blue) on MNIST. 
}

\label{fig:IGSMpersistenceMNIST}
\end{figure}

Next, we investigate the relationship of network complexity and $(\gamma,\sigma)$-stability by revisiting the now classic work of \citet{szegedy2013} on adversarial examples. 

Table \ref{table1} recreates and adds on to part of \cite[Table 1]{szegedy2013} in which networks of differing complexity are trained and attacked using L-BFGS. The table contains new columns showing the average $0.7$-persistence for both natural and adversarial examples for each network, as well as the average distortion for the adversarial examples. The distortion is the $\ell^2$-norm divided by square root of the dimension $n$. The first networks listed are of the form FC10-k, and are fully connected single layer ReLU networks that map each input vector $x \in \RR^{784}$ to an output vector $y \in \RR^{10}$ with a regularization added to the objective function of the form $\lambda\Norm{w}_2/N$, where $\lambda = 10^{-k}$ and $N$ is the number of parameters in the weight vector $w$ defining the network. The higher values of $\lambda$ indicate more regularization.  

FC100-100-10 and FC200-200-10 are fully connected networks with 2 hidden layers (with 100 and 200 nodes, respectively) with regularization added for each layer of perceptrons with the $\lambda$ for each layer equal to $10^{-5}, 10^{-5}$, and  $10^{-6}$. Training for these networks was conducted with a fixed number of epochs (typically 21). For the bottom half of Table \ref{table1}, we also considered networks with four convolutional layers plus a max-pooling layer connected by ReLU to a fully connected hidden layer with increasing numbers of channels denoted as as ``C-Ch,'' where C reflects that this is a CNN and Ch denotes the number of channels. A more detailed description of these networks can be found in Appendix \ref{appendix:CNNs}.

\begin{table}[ht]
\centering
\caption{Recreation of \citet[Table 1]{szegedy2013} for the MNIST dataset.  For each network, we show Testing Accuracy (in \%), Average Distortion ($\|x'-x\|_2/\sqrt{n}$ for adversarial example $x'$ with starting image $x$) of adversarial examples, and new columns show average $0.7$-persistence values for natural (Nat) and adversarial (Adv) images. 300 natural and 300 adversarial examples generated with L-BFGS were used for each aggregation.}
\label{table1}
\begin{tabular}{lllll}
\toprule
Network & Test Acc & Avg Dist & Persist (Nat) & Persist (Adv) \\
\midrule
FC10-4 & 92.09 & 0.123 & 0.93 & 1.68\\
FC10-2 & 90.77 & 0.178 & 1.37 & 4.25\\
FC10-0 & 86.89 & 0.278 & 1.92 & 12.22\\
FC100-100-10 & 97.31 & 0.086 & 0.65 & 0.56 \\
FC200-200-10 & 97.61 & 0.087 & 0.73 & 0.56 \\
\midrule
C-2 & 95.94 & 0.09 & 3.33 & 0.027 \\
C-4 & 97.36 & 0.12 & 0.35 & 0.027 \\
C-8 & 98.50 & 0.11 & 0.43  & 0.0517 \\
C-16 & 98.90 & 0.11 & 0.53 & 0.0994 \\
C-32 & 98.96 & 0.11 & 0.78 & 0.0836 \\
C-64 & 99.00 & 0.10 & 0.81 & 0.0865 \\
C-128 & 99.17 & 0.11 & 0.77 & 0.0883 \\
C-256 & 99.09 & 0.11  & 0.83 & 0.0900 \\
C-512 & 99.22 & 0.10 & 0.793 & 0.0929 \\

\bottomrule
\end{tabular}
\end{table}

The main observation from Table \ref{table1} is that for higher complexity networks,
adversarial examples tend to have smaller persistence than natural examples. Histograms reflecting these observations can be found in the supplemental material. 
Another notable takeaway is that for models with fewer effective parameters, the attack distortion necessary to generate a successful attack is so great that the resulting image is often more stable than a natural image under that model, as seen particularly in the FC10 networks. Once there are sufficiently many parameters available in the neural network, we found that both the average distortion of the adversarial examples and the average $0.7$-persistence of the adversarial examples tended to be smaller. This observation is consistent with the idea that networks with more parameters are more likely to exhibit decision boundaries with more curvature.

\subsection{Results on ImageNet}

For ImageNet \citep{Imagenet-old}, we used pre-trained ImageNet classification models, including alexnet \citep{alexnet} and vgg16 \citep{simonyan2014very}.

We then generated attacks based on the ILSVRC 2015 \citep{ILSVRC15}
validation images for each of these networks using a variety of modern
attack protocols, including Fast Gradient Sign Method (FGSM
\citep{goodfellow_explaining_2014}), Momentum Iterative FGSM (MIFGSM
\citep{dongMIFGSM}), Basic Iterative Method (BIM
\citep{kurakin_adversarial_2016}), Projected Gradient Descent (PGD
\citep{madry_towards_2017}), Randomized FGSM (R+FGSM
\citep{tramer2018ensemble}), and Carlini-Wagner (CW
~\citet{carlini_towards_2016}). These were all generated using the
TorchAttacks by \citet{kim2021torchattacks} toolset.

\subsubsection{Investigation of $(\gamma, \sigma)$-stability on ImageNet}

In this section, we show the results of Gaussian neighborhood sampling in ImageNet. Figures \ref{fig:imagenet_adv} and \ref{fig:persistent_interpimage} arise from vgg16 and adversarial examples created with BIM; results for other networks and attack strategies are similar, with additional figures in the supplementary material. Figure \ref{fig:imagenet_adv} (left) begins with an image $x$ with label \texttt{goldfinch}. For each equally spaced $\sigma\in(0,2)$, 100 i.i.d. samples were drawn from the Gaussian distribution $N(x,\sigma^2I)$, and the counts of the vgg16 classification for each label are shown. In Figure \ref{fig:imagenet_adv} (right), we see the same plot, but for an adversarial example targeted at the class \texttt{indigo\_bunting}, which is another type of bird, using the BIM attack protocol. 

The key observation in Figure \ref{fig:imagenet_adv} is that the frequency of the class of the adversarial example (\texttt{indigo\_bunting}, shown in red) falls off much quicker than the class for the natural example (\texttt{goldfinch}, shown in black). In this particular example, the original class appears again after the adversarial class becomes less prevalent, but only for a short period of $\sigma$, after which other classes begin to dominate. In some examples the original class does not dominate at all after the decline of the adversarial class. The adversarial class almost never dominates for a long period of $\sigma$. 

\begin{figure}[ht]
\centering
\includegraphics[width = .49\textwidth]{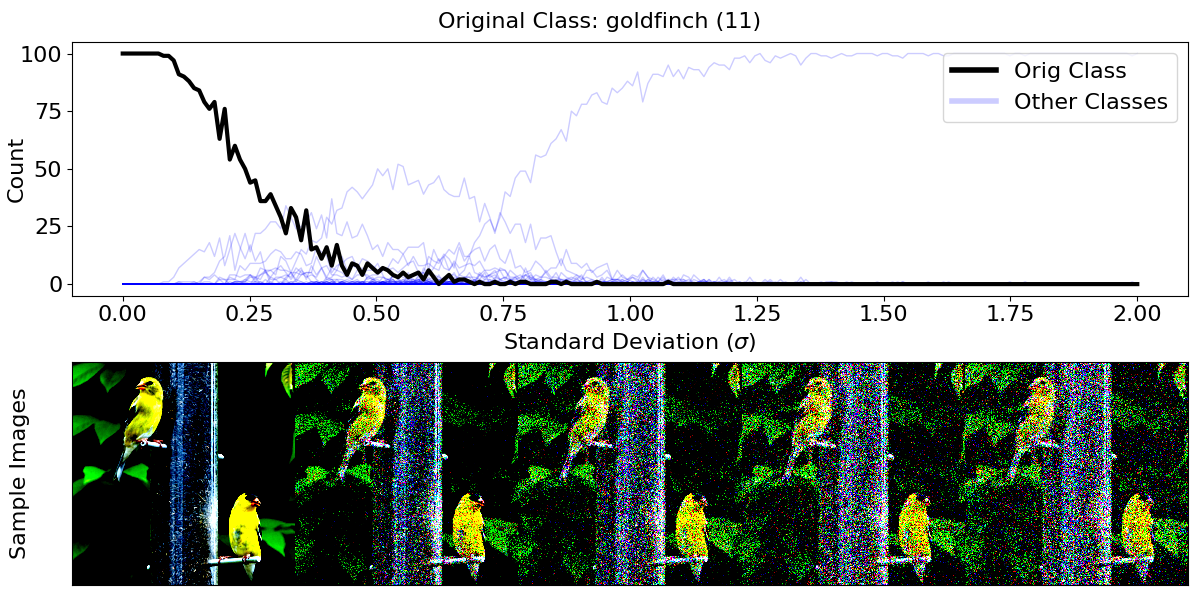}
\includegraphics[width = .49\textwidth]{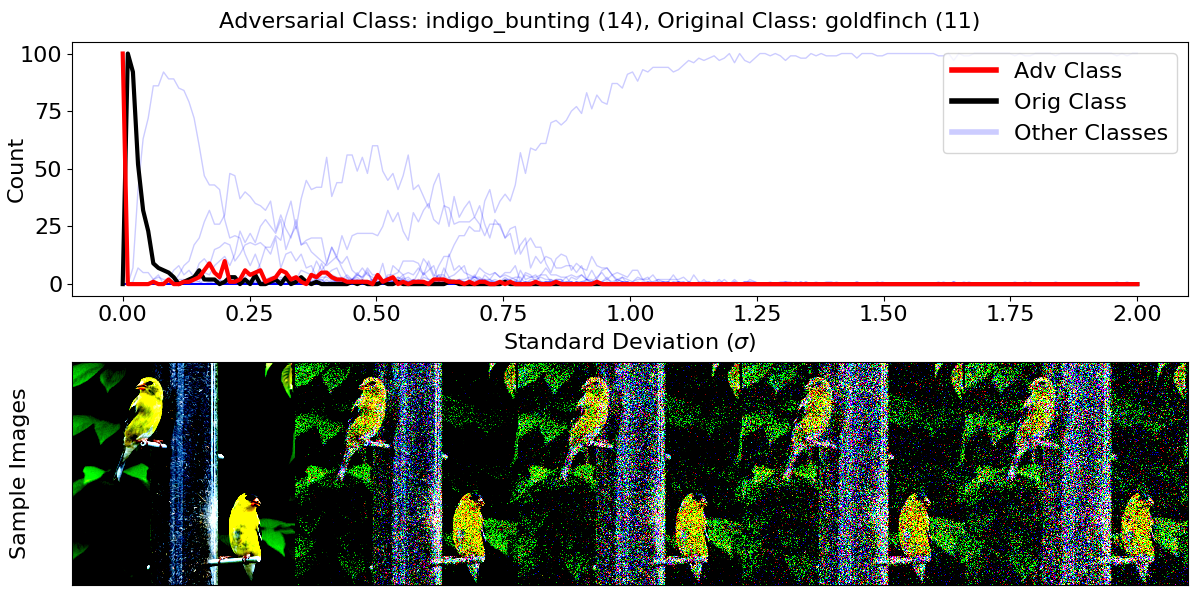}

\caption{Frequency of each class in Gaussian samples with increasing variance around a \texttt{goldfinch} image (left) and an adversarial example of that image targeted at the \texttt{indigo\_bunting} class and calculated using the BIM attack (right). Bottoms show example sample images at different standard deviations for natural (left) and adversarial (right) examples.}
\label{fig:imagenet_adv}
\end{figure}

\subsubsection{Persistence of adversarial examples on ImageNet}

Figure \ref{fig:persistent_interpimage} shows a plot of the $0.7$-persistence along the straight-line path between a natural example and adversarial example as parametrized between $0$ and $1$. It can be seen that the dropoff of persistence occurs precisely around the decision boundary. This indicates some sort of curvature favoring the class of the natural example, since otherwise the persistence would be roughly the same as the decision boundary is crossed.

\begin{figure}[ht]
\centering
\includegraphics[width = \textwidth]
{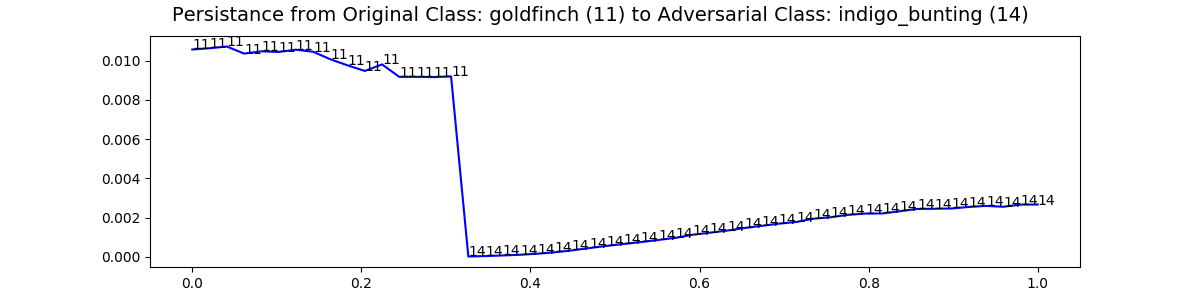}
\caption{The $0.7$-persistence of images along the straight line path from an image in class \texttt{goldfinch} (11) to an adversarial image generated with BIM in the class \texttt{indigo\_bunting} (14) on a vgg16 classifier. The classification of each image on the straight line is listed as a number so that it is possible to see the transition from one class to another. The vertical axis is $0.7$-persistence and the horizontal axis is progress towards the adversarial image.}\label{fig:persistent_interpimage}
\end{figure}

An aggregation of persistence for many randomly selected images from the \texttt{goldfinch} class in the validation set for Imagenet are presented in Table \ref{TAB:PersistenceAlexVGG}. 
\begin{table}[!ht]
\centering

\begin{tabular}{llll}
\toprule
Network/Method & Avg Dist & Persist (Nat) & Persist (Adv) \\
\midrule
alexnet (total) & 0.0194 & 0.0155 & 0.0049 \\ 
\:\: BIM        & 0.0188 & 0.0162 & 0.0050 \\ 
\:\: MIFGSM     & 0.0240 & 0.0159 & 0.0053 \\ 
\:\: PGD        & 0.0188 & 0.0162 & 0.0050 \\ 
\midrule
vgg16   (total) & 0.0154 & 0.0146 & 0.0011 \\ 
\:\: BIM        & 0.0181 & 0.0145 & 0.0012 \\ 
\:\: MIFGSM     & 0.0238 & 0.0149 & 0.0018 \\ 
\:\: PGD        & 0.0181 & 0.0145 & 0.0012 \\ 
\bottomrule
\end{tabular}
\caption{The $0.7$-persistence values for natural (Nat) and
  adversarial (Adv) images along with average distortion for
  adversarial images of alexnet and vgg16 for attacks generated with
  BIM, MIFGSM, and PGD on images from class \texttt{goldfinch}
  targeted toward other classes from the ILSVRC 2015 classification
  labels.} \label{TAB:PersistenceAlexVGG}
\end{table}
For each image of a \texttt{goldfinch} and for each network of alexnet and vgg16, attacks were prepared to a variety of 28 randomly selected targets using a BIM, MIFGSM, PGD, FGSM, R+FGSM, and CW attack strategies. The successful attacks were aggregated and their $0.7$-persistences were computed using the Bracketing Algorithm along with the $0.7$-persistences of the original images from which each attack was generated. Each attack strategy had a slightly different mixture of which source image and attack target combinations resulted in successful attacks. The overall rates for each are listed, as well as particular results on the most successful attack strategies in our experiments, BIM, MIFGSM, and PGD. The results indicate that adversarial images generated for these networks (alexnet and vgg16) using these attacks were less persistent, and hence less stable, than natural images for the same models. 

\subsection{Decision Boundary Interpolation and Angle Measurement} \label{subsec:dbe}

\begin{figure}[!ht]
\centering\includegraphics[width=0.50\linewidth, trim=1.5cm 1.5cm 2cm 2cm, clip]{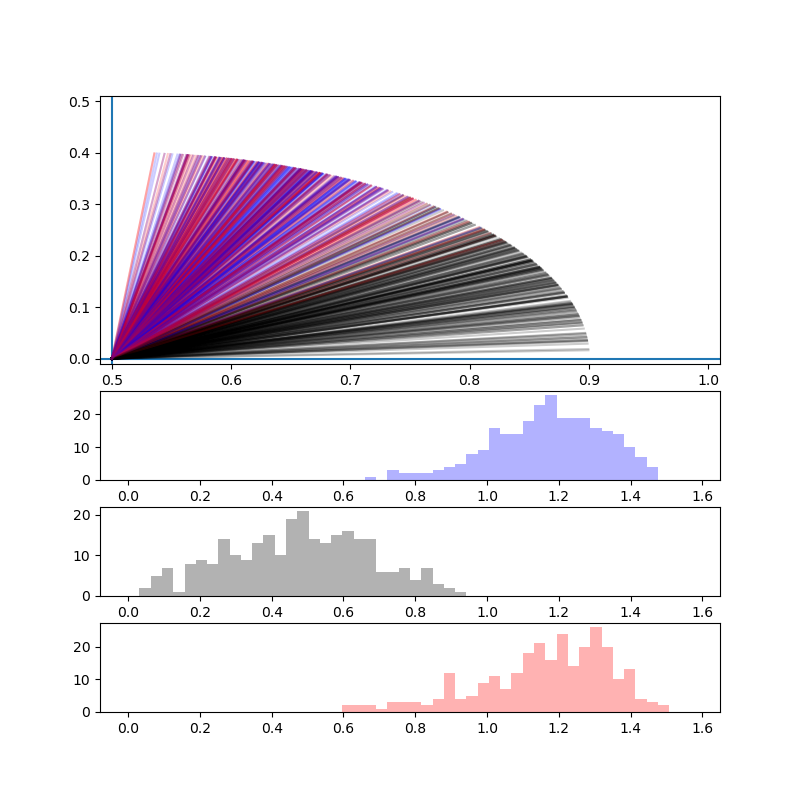}\includegraphics[width=0.50\linewidth, trim=1.5cm 1.5cm 2cm 2cm, clip]{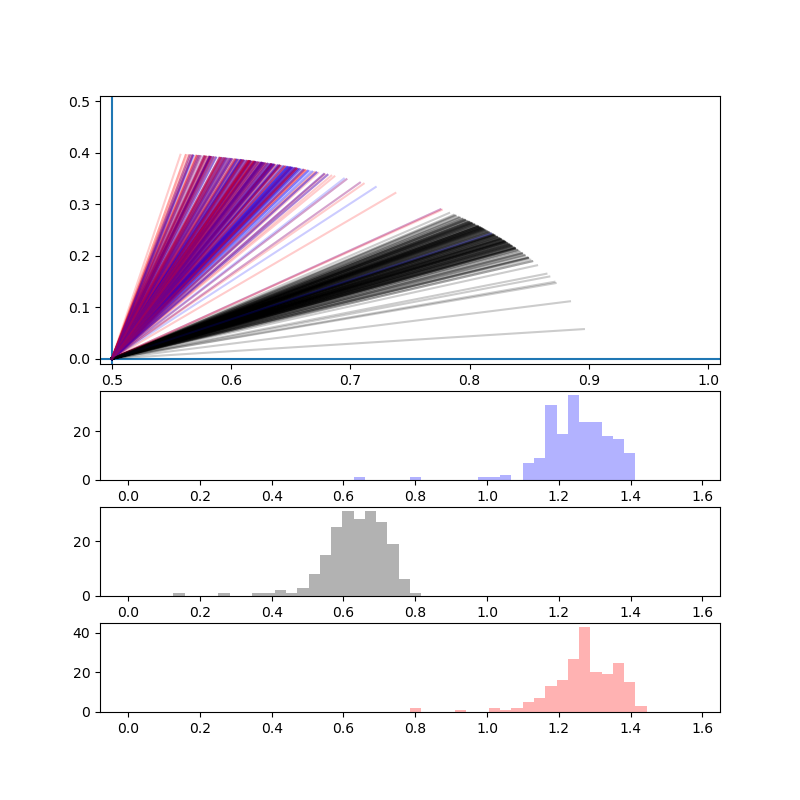}

\caption{Decision boundary incident angles between test to test
  interpolation and a computed normal vector to the decision boundary
  images (left) and between test and adversarial images
  (right).}
\label{fig:dba}
\end{figure}

In order to understand this sudden drop in persistence across the decision boundary observed in Figure ~\ref{fig:persistent_interpimage}, we will investigate incident angle of the interpolation with the decision boundary. In order to measure these angles, we must first interpolate along the decision boundary between two points. We will do this for pairs of test and test and pairs of test and adversary. In both cases, we will use a bracketing algorithm along the interpolation from candidate points to identify a point within machine-precision of the decision boundary $x_b$. 

Next, we will take 5000 samples from a Gaussian centered at this point with small standard deviation $\sigma = 10^{-6}$. Next, for each sample, we will perform an adversarial attack in order to produce a corresponding point on the opposite side of the decision boundary. Now for this new pair (sample and attacked sample), we will repeat the interpolation bracketing procedure in order to obtain the projection of this sample onto the decision boundary along the attack trajectory. Next, we will use singular value decomposition (SVD) on the differences between the projected samples and our decision boundary point $x_b$  to compute singular values and vectors from these projected samples. We will use the right singular vector corresponding with the smallest singular value as an approximation of a normal vector to the decision boundary at $x_b$. This point is difficult to compute due to degeneracy of SVD for small singular values, however in our tests, this value could be computed to a precision of 0.003. We will see that this level of precision exceeds that needed for the angles computed with respect to this normal vector sufficiently. 

We will examine decision boundary incident angles in Figure.~\ref{fig:dba} where angles (plotted Top) are referenced to decision boundary so $\pi/2$ radians (right limit of plots) corresponds with perfect orthogonality to decision boundary. Lines and histograms measure angles of training gradients (Top) linear interpolant (Middle) and adversarial gradients (Bottom). $x$ and $y$ axes are the axes of the unit-circle so angles can be compared. All angles are plotted in the upper-right quadrant for brevity. The lower plots are all histograms with their $y$ axes noting counts and their x-axes showing angles all projected to the range from 0 to $\pi/2$.

From the plots in Figure~\ref{fig:dba} we notice that neither training gradients
nor adversarial gradients are orthogonal to the decision
boundary. From a theory perspective, this is possible because this
problem has more than 2 classes, so that the decision boundary
includes $(0.34, 0.34, 0.32)$ and $(0.4, 0.4, 0.2)$. That is to say
that the level set definition of the decision boundary has degrees of
freedom that do not require orthogonality of gradients. More
interestingly, both natural and adversarial linear interpolants tend
to cross at acute angles with respect to the decision boundary, with
adversarial attacks tending to be closer to orthogonal. This suggests
that obliqueness of the decision boundary with respect to test points
may be related to adversarial vulnerability. We will leverage this understanding with manifold alignment to see if constraining gradients to a lower dimensional manifold, and thus increasing orthogonality of gradients will increase robustness. 

\subsection{Manifold Alignment on MNIST via PCA} \label{subsec:mae}

In order to provide an empirical measure of alignment, we first require a well defined image manifold.
The task of discovering the true structure of \textit{k}-dimensional manifolds in $\mathds{R}^d$ given a set of points sampled on the manifold has been studied previously \citep{khoury2018}.
Many algorithms produce solutions which are provably accurate under data density constraints.
Unfortunately, these algorithms have difficulty extending to domains with large $d$ due to the curse of dimensionality.
Our solution to this fundamental problem is to sidestep it entirely by redefining our dataset.
We begin by projecting our data onto a well known low dimensional manifold, which we can then measure with certainty.
\begin{figure}[ht]
\begin{center}

    \includegraphics[width=0.45\linewidth]{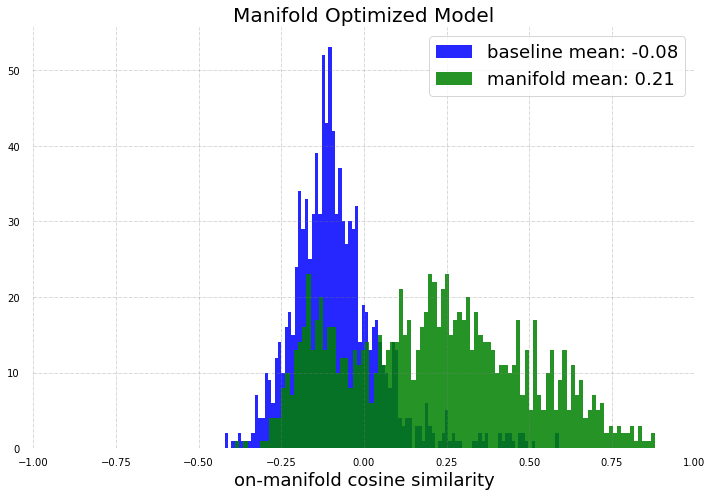}\includegraphics[width=0.45\linewidth]{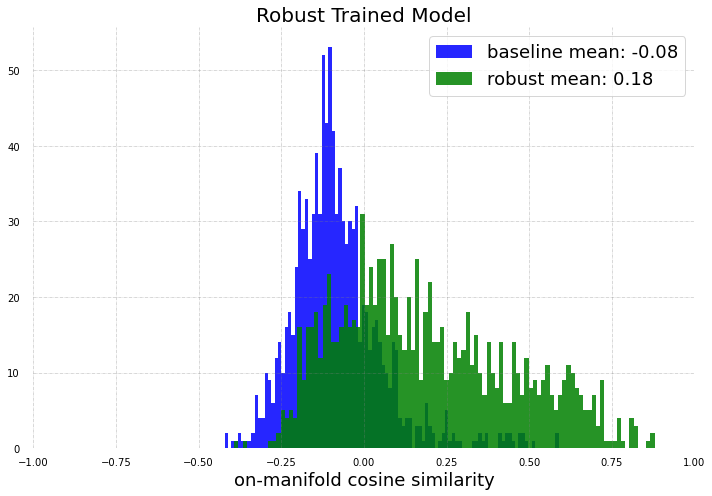}
\end{center}
    \caption{Comparison of manifold aligned components between baseline network, robust trained models, and models with manifold aligned gradients. Large values indicate higher alignment with the manifold. $Y$-axes for both plots are histogram counts.}
    \label{fig:hist_cosine}
\end{figure}

\begin{figure}[ht]
    \centering
    \includegraphics[width=0.45\linewidth]{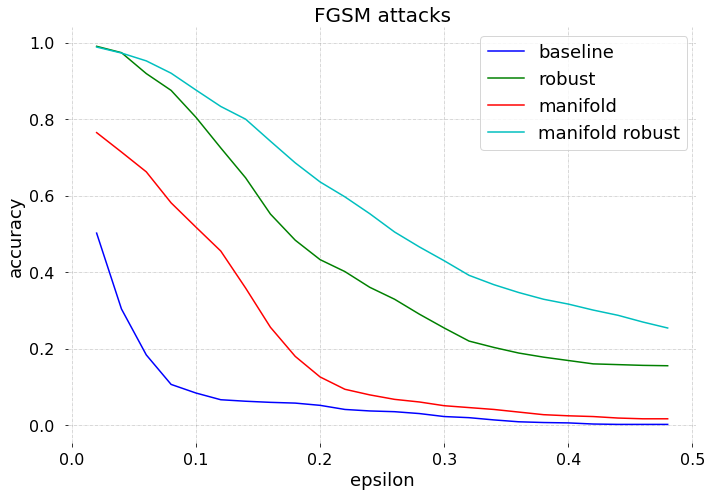}\includegraphics[width=0.45\linewidth]{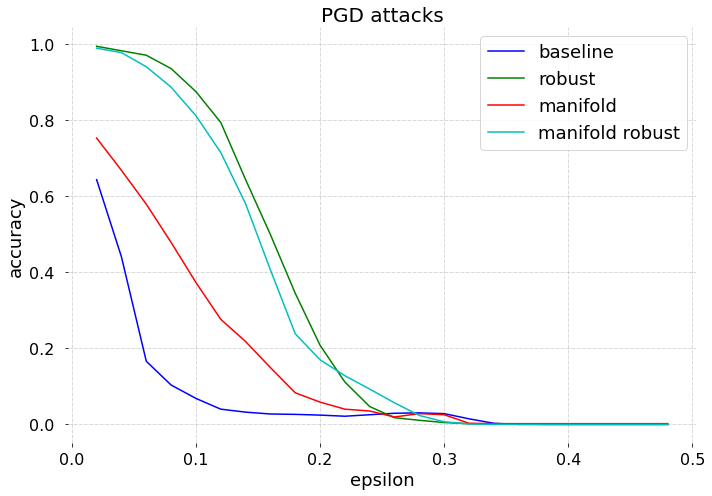}
    \caption{Comparison of adversarial robustness for PMNIST models under various training conditions.}
    \label{fig:model_robustness}
\end{figure}

We first fit a PCA model on all training data, using $k$ components
for each class to form a component matrix $W$, where $k << d$.
Given the original dataset $X$, we create a new dataset $X_{\mathcal{M}} := \{x \times \textbf{W}^T \times \textbf{W} : x \in X \}$.
We will refer to this set of component vectors as $\textbf{W}$.
Because the rank of the linear transformation matrix, $k$, is defined lower than the dimension of the input space, $d$, this creates a dataset which lies on a linear subspace of $\mathds{R}^d$.
This subspace is defined by the span of $X \times \textbf{W}^T$ and any vector in $\mathds{R}^d$ can be projected onto it.
Any data point drawn from $\{z \times \textbf{W}^T : z \in \mathds{R}^k \}$ is considered a valid datapoint.
This gives us a continuous linear subspace which can be used as a data manifold.

Given that it our goal to study the simplest possible case, we chose MNIST as the dataset to be projected and selected $k = 28$ components.
We refer to this new dataset as Projected MNIST (PMNIST).
The true rank of PMNIST is lower than that of the original MNIST data, meaning there was information lost in this projection.
The remaining information we found is sufficient to achieve 92\% accuracy using a feed forward ANN with one hidden layer, otherwise known as a Multilayer Perceptron (MLP), and the resulting images retain their semantic properties as in Figure \ref{fig:perception}. This figure shows attacks performed using PGD using the $l_\infty$ norm. Visual evidence of manifold alignment is often subjective and difficult to quantify. This example is provided as a baseline to substantiate our claim that our empirical measurements of alignment are valid.

\begin{figure*}[ht]
    \centering
    \includegraphics[width=0.25\linewidth]{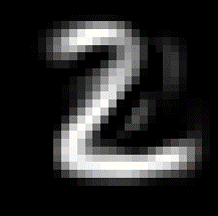}
    \includegraphics[width=0.25\linewidth]{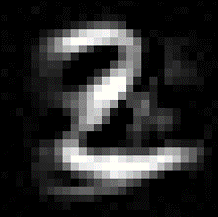}
    \includegraphics[width=0.25\linewidth]{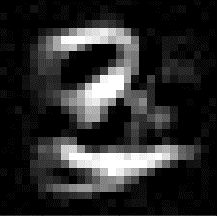}
    \caption{Visual example of manifold aligned gradient model transforming 2 into 3. Original PMNIST image on left, center image is center point between original and attacked, on right is the attacked image. }
    \label{fig:perception}
\end{figure*}

\subsection{Manifold Aligned Gradients} \label{subsec:ma}

Component vectors extracted from the original dataset are used to project gradient examples onto our pre-defined image manifold.

Given a gradient example $\nabla_x = \frac{\partial f_\theta(x, y)}{\partial x}$ where $f_\theta$ represents a neural network parameterized by weights $\theta$, $\nabla_x$ is transformed using the coefficient vectors \textbf{W}.

\begin{equation}
    \rho_x = \nabla_x \times \textbf{W}^T \times \textbf{W}    
\end{equation}
The projection of the original vector onto this new transformed vector
will be referred to as $P_{\mathcal{M}}$.
The ratio of norms of this projection gives a metric of manifold alignment:
\begin{equation}
    \frac{|| \nabla_x || }{||P_{\mathcal{M}}(\nabla_x )||}.
  \label{equ:ratio}
\end{equation}
This gives us a way of measuring the ratio between components of the gradient which are linearly dependent and independent from the manifold.
Additionally, both cosine similarity and the vector rejection were also tested but the norm ratio we found to be the most stable in training.
We use this measure as both a metric and a loss, allowing us to optimize the following objective:
\begin{equation}
  \mathds{E}_{(x,y) \sim \mathcal{D}} \left[ L(\theta, x,y)  + \alpha \frac{|| \nabla_x || }{||P_{\mathcal{M}}(\nabla_x )||} \right]
  \label{equ:loss}
\end{equation}

Where $L(\theta, x, y)$ represents our classification loss term and $\alpha$ is a hyper parameter determining the weight of the manifold alignment loss term. In the case that the right-hand term disappears, we say we have a manifold aligned gradient (MAG). 

\subsection{Manifold Alignment Robustness Results}

All models were two layer MLPs with 1568 nodes in each hidden layer.
The hidden layer size was chosen as twice the input size.
This arrangement was chosen to maintain the simplest possible case.

Two types of attacks were leveraged in this study: fast gradient sign
method (FGSM) \citep{goodfellow_explaining_2014} which performs a
single step based on the sign of an adversarial gradient for each
input; and projected gradient descent (PGD) which performs gradient
descent on input data using adversarial gradients in order to produce
adversarial attacks \citep{madry_towards_2017}. For FGSM, $\varepsilon$ is a step size applied to the sign of the gradient so that for an input $x$, its attack $x'$ is calculated as:
\begin{align}
  x' = x + \epsilon \cdot \text{sign}(\nabla_x f(x))
\end{align}
which is a vector with $L^2$ norm of $\sqrt(n)$ where $n$ is the number of input dimensions. The distortion ($\|x'-x\|_2/\sqrt{n}$ for adversarial example $x'$ with starting image $x$) of an FGSM perturbation is exactly equal to the parameter $\epsilon$ used to determine the step. For PGD, steps are computed iteratively using gradient descent. If their distortion is larger than $\epsilon$, they are shortened so that they have distortion of exactly $\epsilon$.
A total of four models were trained and evaluated on these attacks: Baseline (no adversarial training and no MAG loss term), Robust (adversarial examples added during training, but no MAG loss term), MAG (MAG loss term but no adversarial examples in training), and MAG Robust (both MAG loss term and adversarial examples in training).
All models, including the baseline, were trained on PMNIST (a fixed
permutation is applied to the training and test images of the MNIST dataset).
``Robust" in our case refers to models trained with new adversarial
examples labeled for their class \emph{before} perturbation during
each epoch consistent with ~\citet{tramer2019adversarial}. All adversarial training was conducted using adversarial examples with a maximum distortion of $\epsilon = 0.1$.

Figure \ref{fig:hist_cosine} shows the cosine similarity of the
gradient and its projection onto the reduced space $W$ on the testing set of PMNIST for both the MAG model and Robust model. Higher values (closer to 1) indicate the model is more aligned with the manifold. Both Robust and MAG models here show greater manifold alignment than the Baseline. This demonstrates that adversarial training leads to manifold aligned gradients.

Figure \ref{fig:model_robustness} shows the adversarial robustness of each model. Attacks are prepared using a range of a distortion parameter epsilon. For FGSM, the sign of the gradient is multiplied by each epsilon. For PGD, epsilon is determined by a weight on the $l_2$ norm term of the adversarial loss function. Many variations of the $l_2$ weight are performed, and then they are aggregated and the distance of each perturbation is plotted as epsilon. For both FGSM and PGD, we see a slight increase in robustness from using manifold aligned gradients. Adversarial training still improves performance significantly more than manifold aligned gradients. Another observation to note is that when both the manifold alignment and adversarial objective were optimized, increased robustness against FGSM attacks was observed. All robust models were trained using the $l_\infty$ norm at epsilon = 0.1. The fact that this performance increase is not shared by PGD training may indicate a relationship between these methods. We hypothesize that a linear representation of the image manifold is sufficient to defend against linear attacks such as FGSM, but cannot defend against a non-linear adversary.

\section{Conclusion}

In order to better understand the observed tendency for points near natural data to be classified similarly and points near
adversarial examples to be classified differently, we defined a notion of $(\gamma,\sigma)$-stability which is easily estimated by Monte Carlo sampling. For any data point $x$, we then define the $\gamma$-persistence to to be the smallest $\sigma_\gamma$ such that the probability of similarly classified data is at least $\gamma$ when sampling from Gaussian distributions with mean $x$ and standard deviation less than $\sigma_\gamma$. The persistence value can be quickly estimated by a Bracketing Algorithm. These two measures were considered with regard to both the MNIST and ImageNet datasets and with respect to a variety of classifiers and adversarial attacks. We found that adversarial examples were much less stable than natural examples in that the $0.7$-persistence for natural data was usually significantly larger than the $0.7$-persistence for adversarial examples. We also saw that the dropoff of the persistence tends to happen precisely near the decision boundary. Each of these observations is strong evidence toward the hypothesis that adversarial examples arise inside cones or high curvature regions in the adversarial class, whereas natural images lie outside such regions.

We also found that often the most likely class for perturbations of an adversarial examples is a class other than the class of the original natural example used to generate the adversarial example; instead, some other background class is favored. In addition, we found that some adversarial examples may be more stable than others, and a more detailed probing using the concept of $(\gamma,\sigma)$-stability and the $\gamma$-persistence statistic may be able to help with a more nuanced understanding of the geometry and curvature of the decision boundary. Although not pursued here, the observations and statistics used in this paper could potentially be used to develop methods to detect adversarial examples as in \citep{crecchi2019,frosst2018,hosseini2019odds,Lee2018ASU,qin2020,roth19aodds} and others. As with other methods of detection, this may be susceptible to adaptive attacks as discussed by ~\citet{tramer2020adaptive}.

For the future, we have made several observations: We found that some adversarial examples may be more stable than others. More detailed probing using the concept of $(\gamma,\sigma)$-stability and the $\gamma$-persistence along linear interpolation between natural images and between natural and adversarial images reveals sharp drops in persistence. Sharp drops in persistence correspond with oblique angles of incidence between linear interpolation vectors and the decision boundary learned by neural networks. Combining these observations, we can form a conjecture: Adversarial examples appear to exist near regions surrounded by negatively curved structures bounded by decision surfaces with relatively small angles relative to linear interpolation among training and testing data. This conjecture compares with the dimpled manifold hypothesis ~\citep{shamir2021dimpled}, however our techniques provide geometric information that allows us to gain a more detailed analysis of this region than in that work. In addition, our analysis of manifold alignment of gradients reinforces the notion that the obliqueness we observe may be a property which can be isolated and trained out of neural networks to some extent. Future work should focus on refining this conjecture with further tools to complete the spatial and mathematical picture surrounding adversarial examples.

\chapter{An Exact Kernel Equivalence for Finite Classification Models} 
\label{Chapter4} 

\tikzset{>=latex} 

\usetikzlibrary{fadings}
\usetikzlibrary{arrows.meta}
\tikzset{arl/.style={line width=4pt, {-Latex[left]}, #1}}
\tikzset{arr/.style={line width=4pt, {-Latex[right]}, #1}}

\colorlet{veccol}{green!70!black}
\colorlet{vcol}{green!70!black}
\colorlet{xcol}{blue!85!black}
\colorlet{projcol}{xcol!60}
\colorlet{unitcol}{xcol!60!black!85}
\colorlet{myblue}{blue!70!black}
\colorlet{myred}{red!90!black}
\colorlet{ntk}{red!90!white}
\colorlet{dpk}{blue!70!white}
\colorlet{epkt}{green!40!black}
\colorlet{diffpk}{orange}
\colorlet{epk}{black!70!gray}
\colorlet{mypurple}{blue!50!red!80!black!80}
\tikzstyle{vector}=[->,line width=0.65mm, xcol]
\usetikzlibrary{intersections, pgfplots.fillbetween}



\theoremstyle{remark}

\newcounter{remcounter}
\newcommand*{\remlabel}[1]{\refstepcounter{remcounter}\theremcounter\label{#1}}
\newcommand*{\remref}[1]{\ref{#1}}

In the process of trying to understand adversarial examples in a
geometric sense, a question arises: How can we directly extract
geometric properties from machine learning models. In this line of
thought, kernel methods are particularly appealing because they
implicitly define a spatial transform which is used to make
predictions. The kernel, a symmetric positive-definite bilinear map,
at the core of all kernel methods can be written as an inner-product
in an appropriate Hilbert space for all problems. This is a spatial
metric! Furthermore, kernel methods make predictions by comparing a
test point (using the kernel) with all known training points:

\begin{align}
  P(x) = b + \sum_i K(x, x_i)
\end{align}

Any prediction can be decomposed into the \emph{spatial}
contribution from each training point. The value of this property of kernel machines inspired a careful review of the Neural-Tangent-Kernel
and related literature eventually leading to a formulation posed by ~\citet{domingos2020every}. 

The resulting paper was accepted to the
archival proceedings track of the Topology, Algebra, and Geometry
Workshop at the International Conference on Machine Learning (ICML)
2023. In this work, we propose the first exact path kernel
representation for general gradient trained classifiers. The primary
derivation and proof was written by Brian Bell and the supporting
implementation and work were mostly conducted by Michael Geyer. The
central focus of the paper is on the derivation and demonstration that
this method works in practice. The interest that gave rise for this
approach comes from the fact that kernel methods and more specifically
bilinear map based models decompose their predictions into a
contribution from each of their training data. The Neural Tangent
Kernel (NTK) is an interesting tool, but predicated on too many
approximation assumptions. Exact formulation allows a much more
solid foundation for analyzing neural networks and suggests the
possibility that predictions can be decomposed using this
framework. The implementation and application of this method to real
machine learning models and tasks demonstrates that it is not only a
theoretical framework; it is practical! As stated above, this paper
was accepted for archival publication at the Topology Algebra and
Geometry (TAG) workshop at ICML 2023 in Honolulu, Hawaii.

\section{Introduction}

This study investigates the relationship between kernel methods and finite parametric models. To date, interpreting the predictions of complex models, like neural networks, has proven to be challenging. Prior work has shown that the inference-time predictions of a neural network can be exactly written as a sum of independent predictions computed with respect to each training point. We formally show that classification models trained with cross-entropy loss can be exactly formulated as a kernel machine. It is our hope that these new theoretical results will open new research directions in the interpretation of neural network behavior.

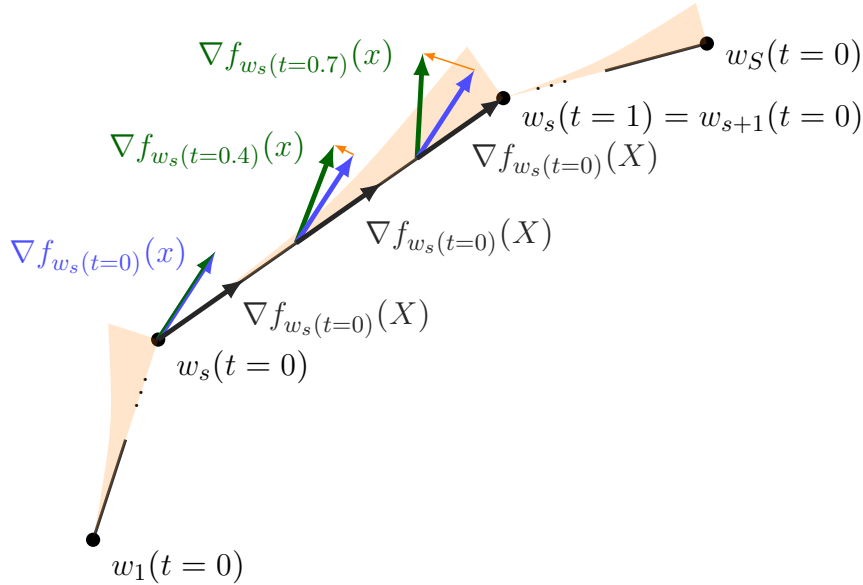
\begin{figure}[!ht]
\centering
\begin{tikzpicture}[scale=1.4]
\def\bang{72}
\def\ang{35}
\def\lang{15}
\def\mang{57}
\def\sang{69}
\def\kang{90}
\def\tang{87}
\def\dang{34}
\def\xs{1.0}
\def\xa{4}
\def\xe{1.0}
\def\xd{0.75}
\def\xn{0.5}
\def\xi{0.4}
\coordinate  (s1) at (0,0);
\coordinate  (si1) at ($ (s1) + (\bang:1) $);
\coordinate  (s2) at ($ (si1) + (\bang:\xs) $);
\coordinate  (s3) at ($ (s2) + (\ang:\xa) $);
\coordinate  (si3) at ($(s3) + (\lang:\xs) $);
\coordinate  (s4) at ($(si3) + (\lang:\xs) $);

\path (si1) -- (s2) node [midway, sloped] (elip) {\ldots};
\path (s3) -- (si3) node [midway, sloped] (elip) {\ldots};
\node[fill=black,circle,inner sep=1.9] (s01) at (s1) {};
\node [below right=0mm and 1mm, rotate=\kang-90] (ss1) at (s01) {$w_1(t=0)$};
\node[fill=black,circle,inner sep=1.9] (s02) at (s2) {};
\node [below right=0mm and 1mm, rotate=\kang-90] (ss2) at (s02) {$w_s(t=0)$};
\node[fill=black,circle,inner sep=1.9] (s03) at (s3) {};
\node [below right=-1mm and 1mm, rotate=\kang-90] (ss3) at (s3) {$w_s(t=1)=w_{s+1}(t=0)$}; 
\node[fill=black,circle,inner sep=1.9] (s04) at (s4) {};
\node [below right=-2mm and 1mm, rotate=\kang-90] (ss4) at (s04) {$w_S(t=0)$};
\draw[black, line width=0.4mm, draw opacity=0.3] (s2) -- (s3);
\draw[black, line width=0.4mm, draw opacity=0.3] (s1) -- (si1);
\draw[black, line width=0.4mm, draw opacity=0.3] (si3) -- (s4);

\coordinate (p1) at (s2);
\coordinate (p1a) at ($ (p1)+(\ang:\xe) $);
\coordinate (p1b) at ($ (p1)+(\mang:\xe) $);
\coordinate (p2) at ($ (p1a)+(\ang:0.6) $);
\coordinate (p2a) at ($ (p2)+(\ang:\xe) $);
\coordinate (p2b) at ($ (p2)+(\sang:\xe) $); 
\coordinate (p2c) at ($ (p2)+(\mang:\xe) $); 
\coordinate (p3) at ($ (p2a)+(\ang:\xi) $);
\coordinate (p3a) at ($ (p3)+(\ang:\xe) $);
\coordinate (p3b) at ($ (p3)+(\tang:\xe) $);
\coordinate (p3c) at ($ (p3)+(\mang:\xe) $);
\coordinate (n0) at ($ (s1)+(\bang:\xn) $);
\coordinate (n1) at ($ (p1)+(\bang:\xn) $);
\coordinate (n2) at ($ (p2)+(\bang:\xn) $);
\coordinate (n3) at ($ (p3)+(\bang:\xn) $);
\coordinate (d0) at ($ (s1)+(\dang:\xd) $);
\coordinate (d1) at ($ (p1)+(\dang:\xd) $);
\coordinate (d2) at ($ (p2)+(\dang:\xd) $);
\coordinate (d3) at ($ (p3)+(\dang:\xd) $);

\draw[name path=pb, black, line width=0.4mm, draw opacity=0.7] (s2) -- (s3);
\draw[black, line width=0.4mm, draw opacity=0.7] (s1) -- (si1);
\path[name path=fb, black, line width=0.4mm, draw opacity=0.7] (s1) -- (s2);
\draw[name path=gb, black, line width=0.4mm, draw opacity=0.7] (si3) -- (s4);  
  \draw [name path=pa, thick , ->, opacity=0.0]
  (p1) .. controls ($ (s2) + (37:2.1) $)  .. ($ (s3) + (125:0.6) $);
  \draw [name path=fa, thick , ->, opacity=0.0]
  (s1) .. controls ($ (s1) + (80:1.1) $)  .. ($ (s2) + (162:0.50) $);
  \draw [name path=ga, thick , ->, opacity=0.0]
  (s3) .. controls ($ (s3) + (20:1) $)  .. ($ (s4) + (105:0.4) $);
\tikzfillbetween[of=pa and pb]{orange, opacity=0.2};
\tikzfillbetween[of=fa and fb]{orange, opacity=0.2}; 
\tikzfillbetween[of=ga and gb]{orange, opacity=0.2};
    
\draw[vector, ->, dpk] (p1) -- (p1b) node[scale=1,above left=-5mm and 2mm, draw opacity=0.5] {$\nabla f_{w_s(t=0)}(x)$};
\begin{scope}
\clip (p1) -- (p1b) -- ($(p1b) + (90+\mang:0.1) $) -- ($ (p1) + (90+\mang:0.1) $);

\draw[vector, ->, epkt] (p1) -- (p1b) node[scale=1,above left=-5mm and 2mm, draw opacity=0.5] {$\nabla f_{w_s(t=0)}(x)$};

\end{scope}
\draw[vector, ->, epk] (p1) -- (p1a) node[scale=1,below right=1mm and -2mm, draw opacity=0.5, rotate=\kang-90] {$\nabla f_{w_s(t=0)}(X)$};
\draw[vector, ->, epkt] (p2) -- (p2b) node[scale=1,above left=-5mm and 2mm, draw opacity=0.5] {$\nabla f_{w_s(t=0.4)}(x)$};
\draw[vector, ->, dpk] (p2) -- (p2c) node[scale=1,above left=-5mm and 2mm, draw opacity=0.5] {};
\draw[vector, ->, epk] (p2) -- (p2a) node[scale=1,below right=3mm and -4mm, draw opacity=0.5, rotate=\kang-90] {$\nabla f_{w_s(t=0)}(X)$};
\draw[vector, ->, epkt] (p3) -- (p3b) node[scale=1,above left=-5mm and 2mm, draw opacity=0.5] {$\nabla f_{w_s(t=0.7)}(x)$};
\draw[vector, ->, dpk] (p3) -- (p3c) node[scale=1,above left=-5mm and 2mm, draw opacity=0.5] {};
\draw[vector, ->, epk] (p3) -- (p3a) node[scale=1,below right=4mm and -6mm, draw opacity=0.5, rotate=\kang-90] {$\nabla f_{w_s(t=0)}(X)$};
\draw[->,line width=0.2mm, xcol, diffpk] (p2c) -- (p2b) node[scale=1,below right=1mm and -2mm, draw opacity=0.5] {};
\draw[->,line width=0.2mm, xcol, diffpk] (p3c) -- (p3b) node[scale=1,below right=1mm and -2mm, draw opacity=0.5] {};

    

\end{tikzpicture}
\caption{Comparison of test gradients used by Discrete Path Kernel (DPK) from prior work (Blue) and the Exact Path Kernel (EPK) proposed in this work (green) versus total training vectors (black) used for both kernel formulations along a discrete training path with $S$ steps. Orange shading indicates cosine error of DPK test gradients versus EPK test gradients shown in practice in Fig.~\ref{fig:error}. }
\label{fig:vecs}
\end{figure}

\begin{figure}[!ht]
        \centering
        \includegraphics[width=1.01\linewidth]{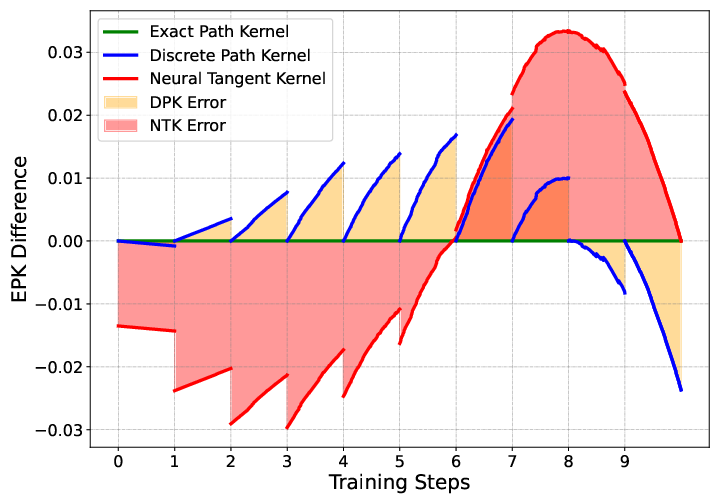}

        \caption{Measurement of gradient alignment on test points across the training path. The EPK is used as a frame of reference. The y-axis is exactly the difference between the EPK and other representations. For example $EPK-DPK = \langle \phi_{s,t}(X), \phi_{s,t}(x) - \phi_{s,0}(x) \rangle$ (See Definition 3.4). Shaded regions indicate total accumulated error. Note: this is measuring an angle of error in weight space; therefore, equivalent positive and negative error will not result in zero error.}
        \label{fig:error}
\end{figure}

There has recently been a surge of interest in the connection between neural networks and kernel methods ~\citep{bietti2019bias, du2019graphntk, tancik2020fourierfeatures, abdar2021uq, geifman2020similarity, chen2020generalized, alemohammad2021recurrent}. Much of this work has been motivated by the neural tangent kernel (NTK), which describes the training dynamics of neural networks in the infinite limit of network width ~\citep{jacot2018neural}. The NTK has the form
\begin{align}
  f_{w}(x) = \sum_{i=1}^N L'(f_w(x_i), y_i) \langle \nabla_w f_w(x), \nabla_w f_w(x_i) \rangle
\end{align}
where $L'(a, b) = \dfrac{\partial L(a, b)}{\partial a}$. 
We argue that many intriguing behaviors arise in the \emph{finite} parameter regime which should, for example, satisfy the universal law of robustness proposed by ~\citet{DBLP:conf/nips/BubeckS21}. 
All prior works, to the best of our knowledge, appeal to discrete approximations of the kernel corresponding to a neural network. 
Specifically, prior approaches are derived under the assumption that training step size is small enough to guarantee close approximation of a gradient flow
~\citep{ghojogh2021, shawe2004kernel, zhao2005extracting}.

In this work, we show that the simplifying assumptions used in prior works (i.e. infinite network width and infinitesimal gradient descent steps) are not necessary. Our \textbf{Exact Path Kernel (EPK)} provides the first exact method to study the behavior of finite-sized neural networks used for classification.
Previous results are limited in application ~\citep{incudini2022quantum} due to dependence of the kernel on test data unless strong conditions are imposed on the training process as by ~\citet{chen2021equivalence}. We show, however, that the training step sizes used in practice do not closely follow this gradient flow, introducing significant error into all prior approaches (Figure~\ref{fig:error}).

Our experimental results build on prior studies attempting to evaluate empirical properties of the kernels corresponding to finite neural networks ~\citep{DBLP:conf/iclr/LeeBNSPS18, chen2021equivalence}. While the properties of infinite neural networks are fairly well understood ~\citep{neal1996priors}, we find that the kernels learned by finite neural networks have non-intuitive properties that may explain the failures of modern neural networks on important tasks such as robust classification and calibration on out-of-distribution data.

This paper makes the following significant theoretical and experimental contributions:
\begin{enumerate}
    \item We prove that finite-sized neural networks trained with finite-sized gradient descent steps and cross-entropy loss can be exactly represented as kernel machines using the EPK. Our derivation incorporates a previously-proposed path kernel, but extends this method to account for practical training procedures ~\citep{domingos2020every, chen2021equivalence}.
  
    \item We demonstrate that it is computationally tractable to estimate the kernel underlying a neural network classifier, including for small convolutional computer vision models.
    \item We compute Gram matrices using the EPK and use them to illuminate prior theory of neural networks and their understanding of uncertainty. 
    \item We employ Gaussian processes to compute the covariance of a neural network's logits and show that this reiterates previously observed shortcomings of neural network generalization.
\end{enumerate}

\section{Related Work}

The main focus for this work is the connection between neural networks and kernel machines in the finite case. While there is a large body of work on the infinite-parameter version of this problem, primarily the neural tangent kernel (NTK) of ~\citet{jacot2018neural}, less work has been focused on finitely parameterized models.   
In the finite setting, some works \citep{domingos2020every, chen2021equivalence} have established an equivalence between artificial neural networks and support vector machines i.e. kernel machines for models which follow continuous gradient flows. One intriguing aspect is a connection between these kernel representations and neural networks as maximum margin classifiers ~\citep{chen2021equivalence, chizat2020maxmargin}. ~\citet{shah2021input} demonstrate that this maximum margin classifier exists in Wasserstien space; however, they also show that model gradients may not contain the required information to represent this. However the discrete approximations proposed in these works for models which are trained with finitely many discrete steps have highly unstable approximation error. 

We will refer to the previous finite approximations of ~\citet{domingos2020} and ~\citet{chen2021equivalence} as the Discrete Path Kernel (DPK). The highly unstable approximation error that results from reliance of the DPK on problems trained with discrete first order numerical optimization e.g. forward Euler, raises concerns regarding the applicability of the continuous path kernel to practical scenarios ~\citep{incudini2022quantum}.
Moreover, the formulation of the sample weights and bias term in the DPK depends on its test points. ~\citet{chen2021equivalence} propose that this can be addressed, in part, by imposing restrictions on the loss function used for training, but do not entirely disentangle the kernel formulation from sample importance weights on training points.

We address the limitations of \citet{domingos2020} and \citet{chen2021equivalence} in Subsection 
~\ref{subsec:disc}. By default, their approach produces a system which can be viewed as an ensemble of kernel machines, but without a single aggregated kernel which can be analyzed directly. ~\citet{chen2021equivalence} propose that the resulting sum over kernel machines can be formulated as a kernel machine so long as the sign of the gradient of the loss stays constant through training; however, we show that this is not necessarily a sufficient restriction. Instead, their formulation leads to one of several non-symmetric functions which can serve as a surrogate to replicate a given models behavior, but without retaining properties of a kernel.

\section{Theoretical Results}

Our goal is to show an equivalence between any given finite parametric model trained with gradient descent $f_w(x)$  (e.g. neural networks) and a kernel based prediction that we construct. We define this equivalence in terms of the output of the parametric model $f_w(x)$ and our kernel method in the sense that they form identical maps from input to output. In the specific case of neural network classification models, we consider the mapping $f_w(x)$ to include all layers of the neural network up to and including the log-softmax activation function. Formally:
\begin{definition}
A {kernel} is a function of two variables which is symmetric and positive semi-definite. 
\end{definition}

\begin{definition}
Given a Hilbert space $X$, a test point $x \in X$, and a training set $X_T = \{x_1,x_2,...,x_n\} \subset X$ indexed by $I$, a \emph{Kernel Machine} is a model characterized by 
\begin{align}
    \text{K}(x) = b + \sum_{i=1}^n a_i k(x,x_i)
\end{align}
where the $a_i \in \mathbb{R}$ do not depend on $x$, $b \in \mathbb{R}$ is a constant, and $k$ is a kernel ~\citep{rasmussen2006gaussian}.
\end{definition}

By Mercer's Theorem ~\citep{ghojogh2021} a kernel can be produced by composing an inner product on a Hilbert space with a mapping $\phi$ from the space of data into the chosen Hilbert space.
We use this property to construct a kernel machine of the following form.
\begin{align}
    \text{K}(x) = b + \sum_{i\in I} a_i \langle \phi(x), \phi(x_i) \rangle
\end{align}
Where $\phi$ is a function mapping input data into the weight space via gradients. Our $\phi$ will additionally differentiate between test and training points to resolve a discontinuity that arises under discrete training. 

\subsection{Exact Path Kernels}

 We first derive a kernel which is an exact representation of the change in model output over one training step, and then compose our final representation by summing along the finitely many steps.
Models trained by gradient descent can be characterized by a discrete set of intermediate states in the space of their parameters.
These discrete states are often considered to be an estimation of the gradient flow, however in practical settings where $\epsilon > 0$ these discrete states differ from the true gradient flow.
Our primary theoretical contribution is an algorithm which accounts for this difference by observing the true path the model followed during training.
Here we consider the training dynamics of practical gradient descent steps by integrating a discrete path for weights whose states differ from the gradient flow induced by the training set.

\textbf{Gradient Along Training Path vs Gradient Field:}
In order to compute the EPK, gradients on training data must serve two purposes. 
First, they are the reference points for comparison (via inner product) with test points. 
Second, they determine the path of the model in weight space. 
In practice, the path followed during gradient descent does not match the gradient field exactly. 
Instead, the gradient used to move the state of the model forward during training is only computed for finitely many discrete weight states of the model.
In order to produce a path kernel, we must \textit{continuously} compare the model's gradient at test points with \textit{fixed} training gradients along each discrete training step $s$ whose weights we we interpolate linearly by $w_s(t) = w_s - t(w_s - w_{s+1})$. We will do this by integrating across the gradient field induced by test points, but holding each training gradient fixed along the entire discrete step taken. This creates an asymmetry, where test gradients are being measured continuously but the training gradients are being measured discretely (see Figure~\ref{fig:vecs}).

To account for this asymmetry in representation, we will redefine our data using an indicator to separate training points from all other points in the input space.
\begin{definition}
\label{fpm}
Let $X$ be two copies of a Hilbert space $H$ with indices $0$ and $1$ so that $X = H \times \{0,1\}$. We will write $x \in H \times \{0,1\}$ so that $x = (x_H, x_I)$ (For brevity, we will omit writing $_H$ and assume each of the following functions defined on $H$ will use $x_H$ and $x_I$ will be a hidden indicator).
Let $ f_{w}$ be a differentiable scalar valued function on $H$ parameterized by $w \in \mathbb{R}^d$. Let $X_T = \{(x_i, 1)\}_{i=1}^M$ be a finite subset of $X$ of size $M$ with corresponding observations $Y_T = \{y_{x_i}\}_{i=1}^M$ with initial parameters $w_0$ so that there is a constant $b \in \mathbb{R}$ such that for all $x$, $ f_{w_0}(x) = b$. Let $L$ be a differentiable loss function of two values which maps $(f(x), y_x)$ into the positive real numbers. Starting with $f_{w_0}$, let $\{w_s\}$ be the sequence of points attained by $N$ forward Euler steps of fixed size $\varepsilon$ so that $w_{s+1} = w_{s} - \varepsilon \nabla L(f(X_T), Y_T)$. Let $w_s(t) = w_s + t(w_{s+1}-w_s)$. Let $x \in H \times \{0\}$ be arbitrary and within the domain of $f_w$ for every $w$. Then $f_{w_s(t)}$ is a \emph{finite parametric gradient model (FPGM)}. 
\end{definition}

\begin{definition}
\label{epk}

Let $f_{w_s(t)}$ be an FPGM with all corresponding assumptions. Then, for a given training step $s$, the \emph{exact path kernel} (EPK) can be written  
\begin{equation}
 K_{\text{EPK}}(x, x', s) = \int_0^1\langle \phi_{s,t}(x), \phi_{s,t}(x')\rangle dt
 \label{eq2}
\end{equation}
where
\begin{align}
\phi_{s, t}(x) &=  \nabla_w f_{w_s(t,x)} (x)\\
w_s(t) &= w_s - t(w_s - w_{s+1})\\
w_s(t,x) &= \begin{cases} w_s(0), & \text{if } x_I = 1\\ w_s(t), & \text{if } x_I = 0 \end{cases}
\end{align}
\textbf{Note:} $\phi$ is deciding whether to select a continuously or
discrete gradient based on whether the data is from the training or
testing copy of the Hilbert space $H$. This is due to the inherent
asymmetry that is apparent from the derivation of this kernel (see
Appendix section ~\ref{proof:eker}). This choice avoids potential discontinuity in the kernel output when a test set happens to contain training points. 
\end{definition}
\begin{restatable}{lemma}{ker}
The exact path kernel (EPK) is a kernel.
\end{restatable}

\begin{proof}
We must show that the associated kernel matrix $K_{\text{EPK}} \in \mathbb{R}^{n\times n}$ defined for an arbitrary subset of data $\{x_i\}_{i=1}^M \subset X$ as $K_{\text{EPK},i,j} = \int_0^1\langle \phi_{s,t}(x_i), \phi_{s,t}(x_j)\rangle dt$ is both symmetric and positive semi-definite.

Since the inner product on a Hilbert space $\langle \cdot, \cdot \rangle$ is symmetric and since the same mapping $\varphi$ is used on the left and right, $K_{\text{EPK}}$ is \textbf{symmetric}. 

To see that $K_{\text{EPK}}$ is \textbf{Positive Semi-Definite}, let $\alpha = (\alpha_1, \alpha_2, \dots, \alpha_n)^\top \in \mathbb{R}^n$ be any vector. We need to show that $\alpha^\top K_{\text{EPK}} \alpha \geq 0$. We have

\begin{align}
\alpha^\top K_{\text{EPK}} \alpha &= \sum_{i=1}^n \sum_{j=1}^n \alpha_i \alpha_j \int_0^1 \langle \phi_{s,t}(x_i), \phi_{s,t}(x_j)\rangle dt \\
&=   \int_0^1 \left\langle \sum_{i=1}^n \alpha_i \phi_{s,t}(x_i), \sum_{j=1}^n \alpha_j \phi_{s,t}(x_j)\right\rangle dt \\
&=   \int_0^1 \left\lvert \sum_{i=1}^n \alpha_i^2 \phi_{s,t}(x_i)^2\right\rvert dt 
\geq 0.
\end{align}
\end{proof}

\begin{restatable}[Exact Kernel Ensemble Representation]{theorem}{eker}
\label{thm:eker}
A model $f_{w_N}$ trained using discrete steps matching the conditions of the exact path kernel has the following exact representation as an ensemble of $N$ kernel machines:
\begin{equation}
f_{w_N} = \text{KE}(x) :=  \sum_{s = 1}^N \sum_{i = 1}^{M} a_{i,s} K_{\text{EPK}}(x, x_i, s) + b
\label{ensemble}
\end{equation}
where
\begin{align}
a_{i, s} &= -\varepsilon  L'(f_{w_s(0)}(x_i),  y_i) \\
b &= f_{w_0}(x)
\end{align}
Where $L'(a, b) = \dfrac{\partial L(a, b)}{\partial a}$
\end{restatable}

\begin{proof}
\label{proof:eker}
Let $f_{w}$ be a differentiable function parameterized by parameters $w$ which is trained via $N$ forward Euler steps of fixed step size $\varepsilon$ on a training dataset $X$ with labels $ Y$, with initial parameters $w_0$ so that there is a constant $b$ such that for every $x$, $f_{w_0}(x) = b$, and weights at each step ${w_s : 0 \leq s \leq N}$. Let $x \in X$ be arbitrary and within the domain of $f_w$ for every $w$. For the final trained state of this model $f_{w_N}$, let $y = f_{w_N}(x)$. 

For one step of training, we consider $y_s  = f_{w_s(0)}(x)$ and $y_{s+1} = f_{w_{s+1}}(x)$. We wish to account for the change $y_{s+1} - y_s$ in terms of a gradient flow, so we must compute $\dfrac{\partial y}{dt}$ for a continuously varying parameter $t$. Since $f$ is trained using forward Euler with a step size of $\varepsilon > 0$, this derivative is determined by a step of fixed size of the weights $w_s$ to $w_{s+1}$. We parameterize this step in terms of the weights:

\begin{align}
    \dfrac{d w_s(t)}{dt} &= (w_{s+1} - w_s)\\   
    \int_0^T \dfrac{d w_s(t)}{dt} dt &= \int_0^T (w_{s+1} - w_s)dt
\end{align}
Since $f$ is being trained using forward Euler, across the entire training set $X$ we can write:
\begin{align}
    \dfrac{d w_s(t)}{dt} &= -\varepsilon \nabla_w L(f_{w_s(0)}(X), y_i) = -\varepsilon  \sum_{i=1}^M  \dfrac{\partial L(f_{w_s(0)}(x_i),  y_i)}{\partial w} \label{eq10}
\end{align}
For the loss function $L(a, b)$ we will define its partial derivative with respect to the first variable $a$ as $L'(a, b) = \dfrac{\partial L(a, b)}{\partial b}$. Applying chain rule with this notation and the above substitution, we can write
\begin{align}
    \dfrac{d \hat y}{dt} = \dfrac{d f_{w_s(t)}(x)}{dt} &= \sum_{j = 1}^{d} \dfrac{\partial f}{\partial w_j} \dfrac{d w_j}{dt}\\
&= \sum_{j = 1}^{d} \dfrac{\partial f_{w_s(t)}(x)}{\partial w_j} \left(-\varepsilon \dfrac{\partial L(f_{w_s(0)}(X_T),  Y_T)}{\partial w_j}\right)\\
&= \sum_{j = 1}^{d} \dfrac{\partial f_{w_s(t)}(x)}{\partial w_j} \left(-\varepsilon \sum_{i = 1}^{M}L'(f_{w_s(0)}(x_i),  y_i) \dfrac{\partial  f_{w_s(0)}(x_i)}{\partial w_j}\right)\\
&= -\varepsilon \sum_{i = 1}^{M} L'(f_{w_s(0)}(x_i),  y_i) \sum_{j = 1}^{d} \dfrac{d f_{w_s(t)}(x)}{\partial w_j}  \dfrac{d f_{w_s(0)}(x_i)}{\partial w_j}\\
&= -\varepsilon \sum_{i = 1}^{M} L'(f_{w_s(0)}(x_i),  y_i) \nabla_w f_{w_s(t)}(x) \cdot \nabla_w f_{w_s(0)}(x_i)\label{eq11}\\
\end{align}
Using the fundamental theorem of calculus, we can compute the change in the model's output over step $s$
\begin{align}
    y_{s+1} - y_s &= \int_0^1 -\varepsilon \sum_{i = 1}^{M} L'(f_{w_s(0)}(x_i),  y_i)  \nabla_w f_{w_s(t)}(x) \cdot \nabla_w f_{w_s(0)}(x_i)dt\\
 &=  -\varepsilon \sum_{i = 1}^{M} L'(f_{w_s(0)}(x_i),  y_i) \left(\int_0^1\nabla_w f_{w_s(t)}(x)dt\right) \cdot \nabla_w f_{w_s(0)}(x_i)\\
\end{align}
For all $N$ training steps, we have
\begin{align*}
y_N &= b + \sum_{s=1}^N y_{s+1} - y_s\\
y_N &= b + \sum_{s = 1}^N -\varepsilon \sum_{i = 1}^{M} L'(f_{w_s(0)}(x_i),  y_i)
      \left(\int_0^1\nabla_w f_{w_s(t)}(x)dt\right) \cdot \nabla_w
      f_{w_s(0)}(x_i)\\   \label{eqint}
&= b + \sum_{i = 1}^{M}\sum_{s = 1}^N -\varepsilon  L'(f_{w_s(0)}(x_i),  y_i)  \int_0^1\left\langle \nabla_w f_{w_s(t,x)}(x), \nabla_w f_{w_s(t,x_i)}(x_i) \right\rangle dt\\ 
&= b + \sum_{i = 1}^{M}\sum_{s = 1}^N a_{i, s}  \int_0^1 \left\langle \phi_{s,t}(x), \phi_{s,t}(x_i)\right\rangle dt
\end{align*}
Since an integral of a symmetric positive semi-definite function is still symmetric and positive-definite, each step is thus represented by a kernel machine. 

\end{proof}

Having established this representation, we can introduce $P_S(t)$, the
training path which is composed by placing each of the $S$ training
steps end-to-end. We can rewrite ~\ref{eqint} by combining $\sum_{s =
  1}^S$ and $\int_0^1 dt$ into a single integral $\int_{P_S}$:

\begin{align}
y_N &= b +  -\varepsilon \sum_{i = 1}^{M} \int_{P_S} 
      L'(f_{w_s(0)}(x_i),  y_i)
      \left(\nabla_w f_{w_s(t)}(x)\right) \cdot \nabla_w
      f_{w_s(0)}(x_i)\\   \label{eqint}
\end{align}

Rewriting this way, we can re-evaluate another assumption made during
our statement of this theorem, the bias term $b$. We have forced $b$
to be a constant in order to give our representation a chance of
reducing to the form of a kernel machine in a later theorem
~\ref{thm:ekr}. Let us relax this and replace $b$ with
$f_{w_0(0)}(x)$. We can see that this new representation no longer
requires assumptions about $f_{w_0(0)}(x)$ which gives us a much more
general representation which includes most ANNs in practice, 
\begin{align}
f_{w_F(0)}(x) = f_{w_0(0)}(x) +  -\varepsilon \sum_{i=1}^N \sum_s^S \int_{P_S} L'(f_{w_s(0)}(x_i),  y_i)
      \langle\nabla_w f_{w_s(t)}(x), \nabla_w
      f_{w_s(0)}(x_i)\rangle dt \label{eqint}.
\end{align}


\textbf{Remark ~\remlabel{rem:init}} Note that in this formulation, $b$ depends on the test point $x$.
In order to ensure information is not being leaked from the kernel into this bias term the model $f$ must have constant output for all input. 
When relaxing this property, to allow for models that have a non-constant starting output, but still requiring $b$ to remain constant, we note that this representation ceases to be exact for all $x$.
The resulting approximate representation has logit error bounded by its initial bias which can be chosen as $b = \text{mean}(f_{w_0(0)}(X_T))$.
Starting bias can be minimized by starting with small parameter values which will be out-weighed by contributions from training.
In practice, we sidestep this issue by initializing all weights in the final layer to $0$, resulting in $b=\text{log}(\text{softmax}(0))$, thus removing $b$'s dependence on $x$.

\textbf{Remark ~\remlabel{rem:exact}} 
The exactness of this proof hinges on the \emph{separate} measurement of how the model's parameters change.
The gradients on training data, which are fixed from one step to the next, measure how the parameters are changing.
This is opposed to the gradients on test data, which are \textit{not} fixed and vary with time.
These measure a continuous gradient field for a given point.
We are using interpolation as a way to measure the difference between the step-wise linear training path and the continuous loss gradient field.

\begin{restatable}[Exact Kernel Machine Reduction]{theorem}{ekr}
\label{thm:ekr}
Let $\nabla L(f(w_{s}(x), y)$ be constant  across steps $s$, $(a_{i,s}) = (a_{i,0})$. Let the kernel across all $N$ steps be defined as $K_{\text{NEPK}}(x,x') = \sum_{s = 1}^N a_{i,0} K_{\text{EPK}}(x, x', s).$ Then the exact kernel ensemble representation for $f_{w_N}$ can be reduced exactly to the kernel machine representation:
\begin{equation}
f_{w_N}(x) = \text{KM}(x) := b + \sum_{i = 1}^{M} a_{i,0} K_{\text{NEPK}}(x,x')
\label{exact}
\end{equation}
\end{restatable}

\textbf{Remark} These look a lot like sheaves ~\citep{huybrechts2010geometry} in the RKBS of functions integrated along discrete optimization paths in fact reduce to a sheaf in the reproducing kernel Hilbert space (RKHS) \citep{shilton_gradient_2023}. 

By combining theorems ~\ref{thm:eker} and ~\ref{thm:ekr}, we can construct an exact kernel machine representation for any arbitrary parameterized model trained by gradient descent which satisfies the additional property of having constant loss across training steps (e.g. any ANN using catagorical cross-entropy loss (CCE) for classification). This representation will produce exactly identical output to the model across  the model's entire domain. This establishes exact kernel-neural equivalence for classification ANNs. Furthermore, Theorem ~\ref{thm:eker} establishes an exact kernel ensemble representation without limitation to models using loss functions with constant derivatives across steps. It remains an open problem to determine other conditions under which this ensemble may be reduced to a single kernel representation.  

\subsection{Discussion}

The map $\phi_{s,t}(x)$ depends on both $s$ and $t$, which is non-standard but valid, however an important consequence of this formulation is that the output of this representation is not guaranteed to be continuous. This discontinuity is exactly measuring the error between the model along the exact path compared with the gradient flow for each step. 

We can write another function $k'$ which is continuous but not symmetric, yet still produces an exact representation:
\begin{align}
k'(x, x') = \langle \nabla_w f_{w_s(t)}(x), \nabla_w f_{w_s(0)}(x')\rangle.
\end{align}
The resulting function is a valid kernel if and only if for every $s$ and every $x$, 
\begin{align}
\label{eq:cond}
    \int_0^1 \nabla_w f_{w_s(t)}(x)dt = \nabla_w f_{w_s(0)}(x).
\end{align}

We note that since $f$ is being trained using forward Euler, we can write:
\begin{align}
    \dfrac{d w_s(t)}{dt} &= -\varepsilon \nabla_w L(f_{w_s(0)}(x_i), y_i). \label{dstep}
\end{align}
In other words, our parameterization of this step depends on the step size $\varepsilon$ and as $\varepsilon \to 0$, we have 
\begin{align}
    \int_0^1 \nabla_w f_{w_{s}(t)}(x)dt \approx \nabla_w f_{w_s(0)}(x).
\end{align}
In particular, given a model $f$ that admits a Lipshitz constant $K$ this approximation has error bounded by $\varepsilon K$ and a proof of this convergence is direct. 
This demonstrates that the asymmetry of this function is exactly measuring the disagreement between the discrete steps taken during training with the gradient field. 
This function is one of several subjects for further study, particularly in the context of Gaussian processes whereby the asymmetric Gram matrix corresponding with this function can stand in for a covariance matrix. It may be that the not-symmetric analogue of the covariance in this case has physical meaning relative to uncertainty.

\subsection{Independence from Optimization Scheme}
We can see that by changing equation ~\ref{dstep} we can produce an exact representation for any first order discrete optimization scheme that can be written in terms of model gradients aggregated across subsets of training data. This could include backward Euler, leapfrog, and any variation of adaptive step sizes. This includes stochastic gradient descent, and other forms of subsampling (for which the training sums need only be taken over each sample). One caveat is adversarial training, whereby the $a_i$ are now sampling a measure over the continuum of adversarial images. We can write this exactly, however computation will require approximation across the measure. Modification of this kernel for higher order optimization schemes remains an open problem.


\begin{figure*}[!ht]
    \centering
        \includegraphics[width=0.3\textwidth]{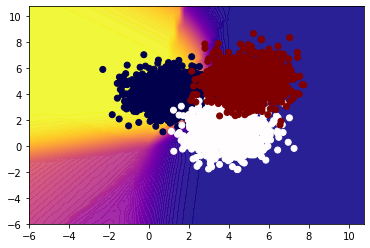}\includegraphics[width=0.3\textwidth]{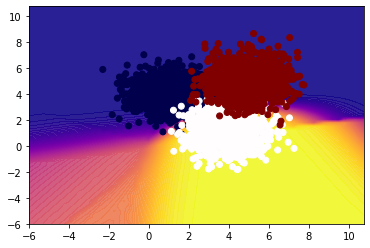}\includegraphics[width=0.3\textwidth]{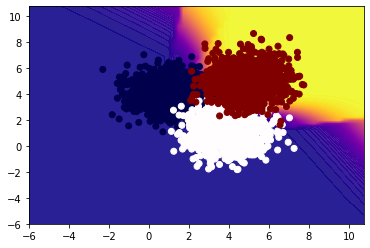}
    \caption{Updated predictions with kernel $a_i$ updated via gradient descent with training data overlaid for classes 1 (left), 2 (middle), and 3 (right). The high prediction confidence in regions far from training points demonstrates that the learned kernel is non-stationary. Axes are the $x$ and $y$ dimensions of a 100 dimensional problem with 3 gaussians whosse means are on this $x-y$ plane.}
    \label{fig:svm}
\end{figure*}
\subsection{Ensemble Reduction}
In order to reduce the ensemble representation of Equation~\eqref{ensemble} to the kernel representation of Equation~\eqref{exact}, we require that the sum over steps still retain the properties of the kernel (symmetry and positive semi-definiteness). In particular we require that for every subset of the training data ${x_i}$ and arbitrary ${\alpha_i}$ and ${\alpha_j}$, we have
\begin{align}
    \sum_{i=1}^n\sum_{j=1}^n \sum_{l=1}^M \sum_{s=1}^N \alpha_i \alpha_j a_{l, s}\int_{0}^1 K_{\text{EPK}}(x_i,x_j) dt \geq 0.
\end{align}
A sufficient condition for this reduction is that the gradient of the loss function does not change throughout training. This is the case for categorical cross-entropy where labels are in $\{0,1\}$. In fact, in this specific context the gradient of the loss function does not depend on $f(x)$, and are fully determined by the ground truth label, making the gradient of the cross-entropy loss a constant value throughout training.
Showing the positive-definiteness of more general loss functions (e.g. mean squared error) will likely require additional regularity conditions on the training path, and is left as future work.

\subsection{Relation to Prior Work}
\label{subsec:disc}
Constant sign loss functions have been previously studied by ~\citet{chen2021equivalence}, however the kernel that they derive for a finite-width case is of the form
\begin{align}
    K(x,x_i) =  \int_0^T | L'(f_t(x_i), y_i)| \langle \nabla_w f_t(x), \nabla_w f_t(x_i) \rangle dt
\end{align}
where $L'(a, b)$ is again $\dfrac{\partial L(a, b)}{\partial a}$. 
The summation across these terms satisfies the positive semi-definite requirement of a kernel. However the weight $|\nabla L(f_t(x_i), y_i)|$ depends on $x_i$ which is one of the two inputs. This makes the resulting function $K(x,x_i)$ asymmetric and therefore not a kernel.

\subsection{Uniqueness}
Uniqueness of this kernel is not guaranteed. 
The mapping from paths in gradient space to kernels is in fact a function, meaning that each finite continuous path has a unique exact kernel representation of the form described above. 
However, this function is not necessarily onto the set of all possible kernels. 
This is evident from the existence of kernels for which representation by a finite parametric function is impossible.
Nor is this function necessarily one-to-one since there is a continuous manifold of equivalent parameter configurations for neural networks.
For a given training path, we can pick another path of equivalent configurations whose gradients will be separated by some constant $\delta > 0$.
The resulting kernel evaluation along this alternate path will be exactly equivalent to the first, despite being a unique path. 
We also note that the linear path $l_2$ interpolation is not the only valid path between two discrete points in weight space.
An equally valid approach is following the changes in model weights along a path defined by Manhattan Distance, and will produce a kernel machine with equivalent outputs.
It remains an open problem to compute paths from two different starting points which both satisfy the constant bias condition from Definition~\eqref{epk} which both converge to the same final parameter configuration and define different kernels.

\section{Experimental Results}
    Our first experiments test the kernel formulation on a dataset 
    which can be visualized in 2d. These experiments serve as a sanity check
    and provide an interpretable representation of what the kernel is learning.
    \begin{figure}[!ht]
        \centering

        \includegraphics[width=0.40\textwidth]{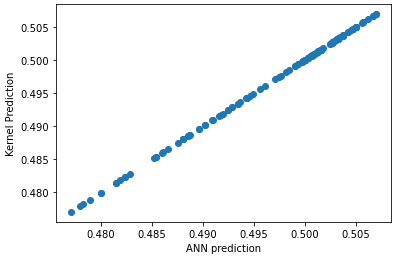}
        \caption{Class 1 EPK Kernel Prediction (Y) versus neural network prediction (X) for 100 test points, demonstrating extremely close agreement.}
        \label{fig:toymatch}
    \end{figure}

\subsection{Evaluating The Kernel} \label{subsec:evaluate}
A small test data set within 100 dimensions is created by generating 1000 random samples with means $(1,4,0,...)$, $(4,1,0,...)$ and $(5,5,0,...)$ and standard deviation $1.0$. These points are labeled according to the mean of the Gaussian used to generate them, providing 1000 points each from 3 classes. A fully connected ReLU network with 1 hidden layer is trained using categorical cross-entropy (CCE) and gradient descent with gradients aggregated across the entire training set for each step. We then compute the EPK for this network, approximating the integral from Equation~\ref{eq2} with 100 steps which replicates the output from the ReLU network within machine precision. The EPK (Kernel) outputs are compared with neural network predictions in Fig.~\ref{fig:toymatch} for class 1. Having established this kernel, and its corresponding kernel machine, one natural extension is to allow the kernel weights $a_i = L'(f_{w_s(t)}(x_i), y_i)$ to be replaced with kernel weights found through optimization. We perform this updating of the kernel weights using a SVM and present its predictions for each of three classes in Fig.~\ref{fig:svm}.

\subsection{Kernel Analysis}
Having established the efficacy of this kernel for model
representation, the next step is to analyze this kernel to understand
how it may inform us about the properties of the corresponding
model. In practice, it becomes immediately apparent that this kernel
lacks typical properties preferred when humans select
kernels. Fig.~\ref{fig:svm} shows that the weights of this kernel are
non-stationary on our small 2d example -- the model has very low uncertainty for predictions far away from training data. Next, we use this kernel to estimate uncertainty. Consistent with many other research works on Gaussian processes (GP) for classification (e.g. ~\citet{rasmussen2006gaussian}) we use a GP to regress to logits. Starting with sample predictions drawn using our kernel as a covariance matrix, we use Monte-Carlo to estimate posteriors with respect to probabilities (post-soft-max) for each prediction across a grid spanning the training points of our toy problem. The result is shown on the right-hand column of Fig.~\ref{fig:cov}. We can see that the kernel predictions are more confident (lower standard deviation) and stronger (higher kernel values) the farther they get from the training data in most directions. 

    \begin{figure}[h]
        \centering
        \includegraphics[width=0.24\textwidth]{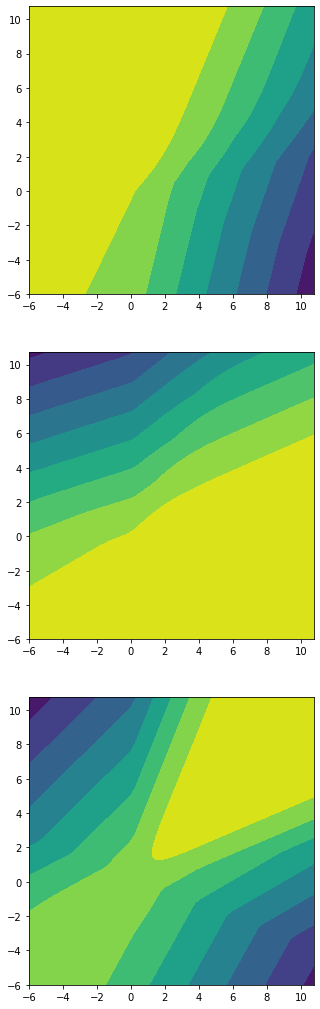}\includegraphics[width=0.24\textwidth]{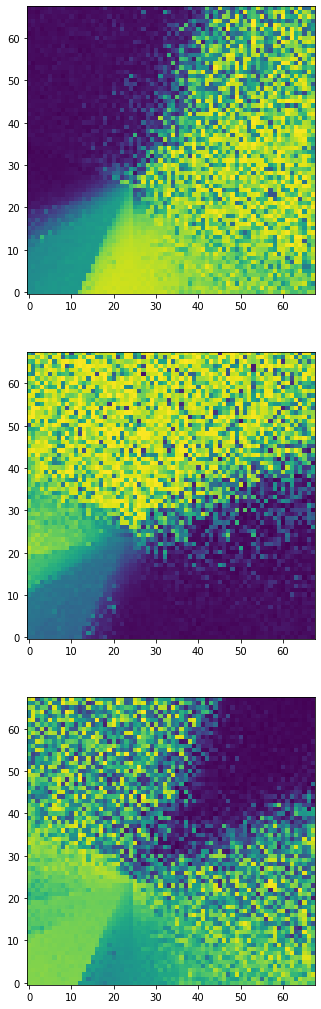}
        \caption{(left) Kernel predicted probabilities measured on a grid around the training set for our 2D problem. Bright yellow means high kernel value. (right) Monte-Carlo estimated standard deviation using the Gram matrices $G(X_{\text{grid}})$ with each element $G_{i, j} = K(x_i, x_j)$ generated using our kernel for the same grid as the kernel values as a prior. Yellow means high standard deviation, blue means low standard deviation.}
        \label{fig:cov}
    \end{figure}

In order to further understand how these strange kernel properties come about, we exercise another advantage of a kernel by analyzing the points that are contributing to the kernel value for a variety of test points. 
In Fig.~\ref{fig:points} we examine the kernel values for each of the training points during evaluation of three points chosen as the mean of the generating distribution for each class. 
The most striking property of these kernel point values is the fact that they are not proportional to the euclidean distance from the test point.
This appears to indicate a set of basis vectors relative to each test point learned by the model based on the training data which are used to spatially transform the data in preparation for classification. This may relate to the correspondence between neural networks and maximum margin classifiers discussed in related work ~\citep{chizat2020maxmargin, shah2021input}. 
Another more subtle property is that some individual data points, mostly close to decision boundaries, are slightly over-weighted compared to the other points in their class. 
This latter property points to the fact that during the latter period of training, once the network has already achieved high accuracy, only the few points which continue to receive incorrect predictions, i.e. caught on the wrong side of a decision boundary, will continue contributing to the training gradient and therefore to the kernel value.

    \begin{figure*}[ht]
        \centering
        \includegraphics[width=0.32\textwidth]{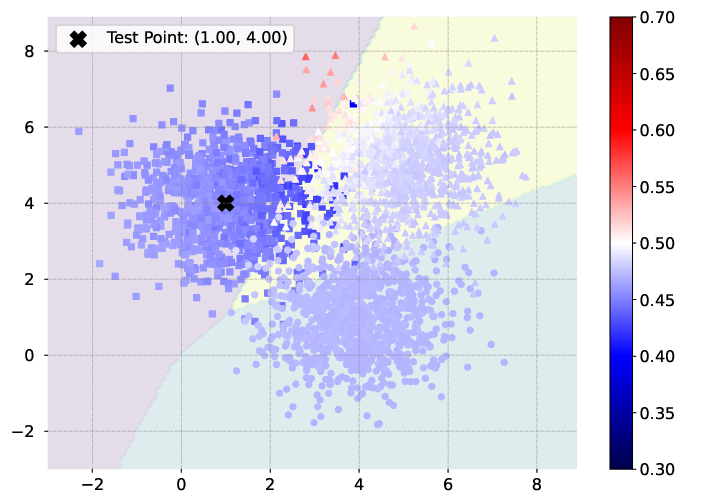}
        \includegraphics[width=0.32\textwidth]{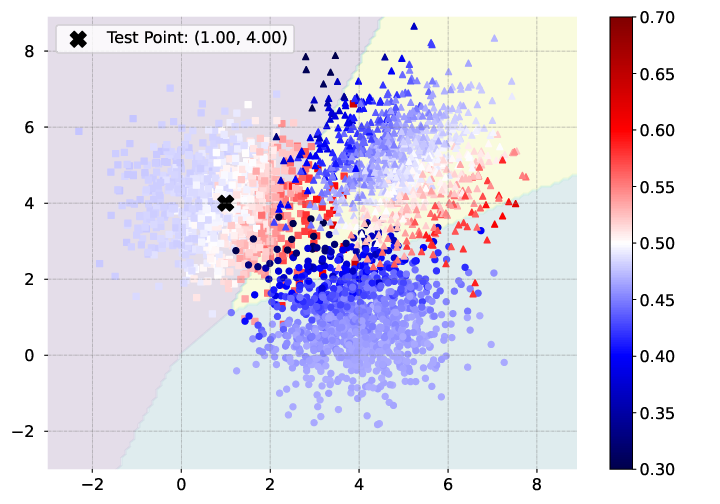}
        \includegraphics[width=0.32\textwidth]{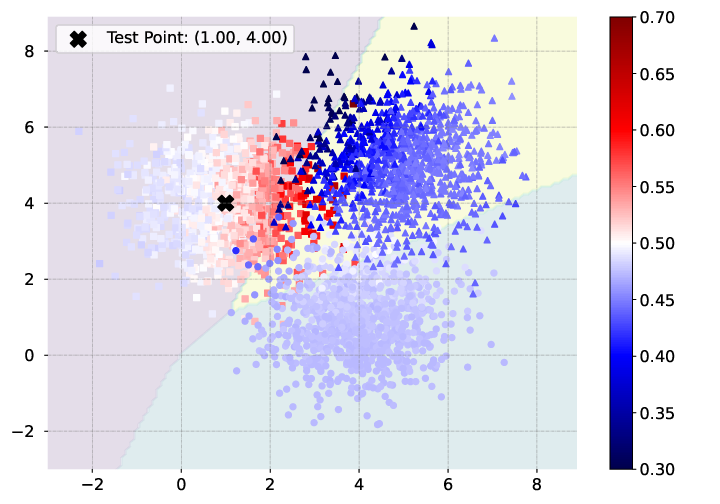}\\
        
        \includegraphics[width=0.32\textwidth]{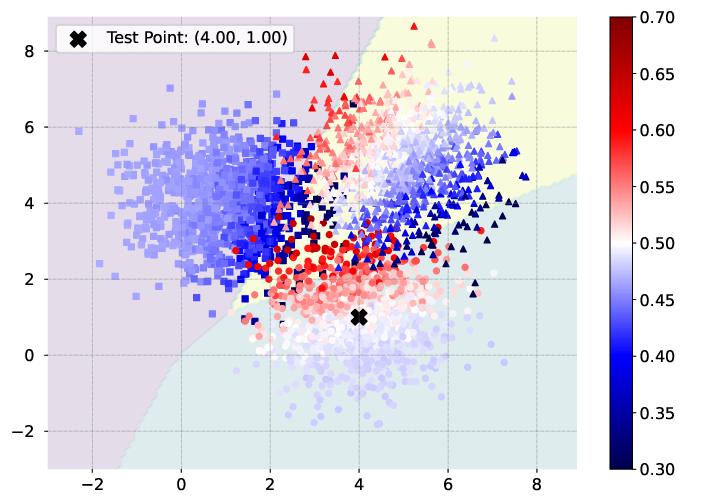}
        \includegraphics[width=0.32\textwidth]{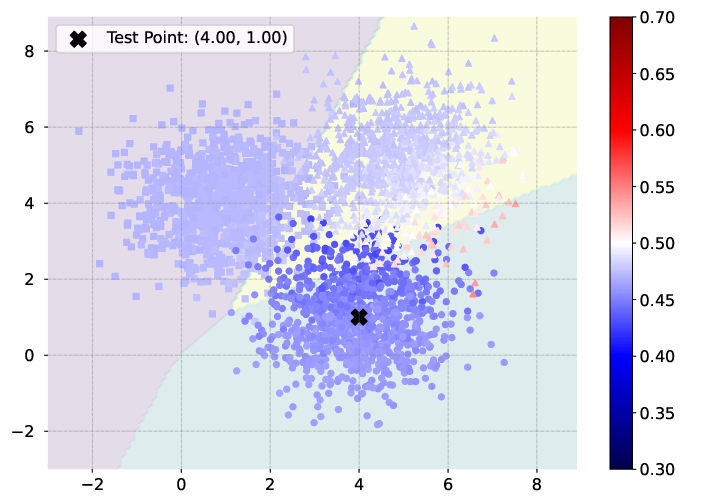}
        \includegraphics[width=0.32\textwidth]{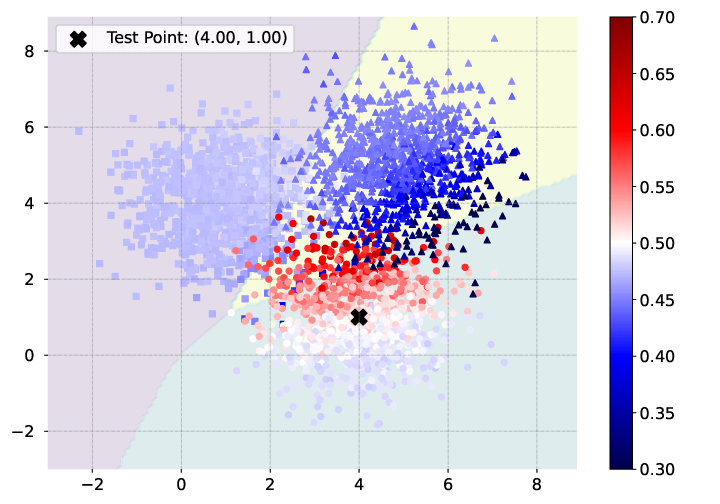}\\

        \includegraphics[width=0.32\textwidth]{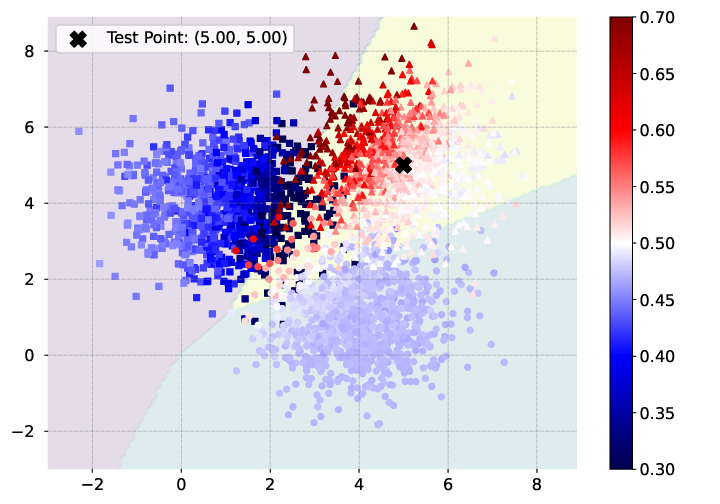}
        \includegraphics[width=0.32\textwidth]{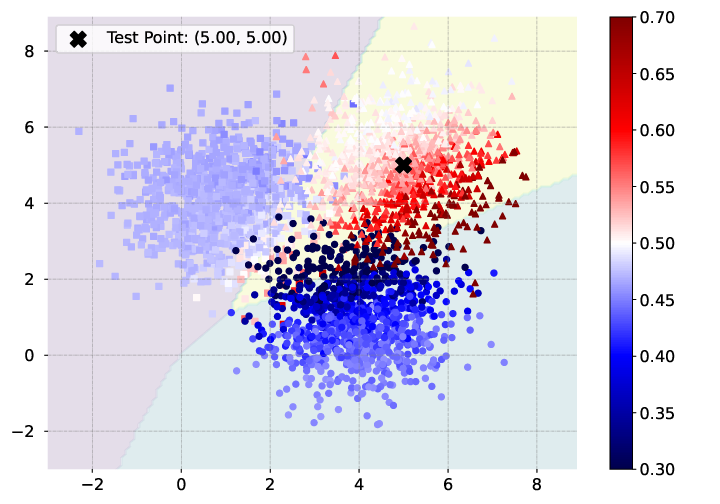}
        \includegraphics[width=0.32\textwidth]{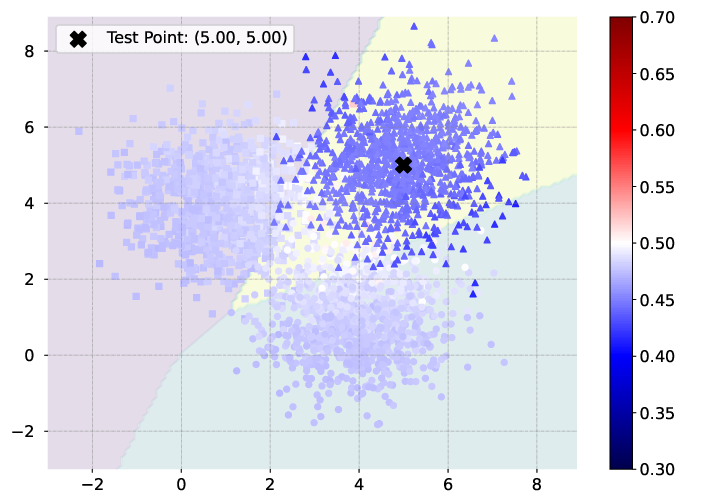}
        \caption{Plots showing kernel values for each training point relative to a test point. Because our kernel is replicating the output of a network, there are three kernel values per sample on a three class problem. This plot shows kernel values for all three classes across three different test points selected as the mean of the generating distribution. Figures on the diagonal show kernel values of the predicted class. Background shading is the neural network decision boundary.}
        \label{fig:points}
    \end{figure*}



\subsection{Extending To Image Data}
    We perform experiments on MNIST to demonstrate the applicability to image data. 
    This kernel representation was generated for convolutional ReLU Network with the categorical cross-entropy loss function, using Pytorch ~\citep{pytorch2019}. 
    The model was trained using forward Euler (gradient descent) using gradients generated as a sum over all training data for each step. 
    The state of the model was saved for every training step. In order to compute the per-training-point gradients needed for the kernel representation, the per-input Jacobians are computed at execution time in the representation by loading the model for each training step $i$, computing the jacobians for each training input to compute $\nabla_w f_{w_s(0)}(x_i)$, and then repeating this procedure for 200 $t$ values between 0 and 1 in order to approximate $\int_0^1 f_{w_s(t)}(x)$. For MNIST, the resulting prediction is very sensitive to the accuracy of this integral approximation, as shown in Fig.~\ref{fig:mnist}. The top plot shows approximation of the above integral with only one step, which corresponds to the DPK from previous work ~\citep{chen2021equivalence, domingos2020, incudini2022quantum} and as we can see, careful approximation of this integral is necessary to achieve an accurate match between the model and kernel. 

\begin{figure}[!h]
    \centering
    \includegraphics[width=0.55\linewidth]{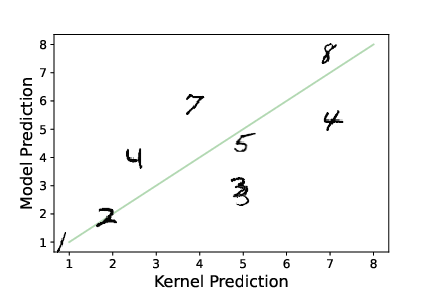}

    \vspace{-9mm}
    
    \includegraphics[width=0.55\linewidth]{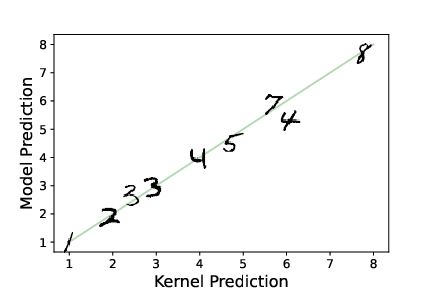}
    
    \vspace{-9mm}
    
    \includegraphics[width=0.55\linewidth]{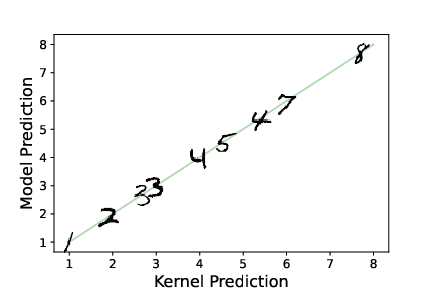}
    \caption{Experiment demonstrating the relationship between model predictions and kernel predictions for varying precision of the integrated path kernel. The top figure shows the integral estimated using only a single step. This is equivalent to the discrete path kernel (DPK) of previous work ~\citep{domingos2020every, chen2021equivalence}. The middle figure shows the kernel evaluated using 10 integral steps. The final figure shows the path kernel evaluated using 200 integral steps.}
    \label{fig:mnist}
\end{figure}
\section{Conclusion and Outlook} 
The implications of a practical and finite kernel representation for the study of neural networks are profound and yet importantly limited by the networks that they are built from. For most gradient trained models, there is a disconnect between the input space (e.g. images) and the parameter space of a network. Parameters are intrinsically difficult to interpret and much work has been spent building approximate mappings that convert model understanding back into the input space in order to interpret features, sample importance, and other details ~\citep{simonyan2013deep, lundberg2017unified, Selvaraju_2019}. The EPK is composed of a direct mapping from the input space into parameter space. This mapping allows for a much deeper understanding of gradient trained models because the internal state of the method has an exact representation mapped from the input space. As we have shown in Fig.~\ref{fig:points}, kernel values derived from gradient methods tell an odd story. We have observed a kernel that picks inputs near decision boundaries to emphasize and derives a spatial transform whose basis vectors depend neither uniformly nor continuously on training points. Although kernel values are linked to sample importance, we have shown that most contributions to the kernel's prediction for a given point are measuring an overall change in the network's internal representation. This supports the notion that most of what a network is doing is fitting a spatial transform based on a wide aggregation of data, and only doing a trivial calculation to the data once this spatial transform has been determined ~\citep{chizat2020maxmargin}. 
As stated in previous work by ~\citet{domingos2020}, this representation has strong implications about the structure of gradient trained models and how they can understand the problems that they solve. Since the kernel weights in this representation are fixed derivatives with respect to the loss function $L$, $a_{i, s} = -\varepsilon  L'(f_{w_s(0)}(x_i),  y_i)$, nearly all of the information used by the network is represented by the kernel mapping function and inner product. Inner products are not just measures of distance, they also measure angle. Figure \ref{fig:grad} shows that for a typical training example, the $L_2$ norm of the weights changes monotonically by only 20-30\% during training. This means that the "learning" of a gradient trained model is dominated by change in angle, which is predicted for kernel methods in high dimensions ~\citep{hardle2004nonparametric}. Also, model training gradients get more aligned with a the vector difference between the start and end point of training as the model trains. Also, that vector is very stable past a certain distance away from the initial training weights.

\begin{figure}[h]
\centering
\includegraphics[width=.85\textwidth]{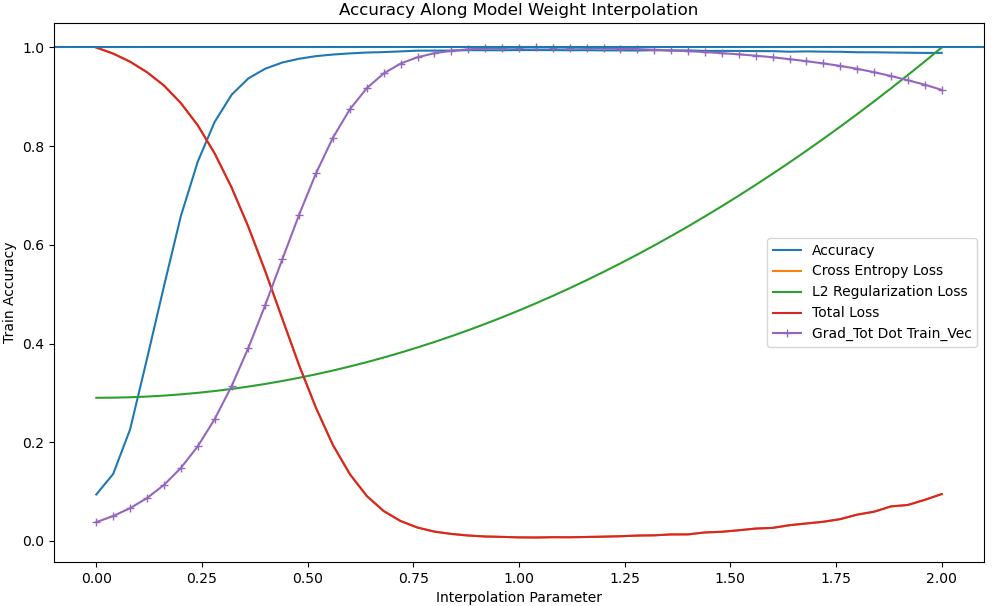}
\caption{This plot shows a linear interpolation $w(t) = w_0 + t(w_{1} - w_0)$ of model parameters $w$ for a convolutional neural network $f_w$ from their starting random state $w_0$ to their ending trained state $w_1$. The hatched purple line shows the dot product of the sum of the  gradient over the training data $X$, $\langle \nabla_w f_{w(t)}(X), (w_1 - w_0)/|w_1 - w_0| \rangle$. The other lines indicate accuracy (blue), L2 Regularization (green increasing), total loss (red decreasing), and cross-entropy loss (orange, covered exactly by total loss showing that cross-entropy dominates the total loss for this regularization weight). }
\label{fig:grad}
\end{figure}


For kernel methods, our result also represents a new direction. Despite their firm mathematical foundations, kernel methods have lost ground since the early 2000s because the features implicitly learned by deep neural networks yield better accuracy than any known hand-crafted kernels for complex high-dimensional problems ~\citep{NIPS2005_663772ea}. 
We are hopeful about the scalability of learned kernels based on recent results in scaling kernel methods ~\citep{snelson2005sparse}. 
Exact kernel equivalence could allow the use of neural networks to implicitly construct a kernel. 
This could allow kernel based classifiers to approach the performance of neural networks on complex data. 
Kernels built in this way may be used with Gaussian processes to allow meaningful direct uncertainty measurement. 
This would allow for much more significant analysis for out-of-distribution samples, including adversarial attacks ~\citep{szegedy2013intriguing, ilyas2019adversarial}. 
There is significant work to be done in improving the properties of the kernels learned by neural networks for these tools to be used in practice.
We are confident that this direct connection between practical neural networks and kernels is a strong first step towards achieving this goal.
\chapter{Exact Path Kernels Naturally Decompose Model Predictions}
\label{Chapter4a}

We can think about the previous paper establishing the exact path
kernel as an attempt to determine the necessary conditions for a
neural network to be expressed exactly by a kernel machine. That paper
also focused on using this kernel representation for uncertainty
quantification as an application. Although this is interesting from a
theory perspective, the applications are very limited and are a few
layers of abstraction away from providing any practical benefit to
modern machine-learning techniques. The purpose of the following paper
is to expand this theory and develop it towards applications that can
be directly useful to the field. In the process of generalizing and
applying the above method as a decomposition, it became obvious that
one modern machine-learning technique was implicitly relying on this
decomposition rather heavily: Out-of-Distribution (OOD)
Detection. The problem of identifying OOD data is orthogonal to the
adversarial problem. In some sense adversarial problems are difficult
because there do not exist practical metrics which can determine that
an adversarial example is not part of the natural data distribution
for a particular task. For OOD examples, data are generated by
distributions which are in some sense obviously distinct from the distribution
of training data for an ML model. It can still be time-consuming or
difficult using statistical techniques to identify this data, so there
is a natural desire to use trained ML models to determine whether data
could be samples from the distributions they were trained on or
not. In many of the applications that follow

This Chapter includes a paper recently submitted to ICLR 2023. The
contents include a generalization of the representation from
~\ref{Chapter4} and two applications of this representation: First to
Out-Of-Distribution (OOD) detection, and the second to measuring
signal manifold dimension. Distinct from the data manifold that exists in input space as the object along which all data are embedded, the signal manifold is the surface along with the data are embedded \emph{according to the model}. OOD detection and signal manifold dimension estimation begin to showcase the advantages of path kernel ensemble representations of neural
networks. This was joint work primarily performed by Brian Bell,
Michael Geyer, where most of the theoretical work and mathematics was
derived and written by Brian Bell and the experimental work and
numerical results were produced by Michael Geyer. This particular
paper includes an extensive literature review (performed by Brian Bell)
analyzing some methods that have recently come to occupy the
cutting-edge of OOD detection algorithms. In the context of this
dissertation, this paper includes two important contributions, one is
a cleaner general definition of the representation from the previous
paper. The other is the decomposition of predictions by taking
gradients of this representation with respect to various spaces. This
second contribution allows the application of this theory to a class
of recent work and demonstrates the ability of this theoretical
foundation to inform practical applications at the cutting edge. 




\section{Introduction}

Out-of-distribution (OOD) detection for machine learning models is a new, quickly growing field important to both reliability and robustness~\citep{hendrycks2019, biggio2014, hendrycks2017, desilva2023, yang2021, filos2020}.
Recent results have empirically shown that parameter gradients are highly informative for OOD detection~\citep{behpour2023, djurisic2023, huang2021}.
To our knowledge, this paper is the first to present theoretical justifications which explain the surprising effectiveness of parameter gradients for OOD detection.

In this paper, we unite empirical insights in cutting edge OOD with recent theoretical development in the representation of finite neural network models with tangent kernels~\citep{bell2023,chen2021equivalence,domingos2020}. 
Both of these bodies of work share approaches for decomposing model predictions in terms of parameter gradients. 
However, the Exact Path Kernel (EPK)~\citep{bell2023} provides not only rigorous theoretical foundation for the use of this method for OOD, but also naturally defines other decompositions which deepen and expand our understanding of model predictions. The application of this theory is directly connected to recent state of the art OOD detection methods.

In addition, this paper provides a connection between tangent kernel methods and dimension estimation. 
At the core of this technique is the ability to extract individual training point sensitivities on test predictions and use these to map the subspace on which parameter gradients can vary, the \emph{parameter tangent space}.
This paper demonstrates a generalization (the gEPK) of the EPK from
\citet{bell2023}. Given a training point $x_{\text{train}}$, the final parameter state $\theta$ of a model $f$, and a test point $x_{\text{test}}$, this
method can exactly measure the \emph{input gradient} $\nabla_{x_\text{train}}f(x_\text{test}; \theta_\text{trained})$.
It is shown that this quantity provides all necessary information for measuring the dimension of the \textit{signal manifold} \citep{srinivas2023} around a given test point.

\begin{figure}[t]
    \centering
    \includegraphics[width=0.9\textwidth]{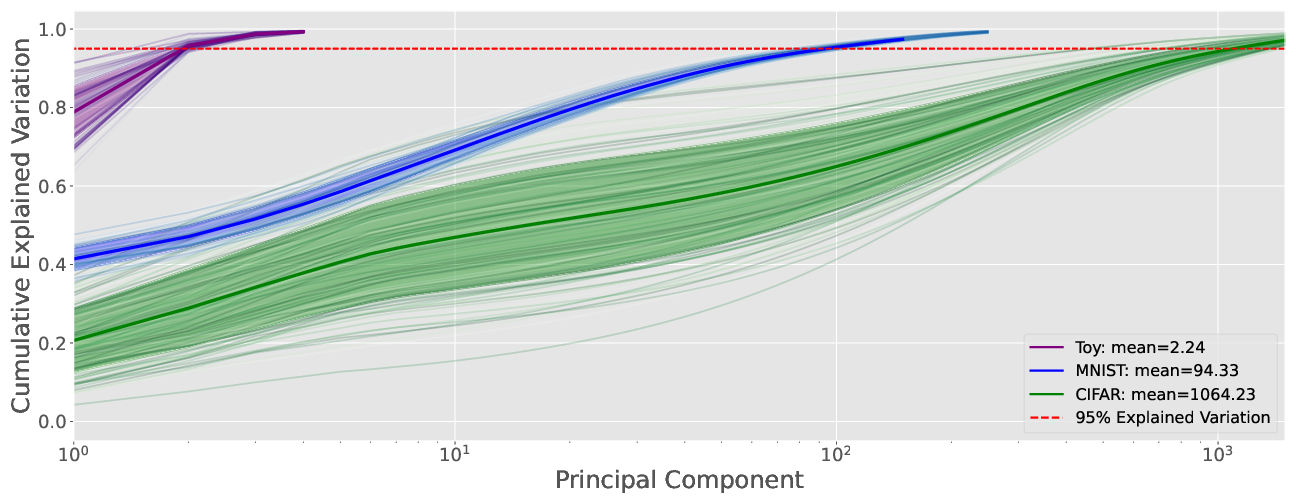}
    \caption{The gEPK naturally provides a measure of input dimension. This plot shows the cumulative sum of the explained variation of training point sensitivities $\nabla_{x_\text{train}}f(x_\text{test}; \theta_\text{trained})$. Different datasets are color coded to show differences in signal dimension.  }
    \label{fig:cdf}
\end{figure}
 

In short, this work leverages the gEPK to:
\begin{itemize}
    \item Generalize and explain the success of recent successful methods in OOD.
    \item Showcase OOD using natural gEPK-based decomposition of model predictions in terms of parameter gradients.
    \item Measure exact input variations and signal manifold dimension around arbitrary test points.

\end{itemize}
The primary contributions of this paper are theoretical in nature: establishing useful decompositions based on the exact representation theorem in Section~\ref{sec:gepk} and writing several leading OOD detection methods in terms of this representation. The preliminary experimental results also support practical tasks of out-of-distribution (OOD) detection and estimating signal manifold dimension. 

\label{sec:input}

\section{Related Work}
While there has been a significant amount of recent work studying the Neural Tangent Kernel (NTK)~\citep{jacot2018neural}, there is still relatively little work exploring its exact counterpart, the path kernels~\citep{bell2023, chen2021equivalence, domingos2020}. While these other works are focused on the precise equivalence between artificial neural networks and SVMs or Kernel machines, this equivalence requires significant restrictions placed on the loss function and model used for a task. This paper seeks to take advantage of this exact representation style without imposing such strict requirements. To the best of our knowledge, this is the first work exploring this loosened equivalence. 

There are several schools of thought about whether OOD data can be learned~\citep{huang2021scaling, mohseni2020, he2015, pillai2013classification, fumera2002}, which part of a model should be interrogated in order to identify OOD examples~\citep{liu2020, lin2021}, whether it is a purely statistical question~\citep{lee2018}, or whether it can simply be solved with more data~\citep{chen2021atom, de_silva_value_2023}. 
The best-performing recent approaches have all used relatively simple modifications of model activation or model gradients \citep{djurisic2023extremely, xu2023vra, sun2022, sun2021}. 
The first methods we explore relate to the use of model gradients to construct statistics which separate in-distribution (ID) examples from OOD examples. 
This is fundamentally a geometric approach which should be comparable
with the method proposed by \citet{sun2022deep} and other related work
by~\citet{gillette2022data}. 
The first prominent method of this type was proposed by~\citet{liang2018}. 
ODIN is still a notable method in this space, and has been followed by many more gradient-based approaches~\citep{behpour2023, huang2021gradients} and has caused some confusion about why these methods work so well~\citep{igoe2022}.

Much recent work has been devoted to measurement of dimension for the subspace in which the input data distribution live for machine-learning tasks.
We will partition this work into works trying to understand this intrinsic data dimension in model agnostic ways~\citep{gillette2022data, yousefzadeh2021deep, kaufman_data_2023, gilmer2018, gong2019, glielmo2022, facco2018, Levina_Bickel_2004} and works trying to understand or extract model's understanding of this subspace~\citep{dominguez-olmedo_data_2023, Ansuini_Laio_Macke_Zoccolan_2019, talwalker2008, Costa_Hero_2004b, giryes2014, Zheng_He_Qiu_Wipf_2022}. 
This paper proposes a new method which bears more similarity to the latter. 
We believe that this approach is more relevant for studying ANNs since they discover their own metric spaces. 
Understanding signal manifolds is both useful in practice for more efficient low rank models~\citep{yang2020, swaminathan2020}, and also for uncertainty quantification and robustness~\citep{Costa_Hero_2004a, wang2021, khoury2018, srinivas2023, song2018pixeldefend, snoek2019}. 

\section{Theoretical Justification : Exact Path Kernel Decomposition}
\label{sec:gepk}
The theoretical foundation of this starts with a modified general form of an recent exact path kernel representation result from~\citet{bell2023}.
We will reuse the structure of the Exact Path Kernel (EPK) without relying on the reduction to a single kernel across training steps.
In order to increase generality, we will not assume the inner products may be reduced across steps, resulting in a representation which is no longer strictly a kernel.
This representation, however, will allow exact and careful decomposition of model predictions according to both input gradients and parameter gradients without the strict requirements of the EPK.
The function, $\varphi_{s,t}(x)$, in the EPK sum defines a linear subspace, the properties of which we will study in detail.
The primary difference between the representation we propose and the original EPK is the EPK maintained symmetry at the cost of continuity, on the other hand the gEPK does not introduce a discontinuity.

\begin{restatable}[Generalized Exact Path Kernel (gEPK)]{theorem}{ekr}
\label{thm:ekr}
Suppose $f(\cdot; \theta): \mathbb{R}^d \rightarrow \mathbb{R}^k$ is a differentiable parametric scalar valued model with parameters $\theta_s \in \mathbb{R}^M$ and $L$ is a loss function. Furthermore, suppose that $f$ has been trained by a series $\{s\}_{s=0}^S$ of discrete steps composed from a sum of loss gradients for the training set $ \sum_{i=1}^N \varepsilon \nabla_\theta L(f(x_i; \theta), y_i)$ on $N$ training data $X_T$ starting from $\theta_0$, with learning rate $\varepsilon$; as is the case with traditional gradient descent. Let $t \in [0, 1]$ be an interpolation variable which parameterizes the line connecting any $\theta_s$ to $\theta_{s+1}$ so that $\theta_s(t) = \theta_s + t(\theta_{s+1} - \theta_s)$. Then for an arbitrary test point $x$, the trained model prediction $f(x; \theta_S)$ can be written:
\begin{equation}
f(x; \theta_S) = f(x; \theta_0) + \sum_{i=1}^N \sum_{s=0}^S \varepsilon \left(\int_0^1 \varphi_{s,t}(x) dt\right) L'(f(x_i; \theta_s), y_i) \left(\varphi_{s, 0}(x_i)\right)
\label{exact}
\end{equation}
\begin{align}
    L'(a, b) &= \dfrac{\partial L(a, b)}{\partial a}\\
    \varphi_{s,t}(x) &\equiv \nabla_\theta f(x; \theta_s(t)), \\
    \theta_s(t) &\equiv \theta_s(0) + t(\theta_{s+1}(0)-\theta_s(0)), \text{ and}\\
    \hat y_{\theta_s(0)} &\equiv f(x; \theta_s(0)).
\end{align}
\end{restatable}
\vspace{-0.25cm}
\begin{proof}
Guided by the proof for Theorem 6 from~\citet{bell2023}, let $\theta$ and $f(\cdot; \theta)$ satisfy the conditions of Theorem~\ref{thm:ekr}, and $x$ be an arbitrary test point. We will measure the change in prediction during one training step from $\hat y_s = f(x; \theta_s)$ to $\hat y_{s+1} = f(x; \theta_{s+1})$ according to its differential along the interpolation from $\theta_s$ to $\theta_{s+1}$. Since we are training using gradient descent, we can write $\theta_{s+1} \equiv \theta_s + \dfrac{d \theta_s(t)}{dt} $. We derive a linear interpolate connecting these states using $t \in [0, 1]$:
\begin{align}
    \dfrac{d \theta_s(t)}{dt} &= \theta_{s+1} - \theta_s\\   
    \int \dfrac{d \theta_s(t)}{dt} dt &= \int (\theta_{s+1} - \theta_s)dt\\
    \theta_s(t) &= \theta_s + t(\theta_{s+1} - \theta_s)
\end{align}
One of the core insights of this definition is the distinction between \textit{training steps} (defined by $s$) and the \textit{path between training steps} (defined by $t$).
By separating these two terms allows a \textit{continuous} integration of the \textit{discrete} behavior of practical neural networks.
Since $f$ is being trained using a sum of gradients weighted by learning rate $\varepsilon$, we can write:
\begin{align}
    \dfrac{d \theta_s(t)}{dt} &= -\varepsilon  \nabla_\theta  L(f(X_T; \theta_s(0)), y_i) 
\end{align}
Applying chain rule and the above substitution, we can write the change in the prediction as 
\begin{align}
    \dfrac{d \hat y}{dt} = \dfrac{d f(x; \theta_s(t))}{dt} &= \sum_{j = 1}^{M} \dfrac{\partial f}{d \theta^j} \dfrac{\partial \theta^j}{dt} = \sum_{j = 1}^{M} \dfrac{d f(x; \theta_s(t))}{\partial \theta^j} \left(-\varepsilon  \dfrac{\partial L(f(X_T; \theta_s(0)),  Y_T)}{\partial \theta^j}\right)\\
&= \sum_{j = 1}^{M} \dfrac{\partial f(x; \theta_s(t))}{\partial \theta^j} \left(- \sum_{i = 1}^{N}\varepsilon L'(f(x_i; \theta_s(0)), y_i) \dfrac{\partial  f(x_i; \theta_s(0))}{\partial \theta^j}\right)\\
&= - \varepsilon \sum_{i = 1}^{N}  \nabla_\theta f(x; \theta_s(t)) \cdot L'(f(x_i; \theta_s(0)), y_i)  \nabla_\theta f(x_i; \theta_s(0))
\end{align}
Using the fundamental theorem of calculus, we can compute the change in the model's output over step $s$ by integrating across $t$.
\begin{align}
    y_{s+1} - y_s &= \int_0^1 - \varepsilon \sum_{i = 1}^{N}   \nabla_\theta f(x; \theta_s(t)) \cdot L'(f(x_i; \theta_s(0)), y_i) \nabla_\theta f(x_i; \theta_s(0))dt\\
 &=  - \sum_{i = 1}^{N} \varepsilon\left(\int_0^1\nabla_\theta f(x; \theta_s(t))dt\right) \cdot L'(f(x_i; \theta_s(0)), y_i)   \nabla_\theta f(x_i; \theta_s(0))
\end{align}
For all $N$ training steps, we have
\begin{align}
y_N &= f(x; \theta_0) + \sum_{s=0}^N y_{s+1} - y_s \\
&= f(x; \theta_0) - \sum_{s=0}^N \sum_{i = 1}^{N} \varepsilon \left(\int_0^1\nabla_\theta f(x; \theta_s(t))dt\right)  \cdot L'(f(x_i; \theta_s(0)), y_i)   \nabla_\theta f(x_i; \theta_s(0)) \label{eq:prediction}
\end{align}
\end{proof}
    
\textbf{Remark 1:} While this theorem is not our main contribution, we provide it along with its brief proof to provide a thorough and useful theoretical foundation for the main results which follow. \\
\textbf{Remark 2:} Many of the remarks from~\citet{bell2023} remain including that this representation holds true for any contiguous subset of a gradient based model, e.g. when applied to only the middle layers of an ANN or only to the final layer. 
This is since each contiguous subset of an ANN can be treated as an ANN in its own right with the activations of the preceding layer as its inputs and its activations as its outputs. 
In this case, the training data consisting of previous layer activations may vary as the model evolves.
One difference in this representation is that we do not introduce a discontinuity into the input space.
This sacrifices symmetry, which disqualifies the resulting formula as a kernel, but retains many of the useful properties needed for OOD and dimension estimation.\\
\textbf{Remark 3:} \eqref{eq:prediction} allows decomposition of predictions into an initial (random) prediction $f(x; \theta_0)$ and a \emph{learned adjustment} which separates the contribution of every training step $s$ and training datum $i$ to the prediction. 

\section{OOD is enabled by Parameter Gradients}
\label{sec:ood}
\begin{figure}[h]
\begin{center}
\includegraphics[width=0.45\textwidth]{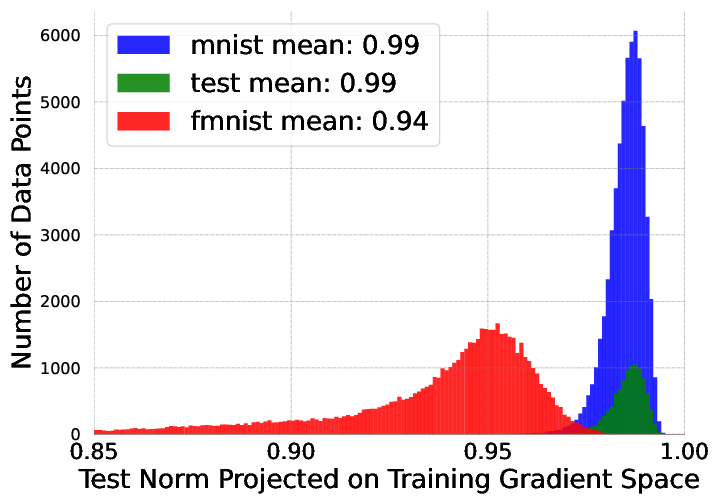}
\includegraphics[width=0.45\textwidth]{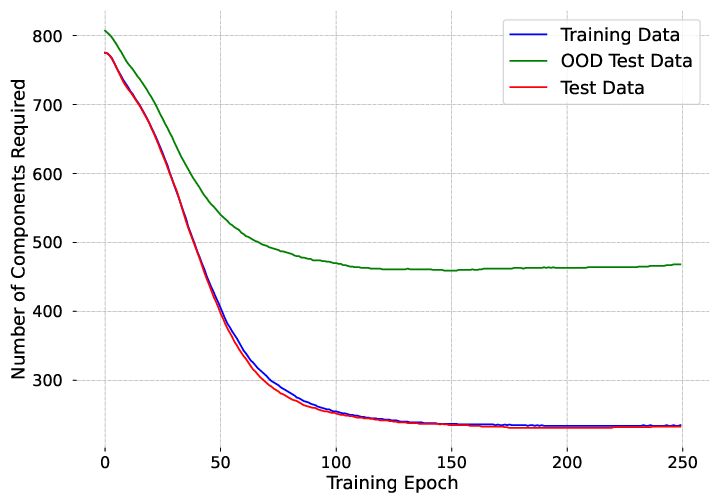}
\end{center}
\caption{OOD detection using difference in training vs. test gradients. Right plot shows the number of components required to explain 95\% variation in weight space across training for a toy problem (three Gaussian distributions embedded in 100 dimensions). Left histogram shows norms of vectors projected onto the gradient weight space defined by the gEPK on MNIST and FMNIST.}
\label{fig:ood}
\end{figure}
One natural application of the gEPK is the separation of predictions into vectors corresponding with the test gradient $\varphi_{s,t}(x)$ for a given test point $x$ and each training vector  weighted by its loss gradient $L'(\hat y_i, y_i) \varphi_{s, 0}(x_i)$. While the test vector depends on the choice of test point $x$, the subspace of training gradient vectors is fixed. By the linear nature of this inner product, it is clear that no variation in test data which is orthogonal to the training vector space can be reflected in a model's prediction. We can state this as a theorem:
\begin{restatable}[Prediction Spanning Vectors]{theorem}{ekr}
\label{thm:oodspan}
\begin{align}
    B = \left\{\varphi_{s, 0}(x_i); i \in \{1,...,N\}, s \in \{1,...,S\}\right\}
\end{align}
spans the subspace of test parameter gradients with non-zero learned adjustments. 
\end{restatable}
\begin{proof}
Suppose for every $s$ and $t$, $\varphi_{s, t}(x) \notin B$. Then for every $i$, $s$, and $t$, $\langle \varphi_{s,t}(x), \varphi_{s, 0}(x_i)\rangle = 0$. Rewriting ~\eqref{eq:prediction} we have:
\begin{align}
   y_N &= f(x; \theta_0) - \sum_{s=0}^N \sum_{i = 1}^{N} \varepsilon \int_0^1 L'(f(x_i; \theta_s(0)), y_i) \langle \varphi_{s, t}(x), \varphi_{s,0}(x_i) \rangle dt.
\end{align}
Note that for this derivation $f$ and $L$ are scalar valued, allowing commutativity of multiplication that would require more care in the multi-class case. 
We can immediately see that every term in the learned adjustment summation will be equal to zero. 
\end{proof}
We will demonstrate that most cutting-edge OOD methods implicitly analyze the spectra of parts of this subspace in order to discriminate in practice. 

\subsection{Expressing Prior OOD Methods with the gEPK}

 We will now establish that most gradient based methods for OOD and some methods which do not explicitly rely on gradients can be written as projections onto subsets of this span. 

\textbf{GradNorm}
The first well-known method to apply gradient information for OOD is  ODIN: Out-of-DIstribution detector for Neural Networks  \citet{liang2018}. This method, inspired by adversarial attacks, perturbs inputs by applying perturbations calculated from input gradients. The method then relies on the difference in these perturbations for in-distribution versus out-of-distribution examples to separate these in practice. This method directly inspired~\citet{huang2021} to create GradNorm. This method which occupied the cutting edge in 2021 computes the gradient of Kullback–Leibler divergence with respect to model parameters so that we have a matrix of size number of classes by number of parameters.:
\begin{align}
    \dfrac{1}{C} \sum_i^C \dfrac{\partial L_{CE}(f(x; \theta), i)}{\partial\hat y}\nabla_\theta f(x; \theta)
\end{align}
This looks like the left side of the inner product from the gEPK, however the scaling factor, $\dfrac{\partial L_{CE}(f(x; \theta), i)}{\partial \hat y}$, does not match. In fact, this approach is averaging across the parameter gradients of this test point with respect to each of its class outputs, which we can see is only a related subset of the full basis used by the model for predictions. This explains improvements made in later methods that are using a more full basis. Another similar method, ExGrad~\citep{igoe2022}, has been proposed which experiments with different similar decompositions and raises some questions about what is special about gradients in OOD -- we hope our result sheds some light on these questions. Another comparable method proposed by~\citet{sun2022deep} may also be equivalent through the connection we establish below in Section~\ref{sec:input} between this decomposition and input gradients which may relate with mapping data manifolds in the Voronoi/Delaunay~\citep{gillette2022data} sense. 

\textbf{ReAct, DICE, ASH, and VRA}
Along with other recent work~\citep{sun2021, sun2022, xu2023vra}, some of the cutting edge for OOD as of early 2023 involves activation truncation techniques like that neatly described by~\citet{djurisic2023extremely}. 
Given a model, $f(x; \theta) = f^{\text{classify}}(\cdot; \theta_{\text{classify}}) \circ f^{\text{represent}}(\cdot; \theta_{\text{represent}}) \circ f^{\text{extract}}(\cdot; \theta_{\text{extract}})$, and an input, $x$, a prediction, $f(x; \theta)$, is computed forward through the network. 
This yields a vector of activations (intermediate model layer outputs), $A(x; \theta_{\text{represent}})$, in the representation layer of the network. 
This representation is then pruned down to the $p^{\text{th}}$
percentile by setting any activations below that percentile to
zero. ~\citet{djurisic2023extremely} mention that their algorithm,
Activation SHaping (ASH), does not depend on statistics from the
training data. However by chain rule, high activations will correspond with high parameter gradients. 
Meaning this truncation is picking a representation for which
\begin{align}
  \left\langle \nabla_\theta f(x; \theta_{\text{represent}}),
  \dfrac{\partial L(\hat y(x_i), y_i)}{\partial \hat y} \nabla_\theta f(x_i;
  \theta_{\text{represent}}) \right\rangle
\end{align}
is high for many training points, $x_i$. We note that the
general kernel representation defined in Section~\ref{sec:gepk} can be
computed for any subset of a composition. Truncation is effectively a
projection onto the gradients on parameters with
the highest variation in the representation layers of the network. 
This may explain some part of the performance advantage of these methods. 

\begin{figure}[t]
    \centering
    \includegraphics[width=0.45\textwidth]{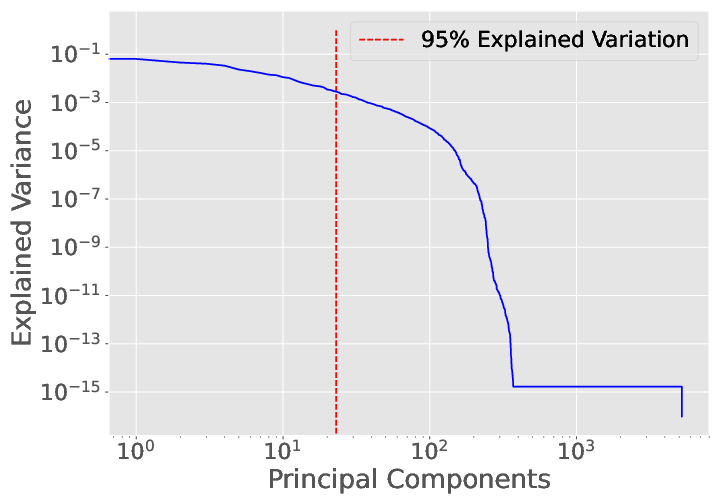}
    \includegraphics[width=0.45\textwidth]{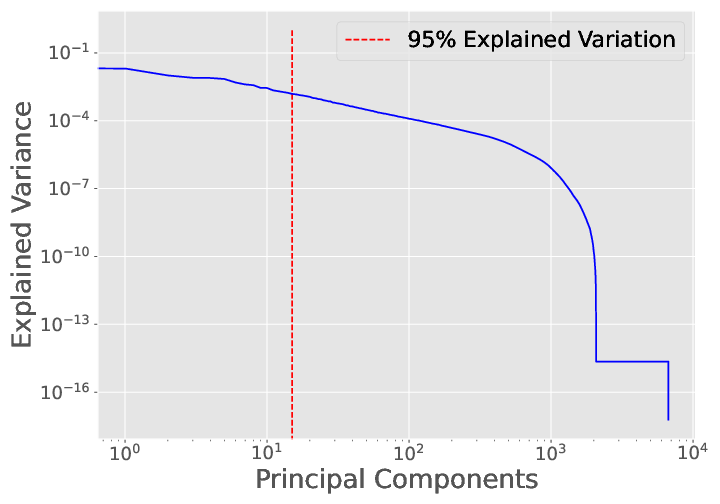}
    \caption{Explained Variance Ratio of parameter gradients. Left: MNIST, Right: CIFAR. 95\% of variation can be explained with a relatively low number of components in both cases.}
    \label{fig:rank}
\end{figure}

\textbf{GradOrth}~\citet{behpour2023} explicitly create a reference basis from parameter gradients on training data for comparison. They do this for only the last layer of a network with mean squared error (MSE) loss, allowing a nicely abbreviated expression for the gradient:
\begin{align}
    \nabla_\theta L(f(x; \theta), y) &= (\theta x - y)x^T = \Omega x^T.
\end{align}
Treating $\Omega$ as an error vector, they prove that all variation of the output must be within the span of the $x^T$ over the training set. They then pick a small subset of the training data and record its activations $R^L_{ID} = [x_1, x_2, ..., x_n]$ over which they compute the SVD, $U^L_{ID} \Sigma^L_{ID} (V^L_{ID})^T = R^L_{ID}$. This representation is then truncated to $k$ principal components according to a threshold $\epsilon_{\text{th}}$ such that 
\begin{align}
\left\| U^L_{ID, k} \Sigma^L_{ID, k} (V^L_{ID})^T\right\|^2_F &\geq \epsilon_\text{th} \|R^L_ID\|^2_F.
\end{align}
This basis $S^L = (U^L_{ID, k})_k$ is now treated as the reference
space onto which test points' final layer gradients can be
projected. They define a numerical score to distinguish OOD examples
as follows: 
\begin{align}
    O(x) = (\nabla_{\theta} L(f(x; \theta_L), y))S^L(S^L)^T
\end{align}
We note that this formulation requires a label $y$ for each of the data being tested for inclusion in the data distribution. 
Despite this drawback, the performance presented by \citet{behpour2023} is impressive.

\subsection{gEPK for OOD}
Theorem \ref{thm:oodspan} provides a more general spanning result immediately. In fact, as we have illustrated in Figure ~\ref{fig:ood}, we can pick a much reduced basis \emph{based only on the final training step} which will span most of the variation in models' learned adjustments.
Theorem ~\ref{thm:oodspan} and the definition of SVD provide the following:
\begin{corollary}
    Let $A$ be a matrix stacking the elements of $B$ as rows. Then let $U \Sigma V^T = A$ as in SVD. Then $\text{Span}(B) = \text{Span}(\text{Rows}(V))$. 
\end{corollary} 
In the case that  the number of training data exceed the number of parameters of a model, the same result holds true for a basis computed only for gradients with respect to the final parameter states $\theta_S$. We will use a truncation, $V'$ of this final training gradient basis which we examine in Fig.~\ref{fig:rank}. This truncation still explains most variation in all layers due to the convergence of training gradients to a smaller subspace as shown in Fig.~\ref{fig:ood}. In future it may be possible to argue statistical expectations about the performance of a sketching approach to producing an equally performant basis without expensive SVD. 




We can see that most, if not all, of the above OOD methods can be
represented by some set of averaging or truncation assumptions on the
basis $V$. These should be mostly caught by the truncated basis
$V$. We test the usefulness of $V'$ to perform OOD detection by
projection onto its span using a sum over the class outputs weighted
by the loss gradients $L'(f(x_i; \theta_S), y_i)$ in
Fig.~\ref{fig:ood}. As the purpose of this paper is not to develop state of the art OOD detection methods, a comparison with recent benchmarks is not provided. Instead, a proof of concept that the gEPK can perform OOD detection is given. 
We note that this scalling has only been extracted
from the final training step, however this assumption is supported by
the convergence of this scaling over training. The gEPK helps explain the high performance of gradient based methods due to the implicit inclusion of the training parameter space in model predictions. This serves to illuminate the otherwise confusing discrepancy raised by~\citet{igoe2022}. 

\begin{figure}[t]
\begin{center}
\begin{tikzpicture}
\node [anchor=north, scale=1.0] (note) at (3.48,9.2) {Testing Point};
\node [anchor=west, rotate=90, scale=1.0] (note) at (0.20,1.8) {Training Point};
\node [anchor=north, scale=0.85] (water) at (2.4,8.7) { Model A};
\node [anchor=north, scale=0.85] (water) at (4.45,8.7) { Model B};
\node [anchor=south] (silly) at (0, -0.5) {};
\node [anchor=south] (silly) at (6, -0.5) {};
\begin{scope}[xshift=0.5cm]
    \node[anchor=south west,inner sep=0] (image) at (0,0) {\includegraphics[width=0.30\textwidth]{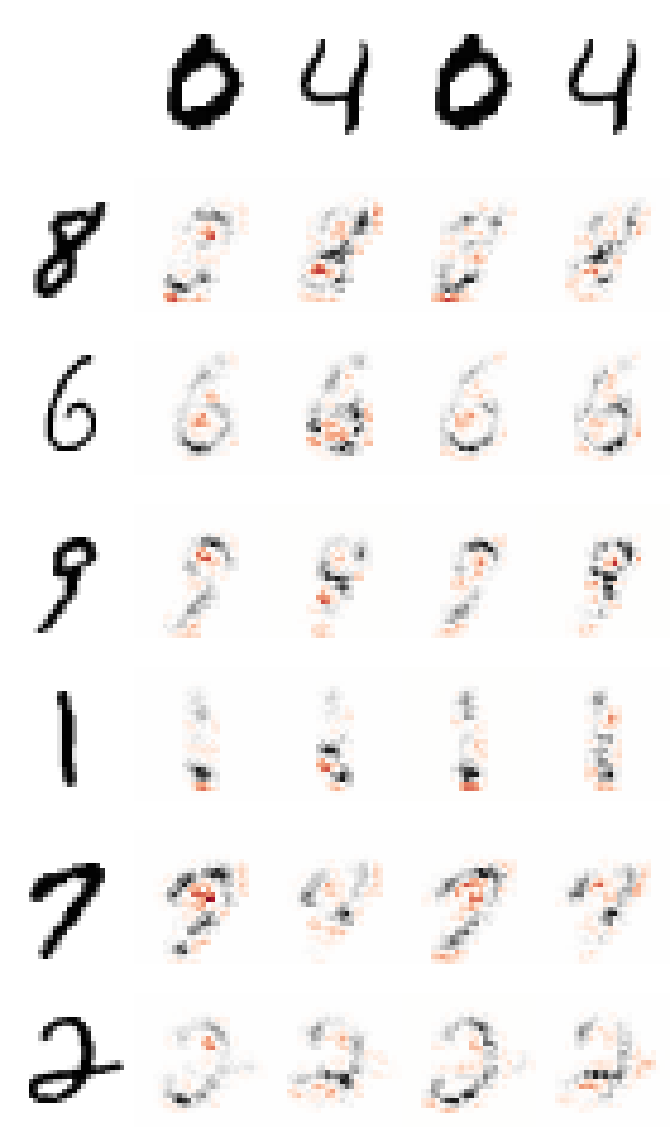}};
    \begin{scope}[x={(image.south east)},y={(image.north west)}]
    \end{scope}
\end{scope}
\end{tikzpicture}
\begin{subfigure}{0.46\textwidth}
\includegraphics[width=\textwidth]{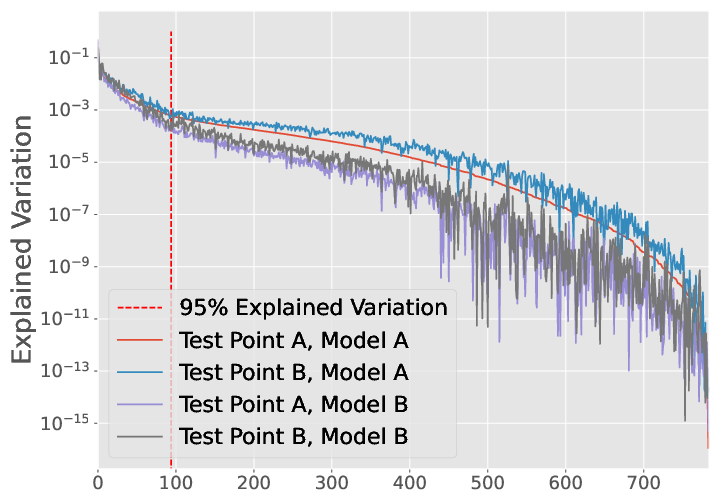}
\includegraphics[width=\textwidth]{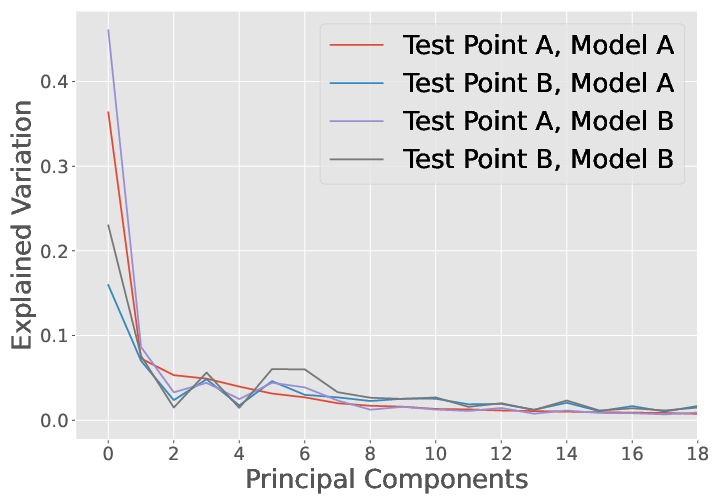}
\end{subfigure}
\end{center}
\caption{Left: Visualization of training point input gradients on test points compared between two models. Positive contribution (black) and negative contribution (red) of each training datum to the prediction for each test point. Elements in the grid are $\dfrac{\partial f(x; \theta_{\text{trained}})}{\partial x_{\text{i}}}$ for a fixed test point $x$ and several distinct training points $x_i$. Right: explained variation of other principal components by a selected principal component across 2 models and 2 test points. Top is log scale of the full spectrum, bottom shows the first 10 components.}
\label{fig:trans}
\end{figure}

In addition, we can see that comparison of test versus training loss gradients is unnecessary, which allows testing on data without ground truth labels (an issue with many recent gradient based OOD techniques). 
For most applications, the SVD of the parameter gradients over all of
the training steps and batches can be pre-computed and compared with
test points as needed. We can see from this body of work, many simplifying assumptions can be made which will preserve the essential bases needed for performance, but still drastically reduce computational cost. It is not necessarily sufficient to pick a basis that spans a target subspace and then truncate based on its variations. The variations must be accurately measured with correct scaling in terms of their contribution to the learned adjustments of a model. 

\section{Signal Manifold Dimension Estimated with Training Input Gradients}

In order to understand the subspace on which a model is sensitive to variation, we may take gradients decomposed into each of the training data. Take, for example, a model, $f(x; \theta)$, which satisfies the necessary conditions for expression as:
\begin{align}
    f(x; \theta_\text{trained}) &= f(x; \theta_0(0)) + \sum_i \sum_s \int_0^1 \varphi_{s,t}(x) \cdot (L'(f(x_i, \theta_s(0)), y_i) \varphi_{s, 0}(x_i)) dt\\
    \varphi_{s,t}(x) &= \nabla_\theta f(x; \theta_s(t))
\end{align}
And $\theta_s(t)$ are the parameters of $f$ for training step $s$ and time $t$ so that $\sum_s \int_0^1 \theta_s(t) dt$ integrates the entire training path taken by the model during training. Given a test point $x$, we can evaluate its subspace by taking, for each $x_j$:
\begin{align}
    \dfrac{\partial f(x; \theta_\text{trained})}{\partial x_j} &= \dfrac{\partial f(x; \theta_0(0))}{\partial x_j} + \sum_i \sum_s \int_0^1 \dfrac{\partial\left(\varphi_{s,t}(x) L'(f(x_i; \theta_s(0)), y_i) \varphi_{s, 0}(x_i)\right)}{\partial x_j} dt\\
                                                               &= \sum_i \sum_s \int_0^1 \varphi_{s,t}(x)dt (\dfrac{\partial L'(f(x, \theta_s(0)), y_i)}{\partial x_j} \varphi_{s, 0}(x_i) \\
                                                               &+ L'(f(x_i; \theta_s(0)), y_i) \dfrac{d\varphi_{s, 0}(x_i)}{dx_j}) 
\end{align}
We can see that these gradients will be zero except when $i = j$, thus we may summarize these gradients as a matrix, $G$, with 
\begin{align}
    G_j = \sum_s \int_0^1 \varphi_{s,t}(x)dt \left(\dfrac{\partial L'(f(x, \theta_s(0)), y_i)}{\partial x_j} \varphi_{s, 0}(x_i) + L'(f(x_i; \theta_s(0)), y_i) \dfrac{d\varphi_{s, 0}(x_i)}{dx_j}\right)
    \label{eq:input_decomp}
\end{align}
While written in this form, it appears we must keep second-order derivatives, however we note that the inner product with $\phi_{s,t}(x)$ eliminates a dimension, so that clever implementation still only requires storage of vectors since the dot-products can be iteratively computed for each training datum requiring storage of only a scalar quantity (vector in the multi-class case). This means the overall data required is a vector (a low rank matrix in the multi-class case). 

The rank of $G$ represents the dimension of the subspace on which the model perceives a test point, $x$, to live, and we can get more detailed information about the variation explained by the span of this matrix by taking its SVD. We can exactly measure the variation explained by each orthogonal component of the $\text{span}(G)$ with respect to the given test point $x$. $G(x)$ can be defined as a map from $x$ to the subspace perceived by the model around $x$. Any local variations in the input space which do not lie on the subspace spanned by $G(x)$ can not be perceived by the model, and will have no effect on the models output.

On MNIST, $G(x)$ creates a matrix which is of size $60000 \times 784 \times 10$ (training points $\times$ input dimension $\times$ class count).
This matrix represents the exact measure of each training points contribution towards a given test prediction.
In order to simplify computation, we reduce this term to $60000 \times 784$ by summing across the class dimension.
This reduction is justified by the same theory as the pseudo-NTK presented by~\citet{pmlr-v202-mohamadi23a}.
Of note is that in practice this matrix is full rank on the input space as seen in Figure \ref{fig:trans}. Note that this decomposition selects similar, but not identical, modes of variation across test points and even across different models. Components in SVD plots are sorted using Test Point A on Model A. By taking these individual gradient contributions for a test point and computing the SVD across the training, the significant modes of variation in the input space can be measured (sigma squared). 
This is despite MNIST having significantly less degrees of variation than its total input size (many pixels in input space are always 0).
Figure \ref{fig:cdf} demonstrates decomposing the input space in this way and provides a view of the signal dimension around individual test points. For a toy problem (3 Gaussian distributions embedded in 100 dimensional space) the model only observes between 2 and 3 unique variations which contribute to 95\% of the information required for prediction. Meanwhile the dimension of the signal manifold observed by the model around MNIST and CIFAR test points is approximately 94 (12\% of the data dimension 784) and 1064 (34\% of the data dimension 3096) respectively. It is likely that different training techniques will provide significantly different signal manifolds and consequently different numbers of components.
\begin{figure}[t]
    \centering
    \includegraphics[width=.45\textwidth]{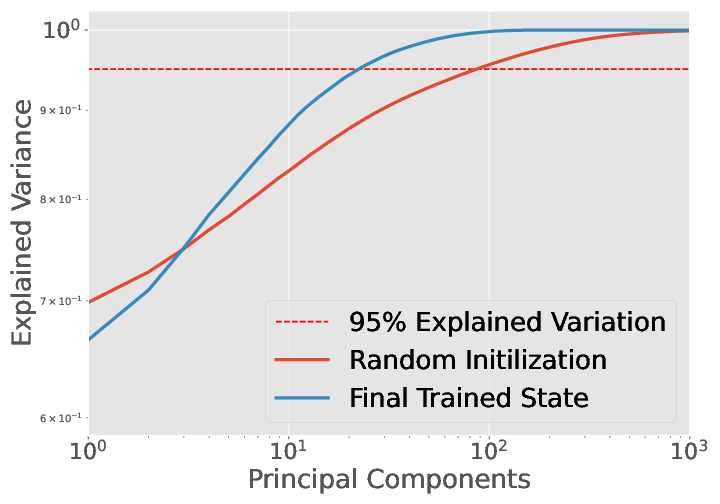}
    \includegraphics[width=.45\textwidth]{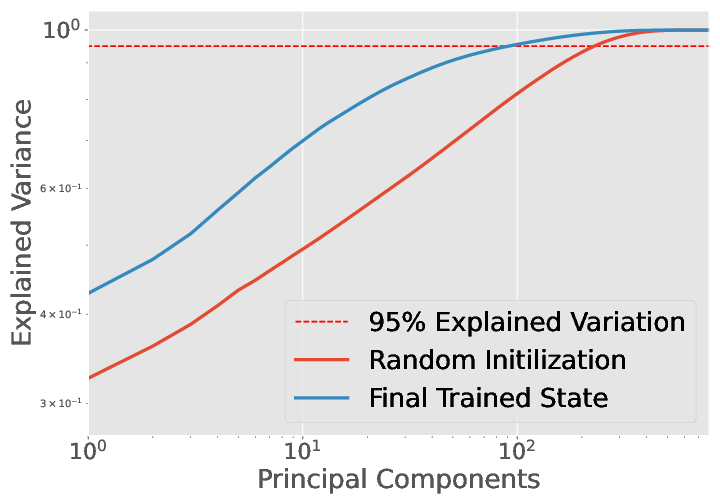}
    \caption{Differences between observing gradients in input space vs. weight space. Left: cumulative sum of explained variation parameter space. Right: cumulative sum of explained variation input space. Red solid line indicates a model at random initialization while the blue solid line represents the fully trained state.}
    \label{fig:compare}
\end{figure}
We can also ask for a decomposition in terms of training input gradients for predictions made on training points. This can be written in the following way:
\begin{align}
    \dfrac{df(x_j; \theta_\text{trained})}{dx_j} &= \dfrac{df(x_j; \theta_0(0))}{dx_j} + \sum_i \sum_s \int_0^1 \dfrac{\partial\left(\varphi_{s,t}(x_j) L'(f(x_i; \theta_s(0)), y_i) \varphi_{s, 0}(x_i)\right)}{\partial x_j} dt
\end{align}
The left hand side is computable without path-decomposition and so can be computed for each training datum to create a gradient matrix, $H_{\theta_\text{trained}}$. Another term, $\dfrac{df(x_j; \theta_0(0))}{dx_j}$ is also easily computable, yielding another matrix $H_{\theta_0}$. By comparing the rank and span of $H_{\theta_\text{trained}}$ and $H_{\theta_0}$ we can understand to what extent the model's spatial representation of the data is due to the initial parameter selection and how much is due to the training path. Also, $H_{\theta_\text{trained}}$ provides sample of gradients across all training data, which in some sense must be spanned by the model's implicit subspace basis. Despite missing the granular subspace information, the rank of this gradient matrix and its explained variation computed using SVD should be related to the model's implicit subspace rank. 
It should be noted that while there is a direct relationship between a models variations in input space and weight space, Figure \ref{fig:compare} shows that this mapping changes greatly from the beginning to end of training. From random initialization, the number of principal components required to achieve 95\% explained variation decreases in both cases shown. Note that at random initialization, the weight space gradients already have only a few directions accounting for significant variation. Disentangling the data dimension using weight space gradients is less effective than doing so in input space \citep{shamir2021dimpled}. Overall this spectrum starts out wide (high dimensional) for $\theta_0$ and much more focused (low dimensional) for $\theta_T$.

One interesting property of using input gradients for training data decomposed according to ~\eqref{eq:input_decomp} is the ability to compare input gradients across models with different initial parameters and even different architectures.
Figure \ref{fig:trans} demonstrates that two models with different random initializations which have been trained on the same dataset have a signal manifold which shares many components.
This is a known result that has been explored in deep learning through properties of adversarial transferability~\citet{szegedy2013intriguing}.
This demonstrates that the gEPK is capable of measuring the degree to which two models rely on the same features directly.
This discovery may lead to the construction of models which are provably robust against transfer attacks.


\section{Conclusion}

This paper presented decompositions based on a general exact path kernel representation for neural networks with a natural decomposition that connects existing out-of-distribution detection methods to a theoretical framework. 
This same representation reveals additional connections to dimension estimation and adversarial transferability. 
These connections are demonstrated with experimental results on computer vision datasets. 
The key insights provided by this decomposition are that model predictions implicitly depend on the parameter tangent space on its training data and that this dependence enables decomposition relative to a single test point by either parameter gradients, or training input gradients. 
This allows users to connect how neural networks learn at training time with how each training point influences the final decisions of a network.
We have demonstrated that the techniques used in practice for OOD are using a subset of the theoretical basis we propose.
Taking into account the entire training path will allow more rigorous methods for OOD detection.
There are many possible directions to continuing work in this area. 
These include understanding of how models depend on implicit prior distributions following (e.g.~\citet{nagler2023}), supporting more robust statistical learning under distribution shifts (e.g. \citet{Simchowitz2023}), and supporting more robust learning.
\chapter{Conclusions and Outlook}

\label{Chapter5} 

Throughout this work, I have sought to understand the intrinsic
structure of machine-learning models and how this gives
rise to adversarial attacks. I began by studying the construction of
neural networks in Chapter~\ref{Chapter1} and learning in practice by generating adversarial
attacks in Chapter~\ref{Chapter2}. This led naturally to a more careful study of decision
boundaries and the development of the Persistence metric in Chapter~\ref{Chapter3}. Uncanny drops in persistence while crossing
decision boundaries toward adversarial attacks indicated that attacks
may exist in highly curved regions. The field conspicuously lacks
tools for evaluating curvature, although some progress is being
made. In the process of searching related work for methods that would allow more
direct analysis I discovered deficiencies in nascent literature on
path kernels starting with the work of ~\citet{domingos2020}. The
initial work to correct these deficiencies and produce an exact path
kernel are presented in Chapter ~\ref{Chapter4}. This new exact path kernel
representation for neural networks has a primary advantage which is to
decompose model predictions into contributions from each training
point. We can use this representation to decompose predictions into
parameter gradient contributions from each training datum, thus the exact path kernel implicitly maps the
training data into a given tangent space. This mapping can be done for
parameter gradients, input space gradients, and several other
spaces related to both of these. Two such decompositions are applied
in Chapter~\ref{Chapter4a} to demonstrate that many cutting edge OOD
detection algorithms implicitly use this decomposition in terms of
parameter gradients, and that training input gradients can be
used to measure signal manifold dimension.

We can summarize the first 3 Chapters as posing some fundamental
geometric questions related to machine-learning robustness. 
Chapters 4 and 5 as propose a new framework for analyzing geometric
properties and demonstrating that this framework can be applied very
generally. As usual in mathematics, trying to solve a specific problem
about geometric causes for adversarial examples has led to a very
general result in a totally different field. Although Chapters 4 and 5
represent significant progress in understanding ANNs mathematically,
they do not directly address the core question that arose within the
first three Chapters: Is relative dimension and curvature a primary
cause of adversarial vulnerability. 

Based on this work, we can now make some supportable conjectures:

\begin{enumerate}
  \item Modern neural network architectures learn decision boundaries
    that are askew from many interpolations among training and testing
    data.
  \item Non-orthogonal decision boundary crossings indicate that
    sharp corners in the decision space learned by ANN models lead
    to adversarial examples.
  \item Bilinear map based representations allow decomposition of
    predictions based on training inputs. These show that models
    implicitly use a lower dimensional implicit representation of data
    to make predictions.
\end{enumerate}

Combining these notions, we have a clear path forward:
    Decompose adversarial attacks using a bilinear map
    representation and then compare these signatures of variation with
    typical modes of variation within training data.
Likely this will show that adversarial attacks are arising
    from sharp corners in decision space, some of which can be
    mitigated by making models more aware of the properties of the
    decision space they have learned, some of which cannot be
    mitigated because training data are insufficient to fully
    constrain decision surfaces. 
From here, there are several important directions for
future research. First is the application of these methods directly to
the adversarial robustness questions from Chapters 1-3. Second is the
complete generalization of this new framework including conditions
under which such representations live in Banach spaces or Hilbert
spaces and what order of accuracy can be maintained using truncated
singular value decomposition. Third is the application of this theory more
broadly to connect with Wasserstein metric spaceswhich is a suspected implicit space learned by machine-learning models. 

\section{Applications to Adversarial Robustness}

\begin{figure}[ht!]
    \centering
    \includegraphics[width=0.42\textwidth]{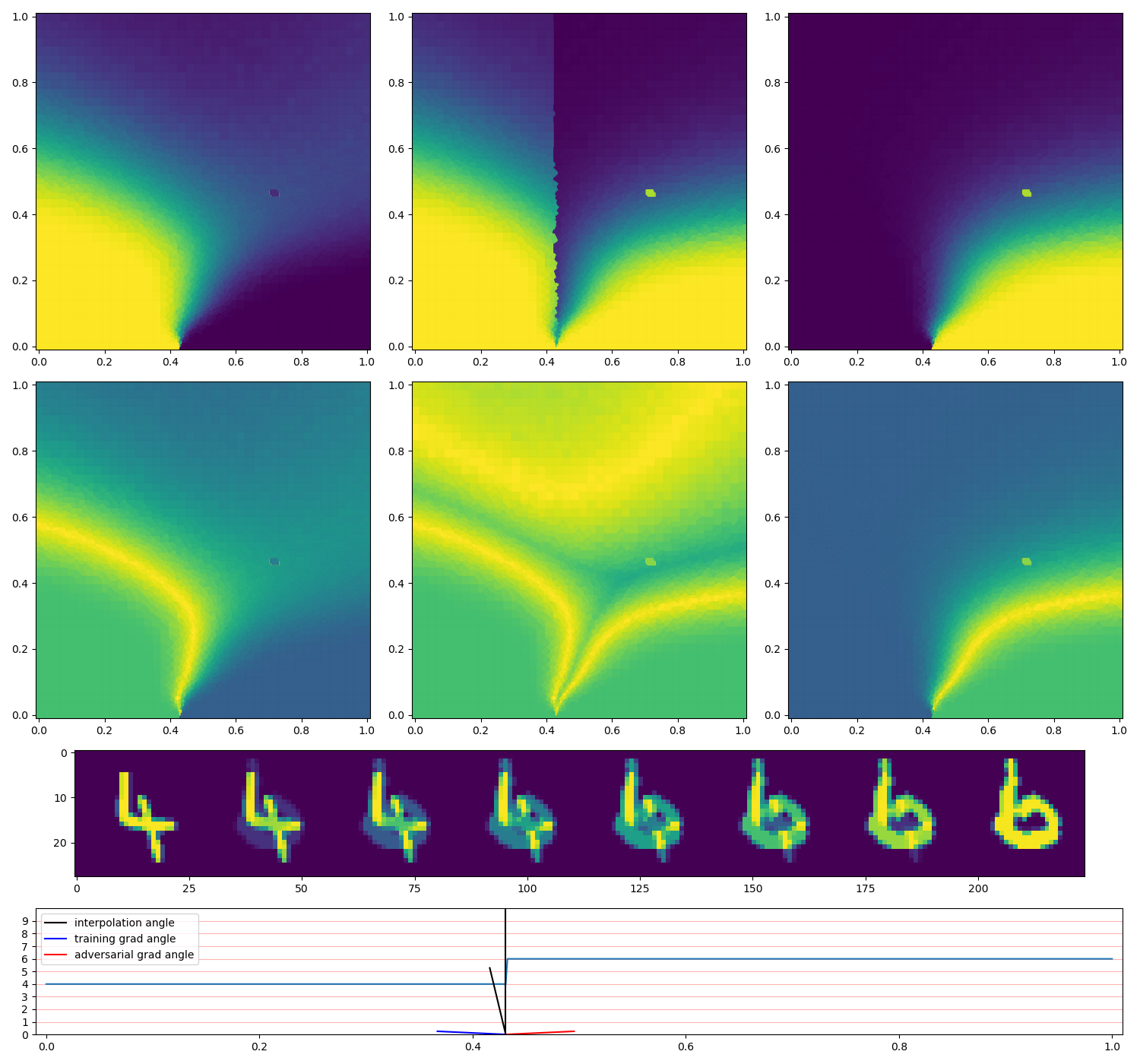}\includegraphics[width=0.42\textwidth]{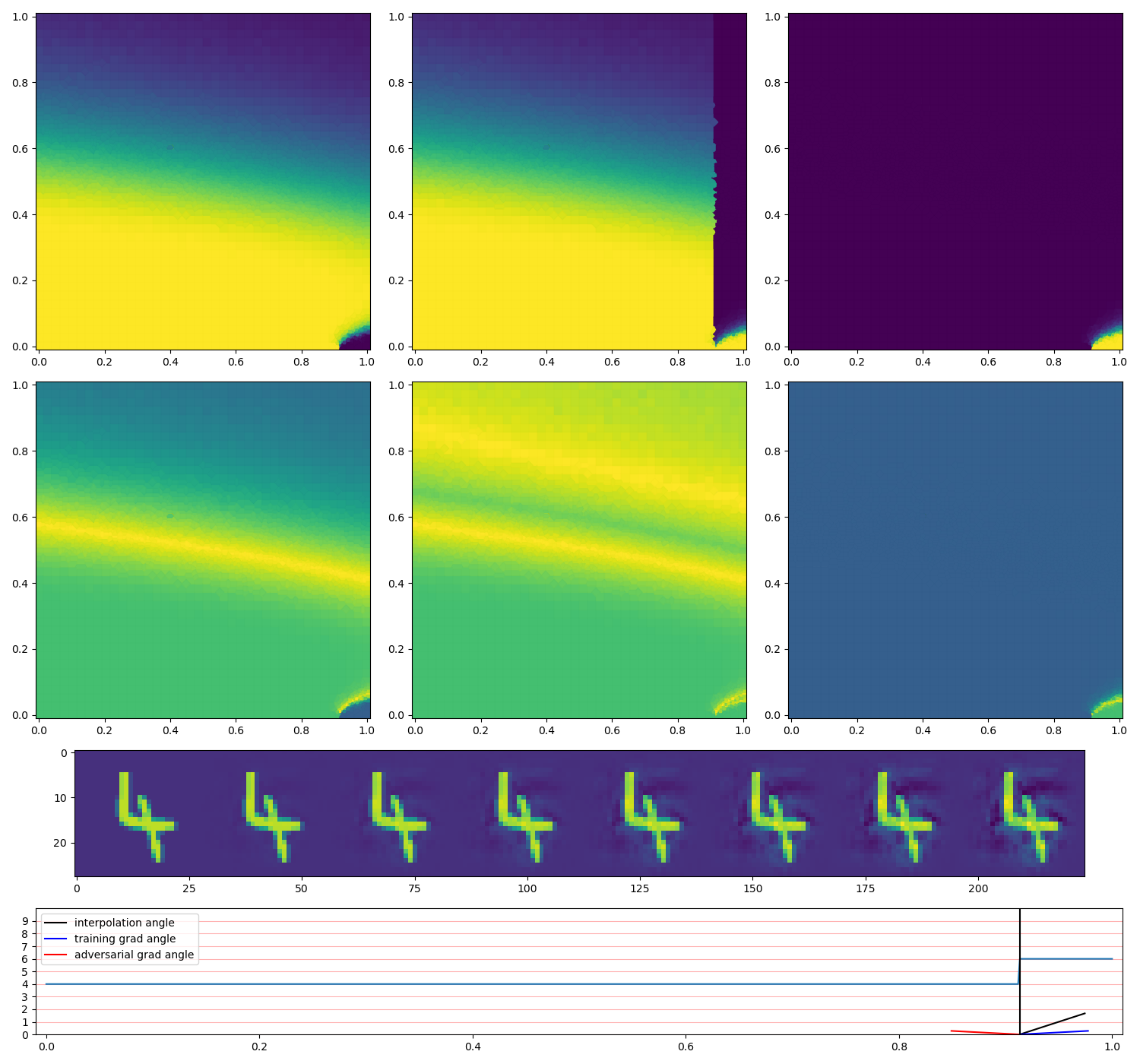}
    \caption{Heatmaps (top) and decision boundary angle plots bottom) showing stability fractions for an interpolation from a natural 4 to a natural 6 (left) and from a natural 4 to an adversarial 6 (right). }
    \label{fig:dbs}
\end{figure}

We can study persistence as a level-set in a sampling versus interpolation plot. We will design this approach using a Delaunay triangulation to iteratively sample this domain in order to visualize a decision boundary between natural images and between an adversarial image and a natural image. We will design heatmaps to show the fraction of uniform samples from a sphere which return the same class (E.g. an image along the interpolation for, say $t = 0.5$ is classified by the model as class 4. Then we will plot a column in the heatmap showing fraction of samples from Gaussians with progressively higher standard deviations which receive class 4.) This visualization is shown in Fig.~\ref{fig:dbs} where the middle  heatmaps show the 0.7 level set of these fractions, middle images show sample images along the linear interpolation across the decision boundary, and bottom plot shows incident angles on the decision boundary of the interpolant (black), the training gradient (blue) and the testing gradient (red). We can see that sharp geometric curvature and shallow angles are observed when interpolating across decision boundaries between natural and especially adversarial images. One line of future research will follow the direct geometric approach by constructing test objects (e.g. wedges which are intersections of angled hyper-planes forming arbitrarily sharp solids with arbitrarily many dimensions fewer than its ambient space) which replicate the structure observed in the practical networks. An example of this is shown in ~\ref{fig:dbs}.

      The second approach is to take advantage of the framework from Chapters ~\ref{Chapter4} and ~\ref{Chapter4a} in order to decompose the training gradients at points on the decision boundary to understand neural networks' learned degrees of freedom at these locations. The goal of this line of research is to understand these geometric constraints and eventually pose both updated training objectives and also better definitions for the identification of adversarial examples in practice. This line of research may have implications beyond robustness, to include uncertainty quantification, ability of models to generalize to data outside their training distribution, 
out-of-distribution detection, and other useful metrics to create
trustworthy AI. In particular, we can decompose adversarial gradients into an
initial gradient plus a component from each training input,
\begin{align}
  f(x, \theta_F) &= f(x, \theta_0) + \sum_i \sum_s \int_0^1 \langle
  \nabla_\theta f(x, \theta_s(t)), f(x_i, \theta_s(t))\rangle dt \\
  \dfrac{\partial f(x, \theta_F)}{\partial x} &= \dfrac{\partial f(x, \theta_0)}{\partial x} + \sum_i \sum_s \int_0^1 \langle
  \nabla_\theta \dfrac{\partial f(x, \theta_s(t))}{\partial x}, f(x_i,
                                 \theta_s(t))\rangle dt.
\end{align}
From this decomposition, we can also replace the individual training
data with an orthogonal basis by choosing vectors as in
~\citet{halko2011finding}, solving SVD for each step:
    Take a model, $f(x, \theta)$, which satisfies the necessary conditions for expression as:
 \begin{align}
     f(x, \theta_\text{tr}) &= f(x, \theta_0(0)) +  \sum_s
\int_0^1                                  \left\langle 
                                   \varphi_{s,t}(x) , \sum_i
                                   \left(L'(x_i, y_i) \varphi_{s,
                                   0}(x_i)\right) \right \rangle dt\\
     \varphi_{s,t}(x) &= \nabla_\theta f(x, \theta_s(t))
 \end{align}

Then we will let the RHS be rows of a matrix $L \cdot A$ where $X$ is a vector with elements $\varphi_{s,t}(x)$, $L =
$diag($L'(x_i, y_i))$ (size is $[N, N]$), and each row $A_i
= \varphi_{s,0}(x_i)$ so that $A$ has size $[N, P]$. Then we can compute $U\Sigma V^T =
A$ ($U$ of size $[N, P]$, $\sigma$ of size $[N, N]$, and $V^T$ of size $[P, P]$. Let's first try projecting each $A_i$ onto the row space of $V$ so that 
\begin{align}
    A_i &= \sum_k A_i V_k^T V_k \\
    \sum_i \left(L'(x_i, y_i) \varphi_{s, 0}(x_i)\right) = \sum_i L_i A_i &= \sum_i \sum_k L_i A_i V_k^T V_k\\
    \sum_i L_i A_i &= \sum_k \left(\sum_i L_i A_i V_k^T\right) V_k\\
    \sum_i L_i A_i &= \sum_k \left((\vec 1^T L A V^T)_k\right) V_k\\
\end{align}
where the subscripts all indicate rows. 
If we truncate  $V$ down to $K$ dimensions, then we have $B = \vec 1^T L A (V^K)^T$ with the matrices having sizes $[1,N], [N,N], [N,P], [P,K]$ (in order) and we can rewrite the truncated representation:

\begin{align}
    \sum_i L_i A_i &= \sum_k \left((\vec 1^T L A (V^K)^T)_k\right) V^K_k\\
    &= \sum_k B_k V^K_k.\\
\end{align}

Now, we can rewrite our original expression:
\begin{align}
     f(x, \theta_\text{tr}) &= f(x, \theta_0(0)) +  \sum_s
\int_0^1                                  \left\langle 
                                   \varphi_{s,t}(x) ,  \sum_k B_k V^K_k \right \rangle dt.
 \end{align}
More generally, we may define Fourier features along paths connecting all of the finitely many training steps together using methods like those from ~\citet{tancik2020fourierfeatures}. By combining these Fourier features spectral components across all training steps, it may be possible to perform a single spectral decomposition which provides a basis for the entire discrete path determined by a model during training. 

Regardless of how sophisticated we make this decomposition, all of
these methods have the advantage of maintaining an exact decomposition of predictions relative to contributions from each training input. This mapping implicitly
defines natural dimension reduction -- by truncating the chosen
basis. This provides a theoretical foundation
from which recent work in manifold study and data augmentation
~\citep{kaufman_data_2023, liu_linear_2023, sipka_differentiable_2023,
  cha_orthogonality-enforced_2023, marbut_reliable_2023,
  gao_out--domain_2023, oh_provable_2023, chen2023aware}.

\section{Generalization in the sense of Reproducing Kernel Banach Spaces}
Given that the representation from Chapter~\ref{Chapter4} and
~\ref{Chapter4a} have a small asymmetry, a general description of
these representations cannot fit within Hilbert Space. There is a
convenient approach building on the theory of Reproducing Kernel
Banach Spaces (RKBS) which is summarized nicely by
~\citet{zhang2009reproducing}. It is by careful construction of a
semi-inner product that I believe our representation can be written in
this way. This allows access to tools built for Banach spaces for
analysis of both accuracy, risk minimization, and other useful
results. A similar line of work is already being pursued by
~\citet{shilton_gradient_2023} with which our work can likely be
connected. Although this is less likely to produce practical
payoffs immediately, I believe that this approach will greatly enhance
the theoretical foundation upon which analysis of neural network
performance and limitations are based. 

\section{Connecting Distributional Learning with Neural Networks}

The neural representations and decompositions proposed in this work
provide images of the tangent space according to the
a given model for each point in the data space. The decomposed gradients for each input datum must be aligned with the natural geodesics which interpolate data. It has been shown increasingly in recent work
(e.g. by  ~\citet{lu2020universal}, ~\citet{yang2022capacity}, \citet{altekruger_neural_2023}) shows that Neural Networks
learn distributions in a sense that approximates the Wasserstein
metric to some order. Also, work by ~\citet{chizat2020maxmargin} that neural network
classifiers are approximately max-margin classifiers in some
implicit space. We cannot necessarily compute this space exactly,
however, by examining the parameter gradient decompositions exposed by this kernel representation, we can pose questions about this metric space in the
dual sense. Connecting these concepts into an
understanding of how Neural Networks embed an approximation of the
Wasserstein metric into euclidean space is likely
euclidean will have significant impact on the machine-learning
community at large.

 
%
%

\renewcommand\thechapter{}

\appendix 

 
\addtocontents{toc}{\protect\renewcommand{\protect\cftchappresnum}{\appendixname\space}}
\addtocontents{toc}{\protect\renewcommand{\protect\cftchappresnum}{\appendixname\space}}
\chapter{Attacks}
\section{L-BFGS}\label{appa}
Limited-memory Broyden-Fletcher-Goldfarb-Shanno (L-BFGS)
is a quasi-newton gradient based optimization algorithm which stores a history of gradients and positions from each previous optimization step \cite{liu1989limited}. The algorithm as implemented to optimize a function $f$ with gradient at step $k$ of $g_k$ is as follows

\begin{algorithm}[H]{L-BFGS}
\begin{algorithmic}
\State Choose $x_0, m, 0 < \beta' < 1/2, \beta' < \beta < 1$, and a symmetric positive definite starting matrix $H_0$. 
\For{$k = 0$ to $k = $ (the number of iterations so far)}
\State $d_k = -H_kg_k$,
\State $x_{k+1} = x_k + \alpha_kd_k$,
Where $\alpha_k$ satisfies 
\begin{align*}
    f(x_k + \alpha_k d_k) &\leq f(x_k) + \beta'\alpha_kg_k^Td_k,\\
    g(x_k + \alpha_k d_k)^Td_k &\geq \beta g_{k}^T d_k.\\
\end{align*}
\Comment{Trying steplength $\alpha_k = 1$ first.}
\State Let $\hat m = \min(k, m - 1).$ 
\For{$i$ from 0 to $\hat m + 1$}
\Comment{Update $H_0$ $\hat m+1$ times using pairs $\{y_j,s_j\}^k_{j = k - \hat m},$}
\begin{align*}
    H_{k+1} &= (V_k^T \cdot V_{k-\hat m}^T)H_0(V_{k - \hat m}\cdots V_k)\\
    &+\rho_{k-\hat m}(V_k^T \cdots V_{k-\hat m+1}^T)s_{k - \hat m} s_{k - \hat m}^T(V_{k-\hat m+1} \cdots V_k)\\
    &+\rho_{k-\hat m + 1}(V_k^T \cdots V_{k-\hat m+2}^T)s_{k - \hat m + 1} s_{k - \hat m + 1}^T(V_{k-\hat m+2} \cdots V_k)\\
    & \vdots\\
    &+\rho_ks_ks_k^T\\
\end{align*}
\EndFor
\EndFor
\end{algorithmic}
\end{algorithm}

\newpage

\chapter{Persistence Tools}

\section{Bracketing Algorithm}\label{bracketing}
This algorithm was implemented in Python for the experiments presented. 

\begin{algorithm}
\begin{algorithmic}
\Function{bracketing}{image, ANN, n, tol, n\_real, c\_i}
 \Comment{ start with same magnitude noise as image}
 \State u\_tol, l\_tol = 1.01, 0.99
 \State a\_var = Variance(image)/4 \Comment{Running Variance}
 \State l\_var, u\_var  = 0, a\_var*2\Comment{ Upper and Lower
   Variance of search space}
 \Comment{ Adversarial image plus noise counts}
 \State a\_counts = zeros(n)
 \State n\_sz = image.shape[0]
 \State mean = Zeros(n\_sz)
 \State I = Identity(n\_sz)
 \State count = 0
\Comment{grab the classification of the image under the network}
\State y\_a = argmax(ANN.forward(image))
\State samp = N(0, u\_var*I, n\_real)
\State image\_as = argmax(ANN.forward(image + samp))
\Comment{Expand search window}
\While{Sum(image\_as == y\_a) $>$ n\_real*tol/2}
\State u\_var = u\_var*2
\State samp = N(0, u\_var*I, n\_real)
\State image\_as = argmax(ANN.forward(image + samp))
\EndWhile
  \Comment{ perform the bracketing }
\For{i in range(0,n)}
\State count+=1
\Comment{compute sample and its torch tensor}
\State samp = N(0, a\_var*I, n\_real)

\State image\_as = argmax(ANN.forward(image + samp))

\State a\_counts[i] = Sum(image\_as == y\_a)

\Comment{floor and ceiling surround number}
\If{((a\_counts[i] $\leq$ Ceil(n\_real*(tol*u\_tol))) \& (a\_counts[i] $>$ Floor(n\_real*(tol*l\_tol))))}

        \Return{a\_var}
    \ElsIf{ (a\_counts[i] $<$ n\_real*tol)} \Comment{we're too high}
        \State u\_var = a\_var
        \State a\_var = (a\_var + l\_var)/2
    \ElsIf{ (a\_counts[i] $\geq$ n\_real*tol)} \Comment{we're too low}
        \State l\_var = a\_var
        \State a\_var = (u\_var + a\_var)/2
        \EndIf
\EndFor

   \Return{a\_var}
\EndFunction
\end{algorithmic}
\end{algorithm}

\begin{algorithm} [h!]
\begin{algorithmic}
\Function{bracketing}{image, classifier ($\CC$), numSamples, $\gamma$, maxSteps, precision}

\State $[\sigma_{\min},\sigma_{\max}] = $\textproc{rangefinder}(image, $\CC$, numSamples, $\gamma$)
\State count $=1$
\While{count$<$maxSteps}
\State $\sigma = \frac{\sigma_{\min}+\sigma_{\max}}{2}$
\State $\gamma_{\textnormal{new}} =$ \textproc{compute\_persistence}($\sigma$, image, numSamples, $\CC$)
\If{$|\gamma_{\textnormal{new}}-\gamma|<$precision}
\State \textbf{return} $\sigma$ 
\ElsIf{$\gamma_{\textnormal{new}}>\gamma$}
\State $\sigma_{\min} = \sigma$
\Else
\State $\sigma_{\max} = \sigma$
\EndIf
\State count = count + 1
\EndWhile

\Return $\sigma$
\EndFunction

\\
\Function{rangefinder}{image, $\CC$, numSamples, $\gamma$}
\State $\sigma_{\min}=.5$,\;\; $\sigma_{\max}=1.5$
\State $\gamma_1 =$ \textproc{compute\_persistence}($\sigma_{\min}$, image, numSamples, $\CC$)
\State $\gamma_2 =$ \textproc{compute\_persistence}($\sigma_{\max}$, image, numSamples, $\CC$)
\While{$\gamma_1<\gamma$ \textbf{or} $\gamma_2>\gamma$}
\If{$\gamma_1<\gamma$}
\State $\sigma_{\min} = .5\sigma_{\min}$
\State $\gamma_1 =$ \textproc{compute\_persistence}($\sigma_{\min}$, image, numSamples, $\CC$)
\EndIf
\If{$\gamma_2>\gamma$}
\State $\sigma_{\max} = 2\sigma_{\max}$
\State $\gamma_2 =$ \textproc{compute\_persistence}($\sigma_{\max}$, image, numSamples, $\CC$)
\EndIf
\EndWhile

\Return{$[\sigma_{\min}, \sigma_{\max}]$}
\EndFunction

\\
\Function{compute\_persistence}{$\sigma$, image, numSamples, $\CC$}
\State sample = $N(\textnormal{image},\sigma^2I,$numSamples)
\State $\gamma_{\textnormal{est}} = \frac{|\{\CC(\textnormal{sample})=\CC(\textnormal{image})\}|}{\textnormal{numSamples}}$

\Return{$\gamma_{\textnormal{est}}$}
\EndFunction
\end{algorithmic}
\caption{Bracketing algorithm for computing $\gamma$-persistence}\label{bracketing}
\end{algorithm}

\section{Bracketing Algorithm}\label{sec:bracketing}
The Bracketing Algorithm is a way to determine persistance of an image with respect to a given classifier, typically a DNN. The algorithm was implemented in Python for the experiments presented. The \textproc{rangefinder} function is not strictly necessary, in that one could directly specify values of $\sigma_{\min}$ and $\sigma_{\max}$, but we include it here so that the code could be automated by a user if so desired.

\section{Convolutional Neural Networks Used} \label{appendix:CNNs}
In Table \ref{table1} we reported results on varying complexity convolutional neural networks. These networks consist of a composition of convolutional layers followed by a maxpool and fully connected layers. 
The details of the network layers are described in Table \ref{tab:CNN} where Ch is the number of channels in the convolutional components. 

\begin{table}[pt]
\centering
\caption{Structure of the CNNs C-Ch used in Table \ref{table1}}
\label{tab:CNN}
\begin{tabular}{llllll}
\toprule
     Layer & Type & Channels & Kernel & Stride & Output Shape \\
\midrule
     0 & Image & 1 & NA & NA & $(1, 28, 28)$ \\
     1 & Conv & Ch & $(5,5)$& $(1,1)$& $(\textnormal{Ch}, 24, 24)$\\
     2 & Conv & Ch & $(5,5)$& $(1,1)$& $(\textnormal{Ch}, 20, 20)$\\
     3 & Conv & Ch & $(5,5)$& $(1,1)$& $(\textnormal{Ch}, 16, 16)$\\
     4 & Conv & Ch & $(5,5)$& $(1,1)$& $(\textnormal{Ch}, 12, 12)$\\
     5 & Max Pool & Ch & $(2, 2)$ & $(2, 2)$& $(\textnormal{Ch}, 6, 6)$ \\
     7 & FC & $(\textnormal{Ch}\cdot 6 \cdot 6, 256)$ & NA & NA & 256 \\
     8 & FC & $(256, 10)$ & NA & NA & 10 \\
     \bottomrule
\end{tabular}
\end{table}
 

\section{Additional Figures}
In this section we provide additional figures to demonstrate some of the experiments from the paper.


\subsection{Additional Figures from MNIST}
In Figure \ref{fig:mnistadv} we begin with an image of a \texttt{1} and generate adversarial examples to the networks described in Section \ref{sec:mnist} via IGSM targeted at each class \texttt{2} through \texttt{9}; plotted are the counts of output classifications by the DNN from samples from Gaussian distributions with increasing standard deviation; this complements Figure \ref{fgsmo} in the main text. Note that the prevalence of the adversarial class falls off quickly in all cases, though the rate is different for different choices of target class.
\begin{figure}[!htb]
    \centering
    \includegraphics[width=.49\textwidth]{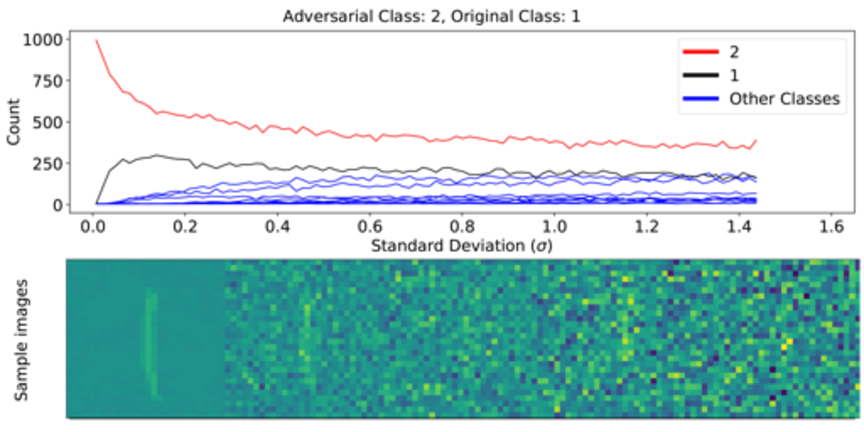}
    \includegraphics[width=.49\textwidth]{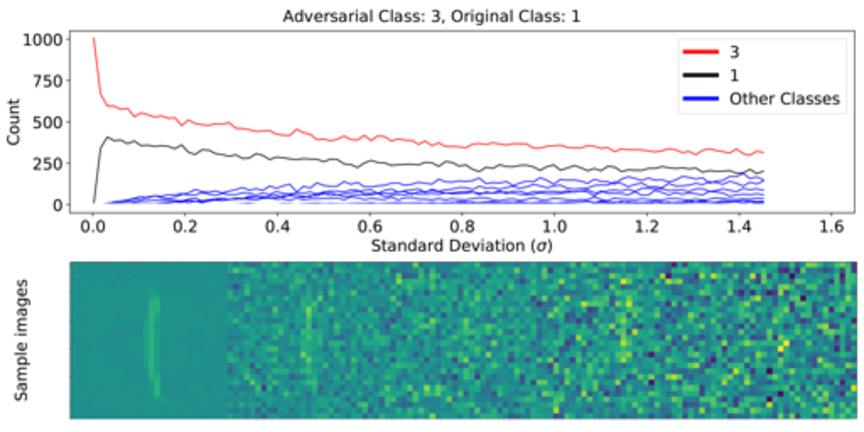}
    \includegraphics[width=.49\textwidth]{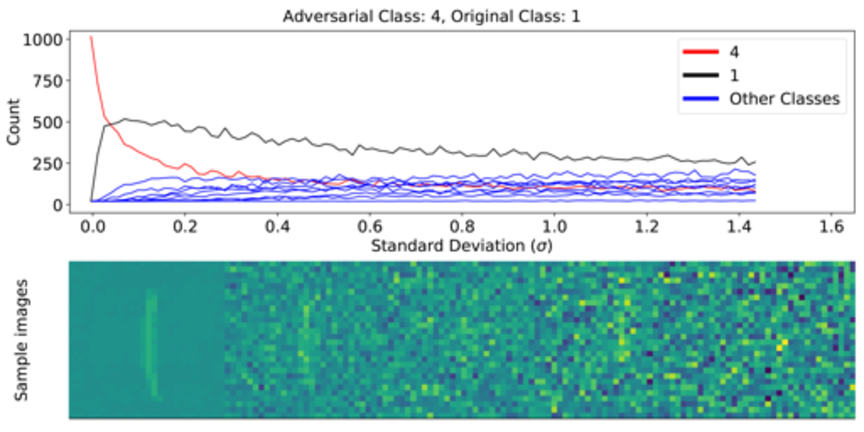}
    \includegraphics[width=.49\textwidth]{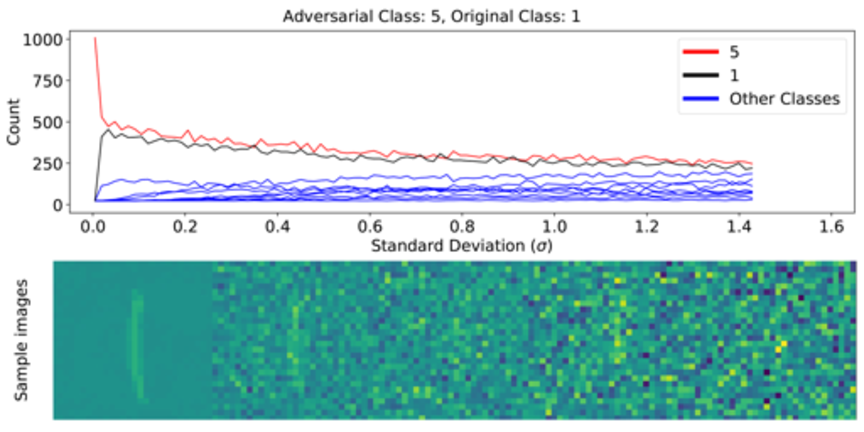}
    \includegraphics[width=.49\textwidth]{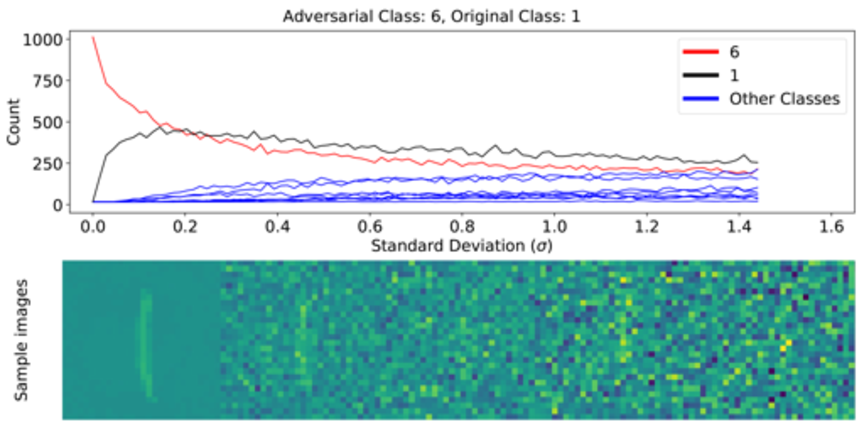}
    \includegraphics[width=.49\textwidth]{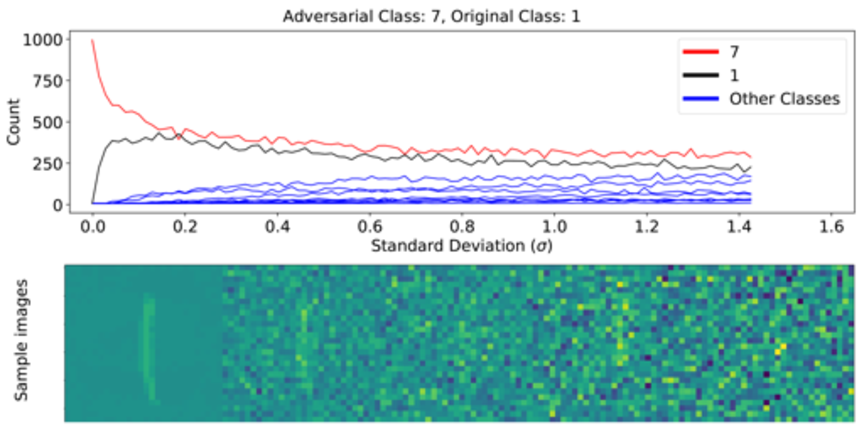}
    \includegraphics[width=.49\textwidth]{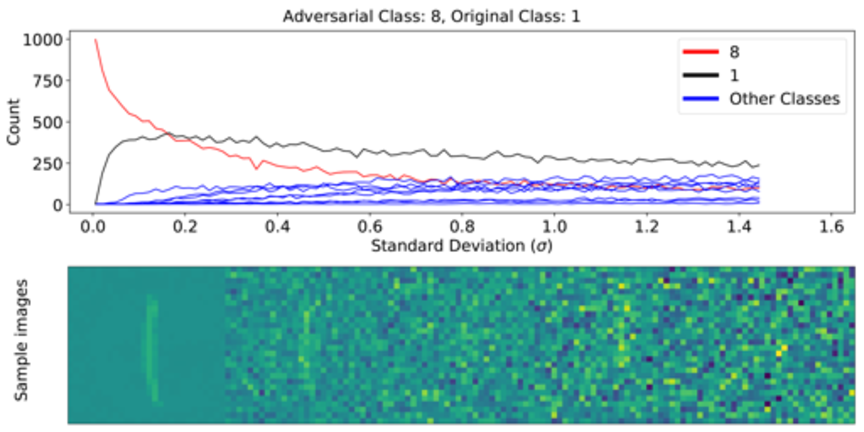}
    \includegraphics[width=.49\textwidth]{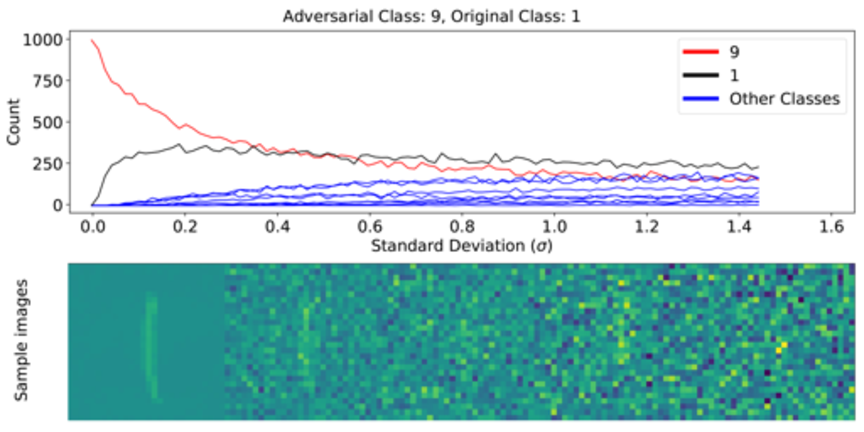}
    \caption{Frequency of each class in Gaussian samples with increasing standard deviations around adversarial attacks of an image of a \texttt{1} targeted at classes \texttt{2} through \texttt{9} on a DNN classifier generated using IGSM. The adversarial class is shown as a red curve. The natural image class (\texttt{1}) is shown in black. Bottoms show example sample images at different standard deviations.}
    \label{fig:mnistadv}
\end{figure}

We also show histograms corresponding to those in Figure \ref{fig:IGSMpersistenceMNIST} and the networks from Table \ref{table1}.  As before, for each image, we used IGSM to generate 9 adversarial examples (one for each target class) yielding a total of 1800 adversarial examples. In addition, we randomly sampled 1800 natural MNIST images. For each of the 3600 images, we computed $0.7$-persistence. In Figure \ref{fig:FC10}, we see histograms of these persistences for the small fully connected networks with increasing levels of regularization. In each case, the test accuracy is relatively low and distortion relatively high. It should be noted that these high-distortion attacks against models with few effective parameters were inherently very stable -- resulting in most of the ``adversarial'' images in these sets having higher persistence than natural images. This suggests a lack of the sharp conical regions which appear to characterize adversarial examples generated against more complicated models.  In Figure \ref{fig:FC100200} we see the larger fully connected networks from Table \ref{table1} and in Figure \ref{fig:CNNs} we see some of the convolutional neural networks from Table \ref{table1}. 

\begin{figure}[!htb]
\centering
\includegraphics[trim=200 80 100 100, clip,width=.32\textwidth]{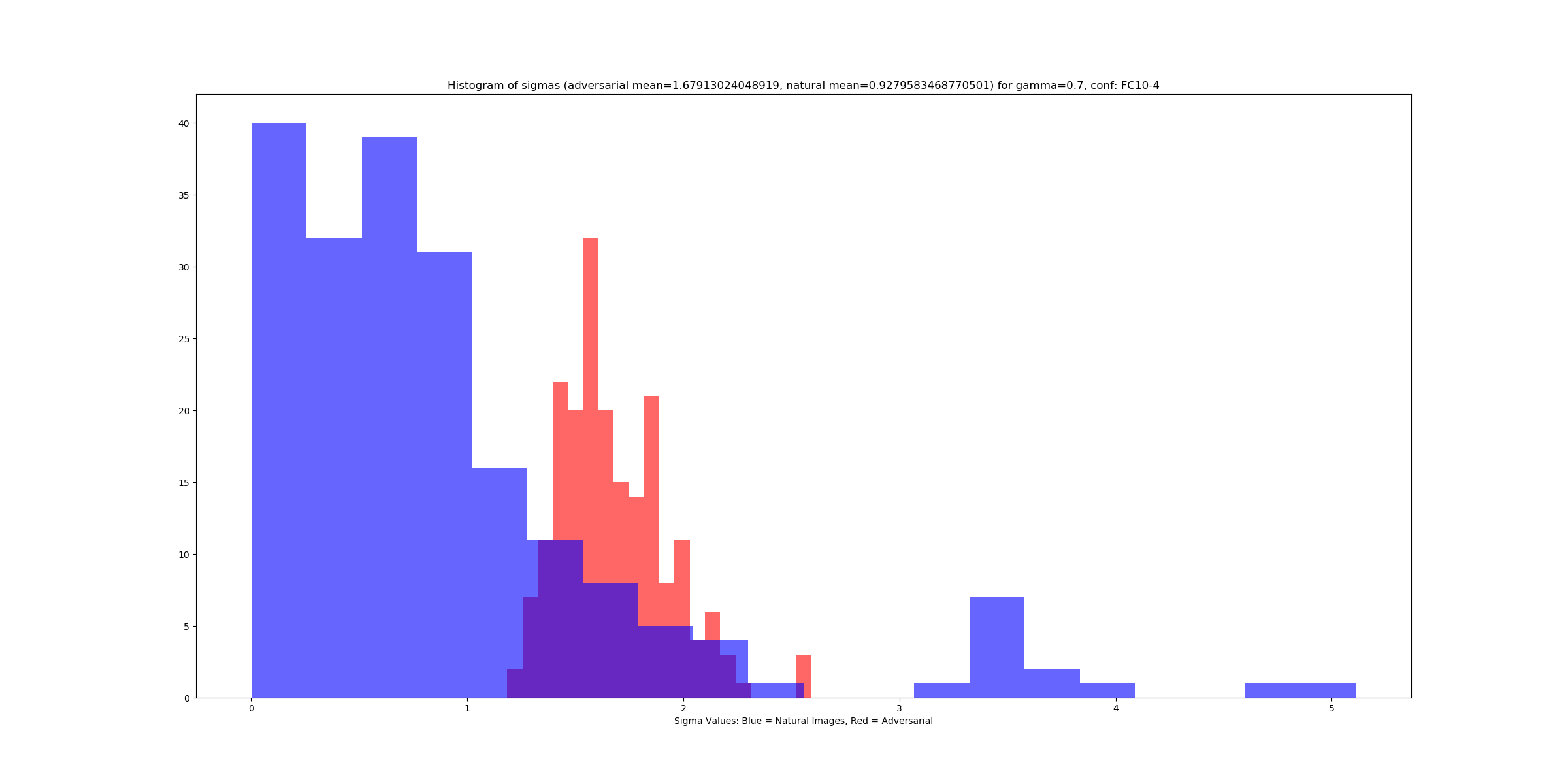}
\includegraphics[trim=200 80 100 100, clip,width=.32\textwidth]{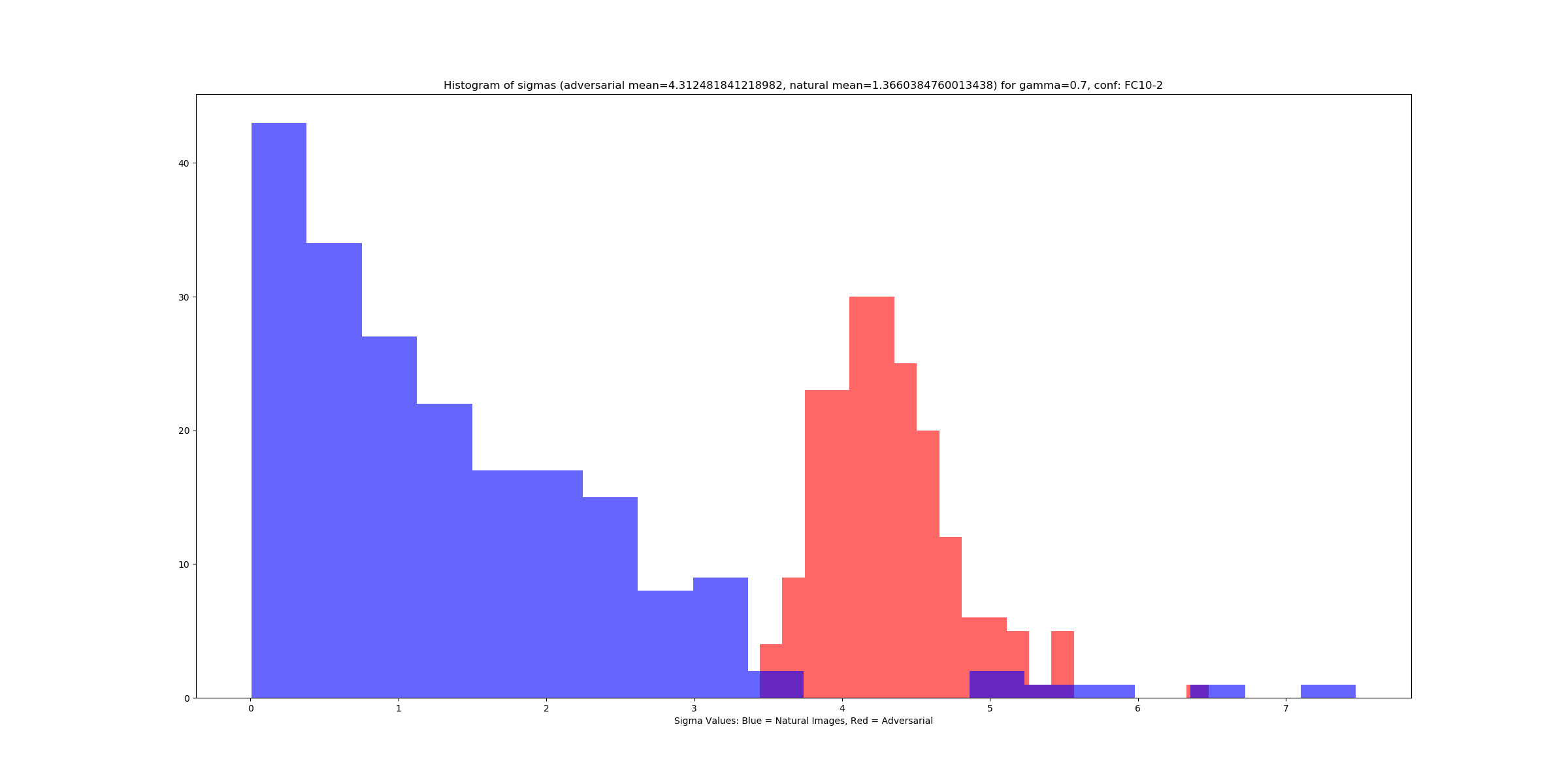}
\includegraphics[trim=200 80 100 100, clip,width=.32\textwidth]{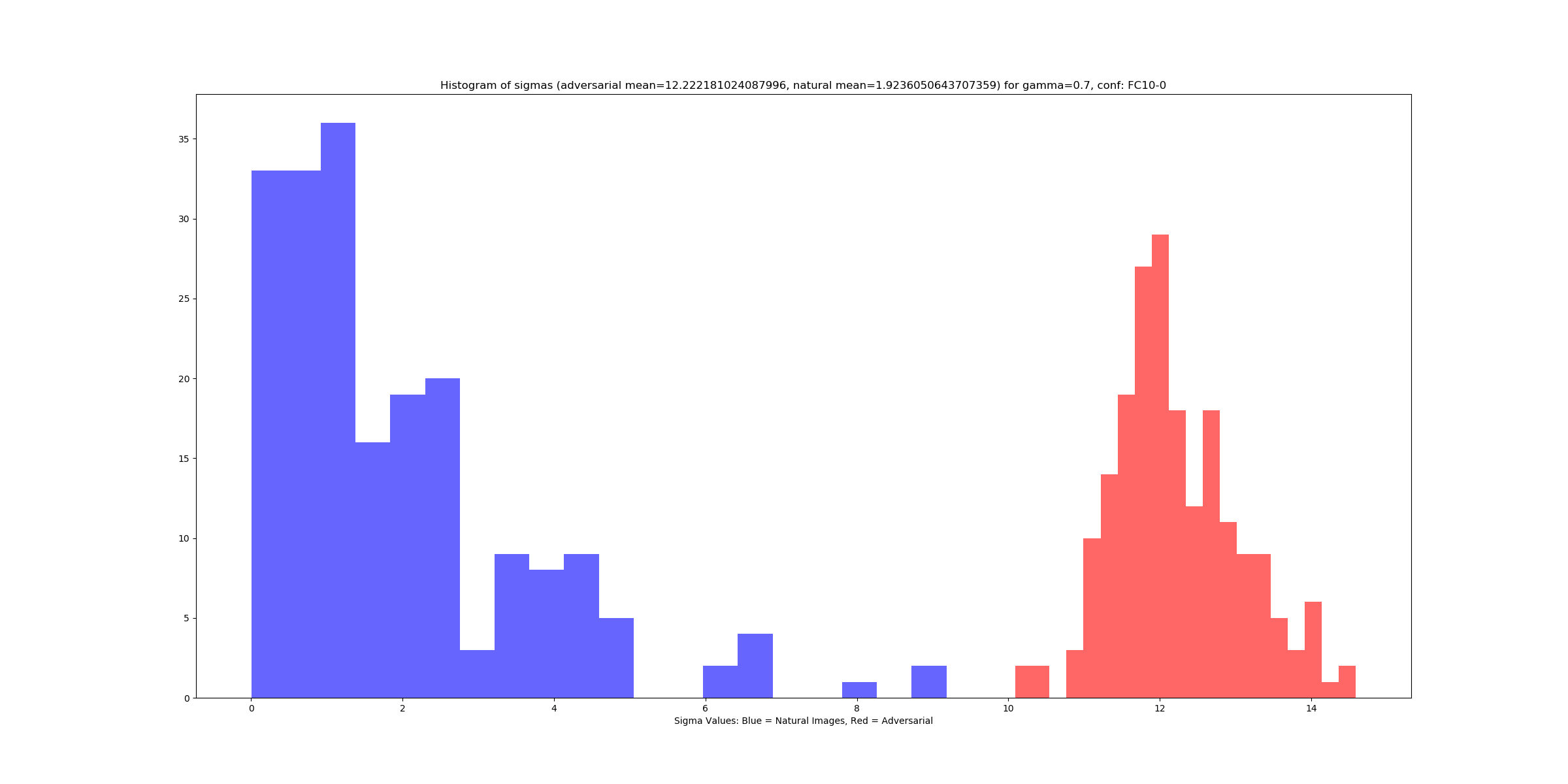}
\caption{Histograms of $0.7$-persistence for FC10-4 (smallest regularization, left), FC10-2 (middle), and FC10-0 (most regularization, right) from Table \ref{table1}. Natural images are in blue, and adversarial images are in red. Note that these are plotted on different scales -- higher regularization forces any "adversaries" to be very stable.\vspace{2em} }
\label{fig:FC10}
\end{figure}

\begin{figure}[!htb]
\centering
\includegraphics[trim=200 80 100 100, clip,width=.49\textwidth]{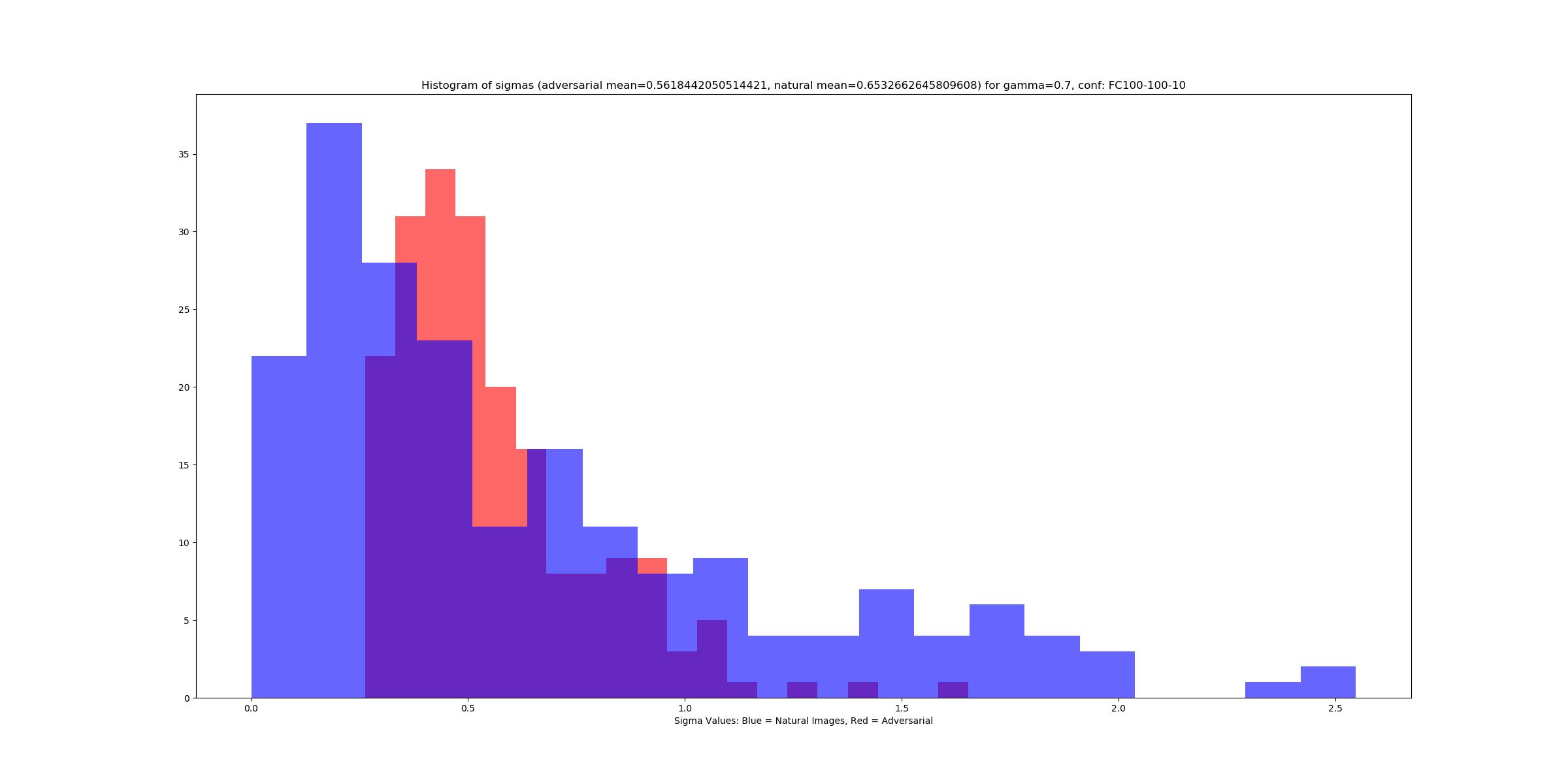}
\includegraphics[trim=200 80 100 100, clip,width=.49\textwidth]{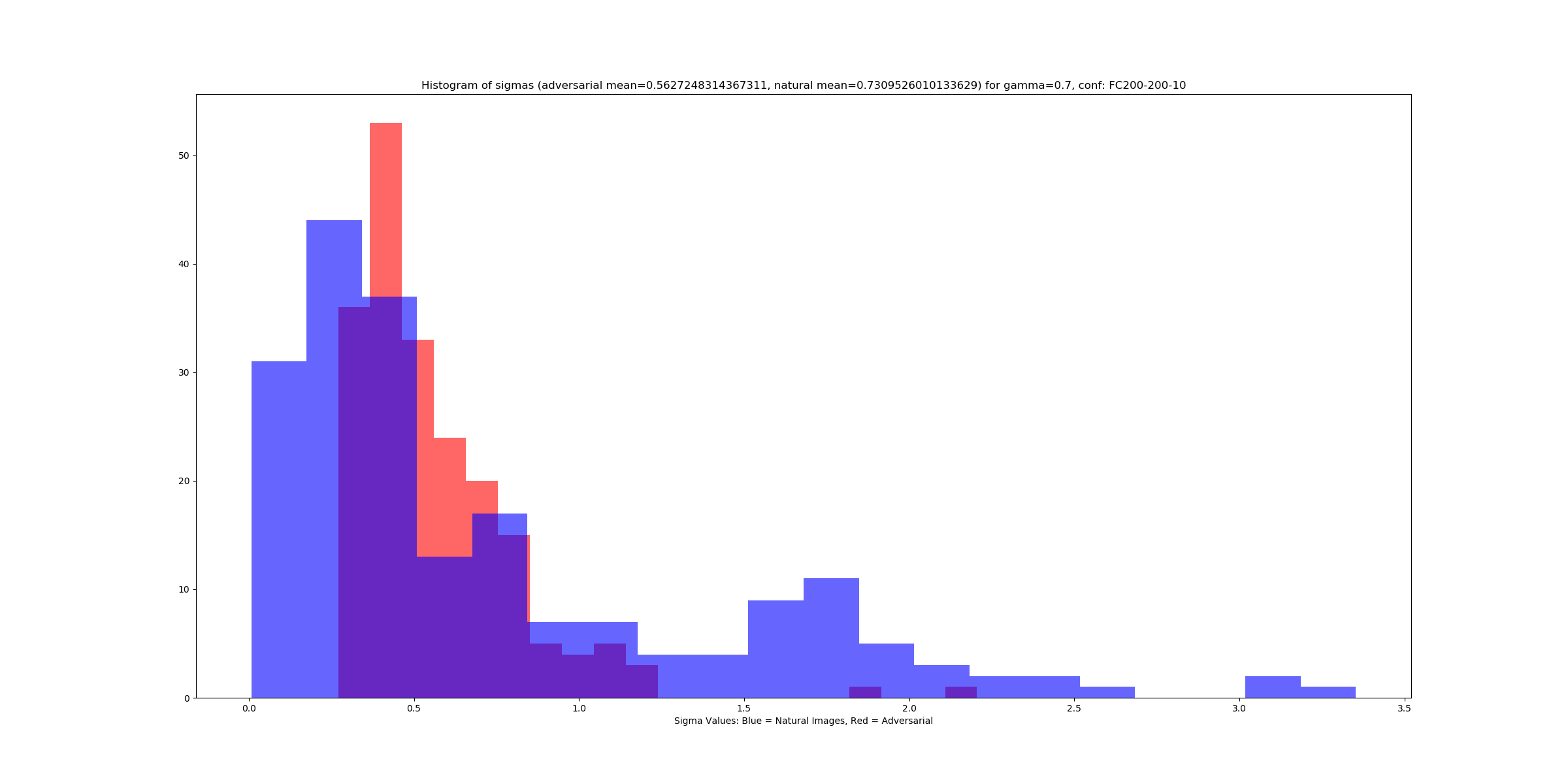}
\caption{Histograms of $0.7$-persistence for FC100-100-10 (left) and FC200-200-10 (right) from Table \ref{table1}. Natural images are in blue, and adversarial images are in red.}
\label{fig:FC100200}
\end{figure}

\begin{figure}[!htb]
\includegraphics[width=.49\textwidth]{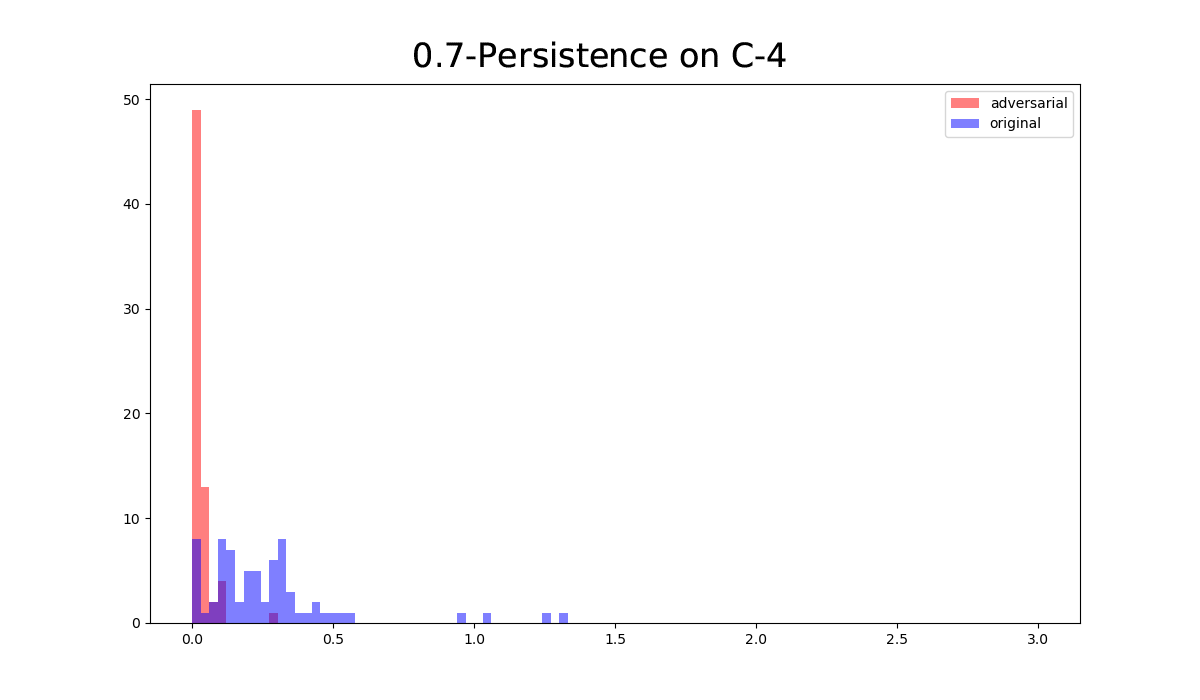}
\includegraphics[width=.49\textwidth]{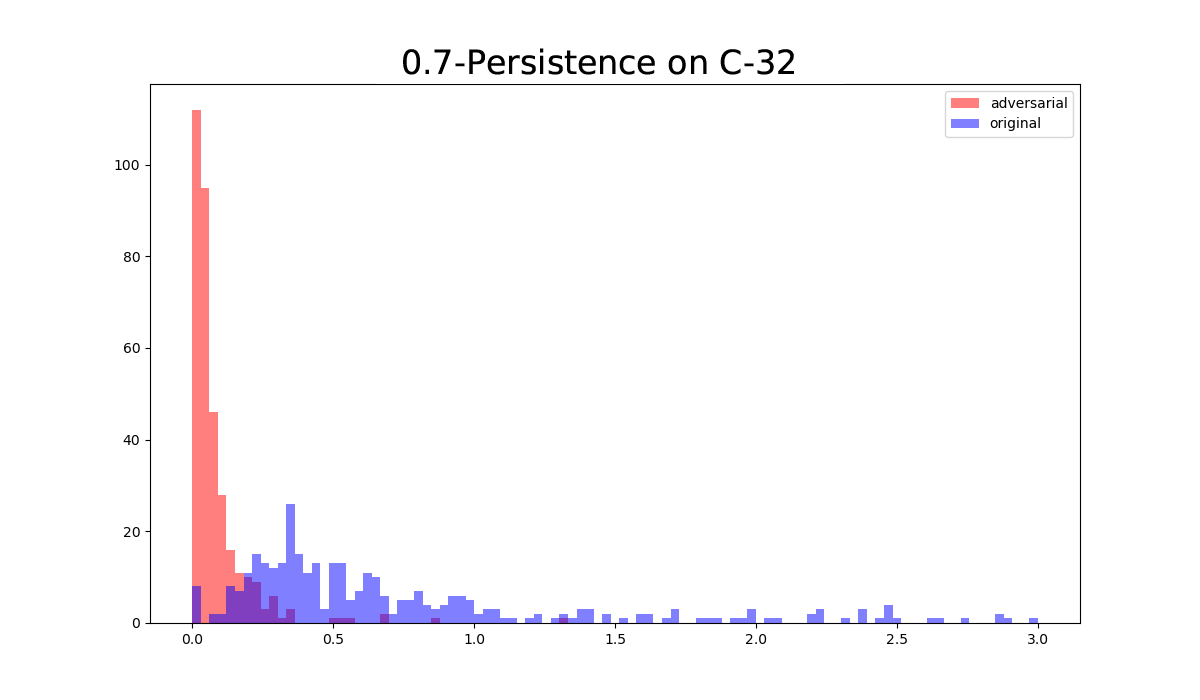}
\includegraphics[width=.49\textwidth]{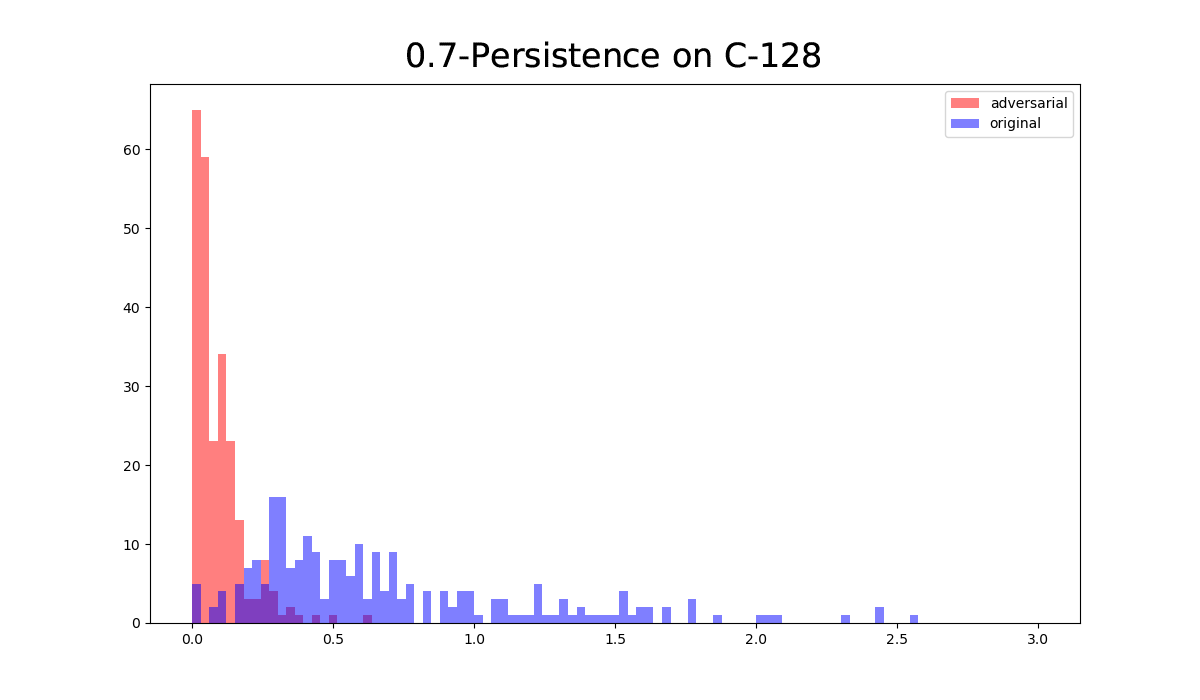}
\includegraphics[width=.49\textwidth]{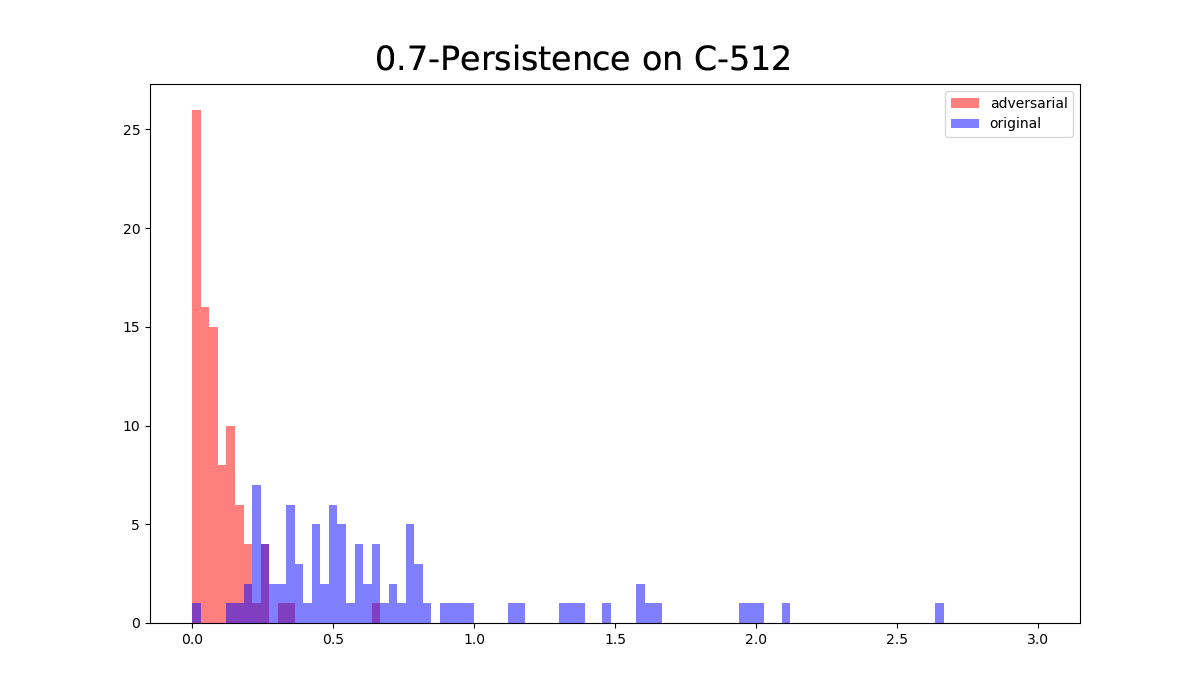}
\caption{Histograms of $0.7$-persistence for C-4 (top left), C-32 (top right), C-128 (bottom left), and C-512 (bottom right) from Table \ref{table1}. Natural images are in blue and adversarial images are in red.}
\label{fig:CNNs}
\end{figure}

\subsection{Additional figures for ImageNet}

In this section we show some additional figures of Gaussian sampling for ImageNet. In Figure \ref{fig:moreimagenet} we see Gaussian sampling of an example of the class \texttt{indigo\_bunting} and the frequency samplings for adversarial attacks of \texttt{goldfinch} toward \texttt{indigo\_bunting} (classifier: alexnet, attack: PGD) and  toward  \texttt{alligator\_lizard} (classifier: vgg16, attack: PGD). Compare the middle image to Figure \ref{fig:imagenet_adv}, which is a similar adversarial attack but used the vgg16 network classifier and the BIM attack. Results are similar. Also note that in each of the cases in Figure \ref{fig:moreimagenet} the label of the original natural image never becomes the most frequent classification when sampling neighborhoods of the adversarial example. 

\begin{figure}[!htb]
    \centering
    \includegraphics[width=.32\textwidth]{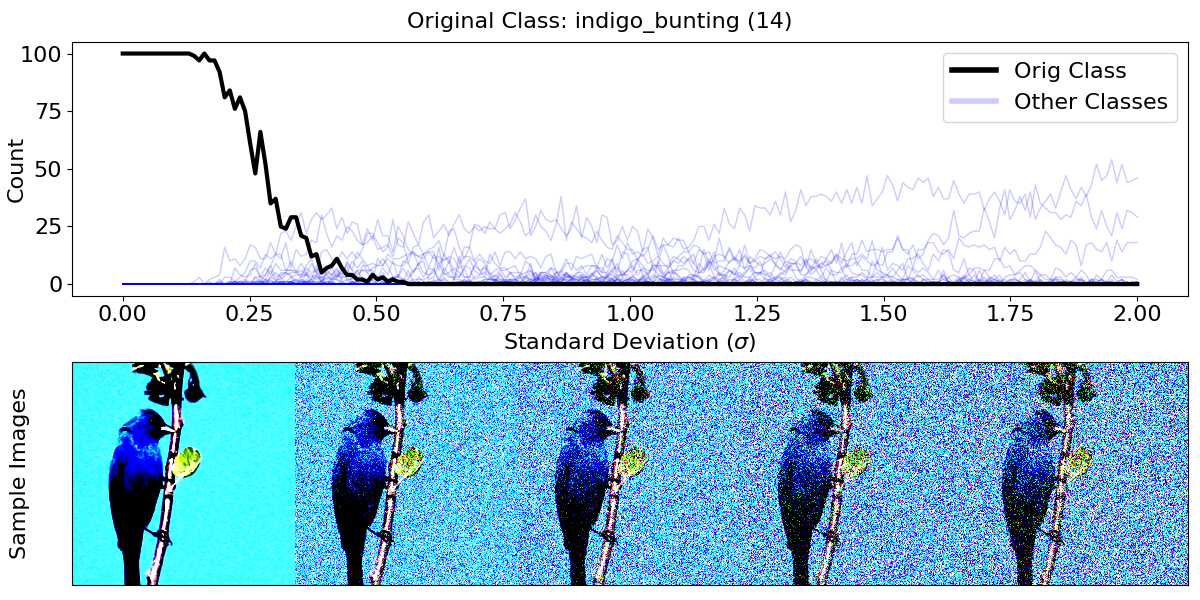}
    \includegraphics[width=.32\textwidth]{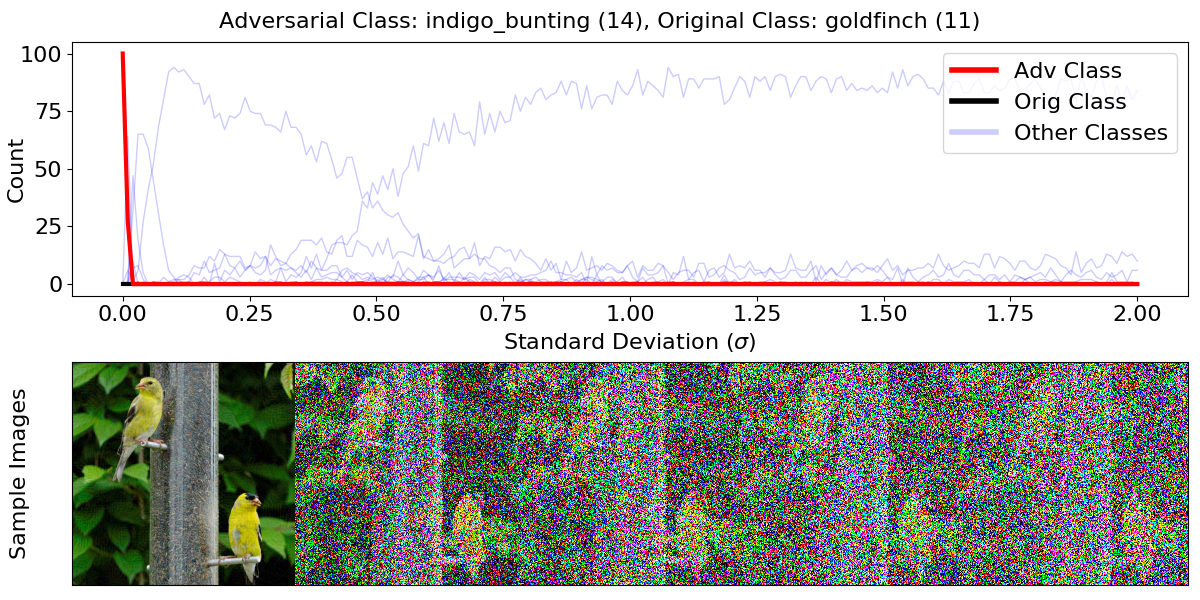}
    \includegraphics[width=.32\textwidth]{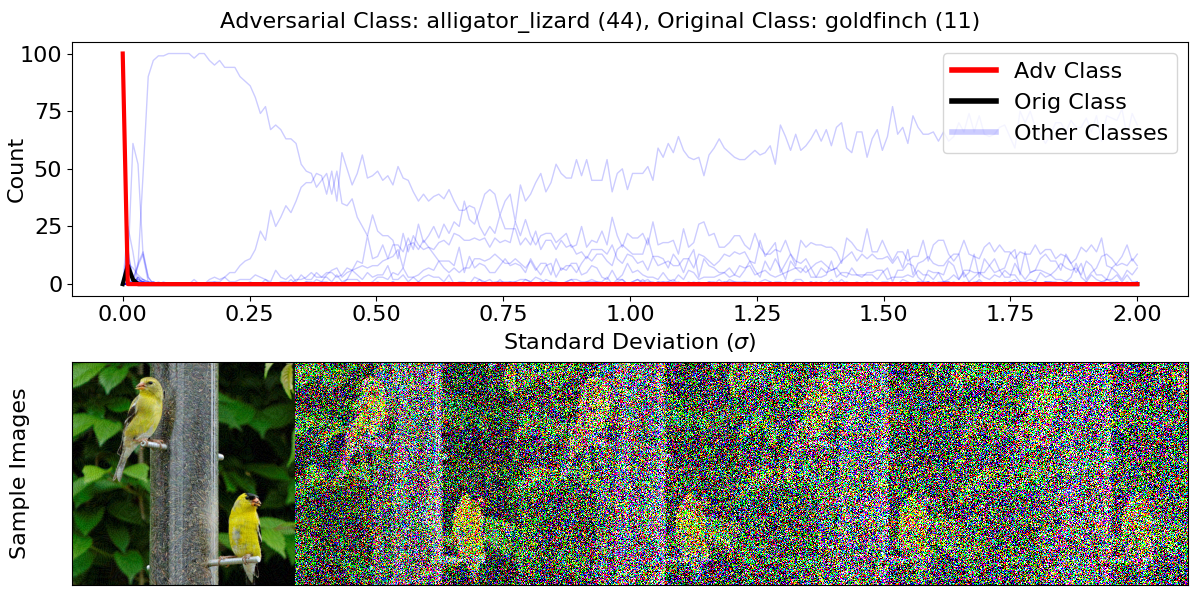}
    \caption{Frequency of each class in Gaussian samples with increasing variance around an \texttt{indigo\_bunting} image (left), an adversarial example of the image in class \texttt{goldfinch} from Figure \ref{fig:imagenet_adv} targeted at the \texttt{indigo\_bunting} class on a alexnet network attacked with PGD (middle), and an adversarial example of the \texttt{goldfinch} image targeted at the \texttt{alligator\_lizard} class on a vgg16 network attacked with PGD (right). Bottoms show example sample images at different standard deviations.}
    \label{fig:moreimagenet}
\end{figure}

In Figure \ref{fig:persistencediffgamma}, we have plotted $\gamma$-persistence along a straight line from a natural image to an adversarial image to it with differing values of the parameter $\gamma$. The $\gamma$-persistence in each case seems to change primarily when crossing the decision boundary. Interestingly, while the choice of $\gamma$ does not make too much of a difference in the left subplot, it leads to more varying persistence values in the right subplot of Figure \ref{fig:persistencediffgamma}.  This suggests that one should be careful not to choose too small of a $\gamma$ value, and that persistence does indeed depend on the landscape of the decision boundary described by the classifier.

\clearpage
\begin{figure}[!htb]
    \centering
    \includegraphics[width=.49\textwidth]{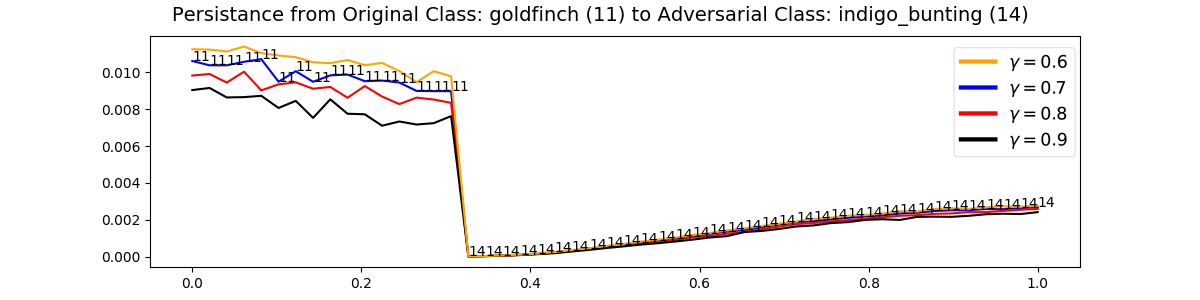}
    \includegraphics[width=.49\textwidth]{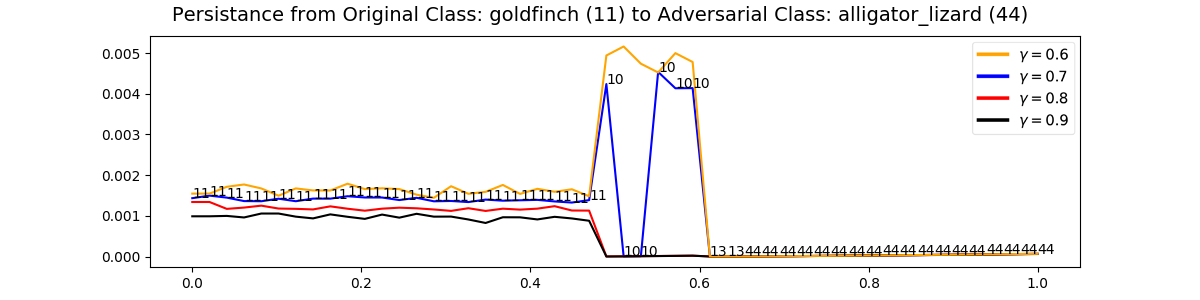}
    \caption{The $\gamma$-persistence of images along the straight line path from an image in class \texttt{goldfinch} (11) to an adversarial image generated with BIM in the class \texttt{indigo\_bunting} (14)  (left) and to an adversarial image generated with PGL in the class \texttt{alligator\_lizard} (44) (right) on a vgg16 classifier with different values of $\gamma$. The classification of each image on the straight line is listed as a number so that it is possible to see the transition from one class to another. The vertical axis is $\gamma$-persistence and the horizontal axis is progress towards the adversarial image.}
    \label{fig:persistencediffgamma}
\end{figure}

\section{Concentration of measures} \label{sec:concentration}

We use Gaussian sampling with varying standard deviation instead of sampling the uniform distributions of balls of varying radius, denoted $U(B_r(0))$ for radius $r$ and center $0$. This is for two reasons. The first is that Gaussian sampling is relatively easy to do. The second is that the concentration phenomenon is different. This can be seen in the following proposition.

\begin{proposition} \label{prop:concentration}
    Suppose $x \sim N(0,\sigma^2 I)$ and $y \sim U(B_r(0))$ where both points come from distributions on $\RR^n$. For $\varepsilon < \sqrt{n}$ and for $\delta < r$ we find the following:
    \begin{align}
        \mathbb{P}\left[\rule{0pt}{15pt} \left| \rule{0pt}{10pt} \Norm{x} - \sigma \sqrt{n} \right| \leq \varepsilon \right] &\geq 1-2e^{-\varepsilon^2/16} \\
        \mathbb{P}\left[\rule{0pt}{15pt} \left| \rule{0pt}{10pt} \Norm{y} - r \right| \leq \delta \right] &\geq 1-e^{-\delta n/r} 
    \end{align}
\end{proposition}
\begin{proof}
    This follows from \cite[Theorems 4.7 and 3.7]{wegner2021lecture}, which are the Gaussian Annulus Theorem and the concentration of measure for the unit ball, when taking account of varying the standard deviation $\sigma$ and radius $r$, respectively.
\end{proof}

The implication is that if we fix the dimension and let $\sigma$ vary, the measures will always be concentrated near spheres of radius $\sigma \sqrt{n}$ and $r$, respectively, in a consistent way. In practice, Gaussians seem to have a bit more spread, as indicated in Figure \ref{fig:sampling}, which shows the norms of $100,000$ points sampled from dimension $n=784$ (left, the dimension of MNIST) and $5,000$ points sampled from dimension $n=196,608$ (right, the dimension of ImageNet).

\begin{figure}[htb]
    \centering
    \includegraphics[width=.5\textwidth]{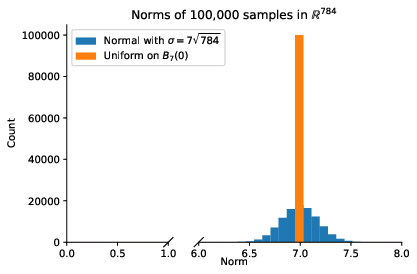}
    \includegraphics[width=.48\textwidth]{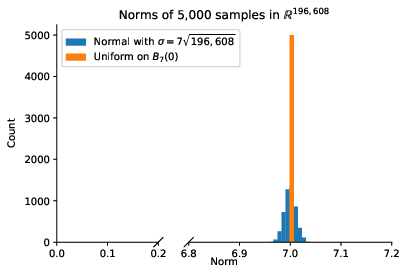}
    \caption{Comparison of the length of samples drawn from $U(B_7(0))$ and $N(0,7\sqrt{n})$ for $n=784$, the dimension of MNIST, (left) and $n=196,608$, the dimension of ImageNet, (right).}
    \label{fig:sampling}
\end{figure}

\chapter{The EPK is a Kernel}

\subsection{Multi-Class Case}

There are two ways of treating our loss function $L$ for a number of classes (or number of output activations) $K$:
\begin{align}
    \text{Case 1: } L &: \mathbb{R}^K \to \mathbb{R}\\
    \text{Case 2: } L &: \mathbb{R}^K \to \mathbb{R}^K\\
\end{align}

\subsubsection{Case 1 Scalar Loss}

Let $L : \mathbb{R}^K \to \mathbb{R}$. We use the chain rule $D (g \circ f) (x) = Dg(f(x))Df(x)$. 

Let $f$ be a vector valued function so that $f : \mathbb{R}^D \to \mathbb{R}^K$  satisfying the conditions from [representation theorem above] with $x \in \mathbb{R}^D$ and $y_i \in \mathbb{R}^K$ for every $i$. We note that $\dfrac{\partial f}{\partial t}$ is a column and has shape $Kx1$ and our first chain rule can be done the old fashioned way on each row of that column:
\begin{align}
    \dfrac{d f}{d t} &= \sum_{j=1}^M \dfrac{\partial f(x)}{\partial w_j} \dfrac{d w_j}{d t}\\
    &= -\varepsilon \sum_{j=1}^M \dfrac{\partial f(x)}{\partial w_j} \sum_{i=1}^N \dfrac{\partial L(f(x_i), y_i)}{\partial w_j}\\
    &\text{Apply chain rule}\\
    &= -\varepsilon \sum_{j=1}^M \dfrac{\partial f(x)}{\partial w_j}
      \sum_{i=1}^N  L'(f(x_i), y_i) \dfrac{d f(x_i)}{d w_j}\\
    &\text{Let}\\
    A &= \dfrac{d f(x)}{\partial w_j} \in \mathbb{R}^{K \times 1}\\
    B &= \dfrac{d L(f(x_i), y_i)}{d f} \in \mathbb{R}^{1 \times K}\\
    C &= \dfrac{d f(x_i)}{\partial w_j} \in \mathbb{R}^{K \times 1}
\end{align}
We have a matrix multiplication $ABC$ and we wish to swap the order so somehow we can pull $B$ out, leaving $A$ and $C$ to compose our product for the representation. Since $BC \in \mathbb{R}$, we have $(BC) = (BC)^T$ and we can write
\begin{align}
    (ABC)^T &= (BC)^TA^T = BCA^T\\
    ABC &= (BCA^T)^T
\end{align}
Note: This condition needs to be checked carefully for other formulations so that we can re-order the product as follows:
\begin{align}
        &= -\varepsilon \sum_{j=1}^M  \sum_{i=1}^N \left(L'(f(x_i), y_i) 
        \dfrac{d f(x_i)}{d w_j} \left(\dfrac{\partial f(x)}{\partial  w_j}\right)^T\right)^T
        \\
    &= -\varepsilon \sum_{i=1}^N \left(L'(f(x_i), y_i) 
    \sum_{j=1}^M \dfrac{d f(x_i)}{\partial w_j} \left(\dfrac{\partial f(x)}{\partial w_j}\right)^T\right)^T\\        
\end{align}
Note, now that we are summing over $j$, so we can write this as an inner product on $j$ with the $\nabla$ operator which in this case is computing the jacobian of $f$ along the dimensions of class (index k) and weight (index j). We can define 
\begin{align}
    (\nabla f(x))_{k,j} &= \dfrac{d f_{k}(x)}{\partial w_j}\\
    &= -\varepsilon \sum_{i=1}^N \left(\dfrac{d L(f(x_i), y_i)}{d f} 
     \nabla f(x_i) (\nabla f(x))^T\right)^T\\    
\end{align}
We note that the dimensions of each of these matrices in order are $[1,K]$, $[K,M]$, and $[M,K]$ which will yield a matrix of dimension $[1, K]$ i.e. a row vector which we then transpose to get back a column of shape $[K, 1]$. Also, we note that our kernel inner product now has shape $[K,K]$. 

\subsection{Schemes Other than Forward Euler (SGD)}

\textbf{Variable Step Size:}
Suppose $f$ is being trained using Variable step sizes so that across the training set $X$:
\begin{align}
    \dfrac{d w_s(t)}{dt} &= -\varepsilon_s \nabla_w L(f_{w_s(0)}(X), y_i) = -\varepsilon \sum_{j = 1}^{d} \sum_{i=1}^M  \dfrac{\partial L(f_{w_s(0)}(X),  y_i)}{\partial w_j} \label{eq10}
\end{align}
This additional dependence of $\varepsilon$ on $s$ simply forces us to keep $\varepsilon$ inside the summation in equation~\ref{eq11}. 

\textbf{Other  Numerical Schemes:} Suppose $f$ is being trained using another numerical scheme so that:
\begin{align}
    \dfrac{d w_s(t)}{dt} &= \varepsilon_{s,l} \nabla_w L(f_{w_s(0)}(x_i), y_i) + \varepsilon_{s-1, l}\nabla_w L(f_{w_{s-1}}(x_i), y_i) + \cdots \\
    &= \varepsilon_{s,l} \sum_{j = 1}^{d} \sum_{i=1}^M  \dfrac{\partial L(f_{w_s(0)}(x_i),  y_i)}{\partial w_j} + \varepsilon_{s-1, l} \sum_{j = 1}^{d} \sum_{i=1}^M  \dfrac{\partial L(f_{w_{s-1}(0)}(x_i),  y_i)}{\partial w_j} + \cdots
\end{align}
This additional dependence of $\varepsilon$ on $s$ and $l$ simply results in an additional summation in equation~\ref{eq11}. Since addition commutes through kernels, this allows separation into a separate kernel for each step contribution. Leapfrog and other first order schemes will fit this category. 

\textbf{Higher Order Schemes:} Luckily these are intractable for for most machine-learning models because they would require introducing dependence of the kernel on input data or require drastic changes. It is an open but intractable problem to derive kernels corresponding to higher order methods.

\section*{Acknowledgements}

This material is based upon work supported by the Department of Energy (National Nuclear Security Administration Minority Serving Institution Partnership Program's CONNECT - the COnsortium on Nuclear sECurity Technologies) DE-NA0004107.
This report was prepared as an account of work sponsored by an agency of the United States Government.
Neither the United States Government nor any agency thereof, nor any of their employees, makes any warranty, express or implied, or assumes any legal liability or responsibility for the accuracy, completeness, or usefulness of any information, apparatus, product, or process disclosed, or represents that its use would not infringe privately owned rights. The views and opinions of authors expressed herein do not necessarily state or reflect those of the United States Government or any agency thereof.

This material is based upon work supported by the National Science Foundation under Grant No. 2134237. We would like to additionally acknowledge funding from NSF TRIPODS Award Number 1740858 and NSF RTG Applied Mathematics and Statistics for Data-Driven Discovery Award Number 1937229. Any opinions, findings, and conclusions or recommendations expressed in this material are those of the author(s) and do not necessarily reflect the views of the National Science Foundation.

\printbibliography[heading=bibintoc]


\end{document}